\documentclass[12pt,dvips,generic,noinfoline]{imsart}

\RequirePackage[OT1]{fontenc}
\RequirePackage{amsthm,amsmath,amsfonts,natbib}
\RequirePackage{hypernat}
\usepackage[ruled,section]{algorithm}
\usepackage[noend]{algorithmic}
\usepackage{fullpage}
\usepackage{graphicx}
\usepackage{pstricks,pst-plot,pst-node}
\usepackage{subfig}

\startlocaldefs
\numberwithin{equation}{section}
\theoremstyle{plain}
\newtheorem{theorem}{Theorem}[section]
\newtheorem{corollary}{Corollary}[section]
\newtheorem{proposition}{Proposition}[section]
\newtheorem{lemma}{Lemma}[section]
\newtheoremstyle{remark}{\topsep}{\topsep}%
     {\normalfont}% Body font
     {}           % Indent amount (empty = no indent, \parindent = para indent)
     {\bfseries}  % Thm head font
     {.}          % Punctuation after thm head
     {.5em}       % Space after thm head (\newline = linebreak)
     {\thmname{#1}\thmnumber{ #2}\thmnote{ #3}}% Thm head spec
\theoremstyle{remark}
\newtheorem{remark}{Remark}[section]

\newtheorem{assumption}{Assumption}[section]
\newtheorem{definition}{Definition}[section]
\endlocaldefs

\def\X{\mathcal{X}}
\def\comma{\unskip,~}
\def\truep{p^*}
\def\div{\|\,}
\long\def\comment#1{}

\def\supp{\mathop{\text{supp}\kern.2ex}}
\def\argmin{\mathop{\text{\rm arg\,min}}}
\def\argmax{\mathop{\text{\rm arg\,max}}}
\let\hat\widehat

\let\hat\widehat

\def\ds{\displaystyle}

\def\1{{(1)}}
\def\2{{(2)}}

\def\cF{{\mathcal{F}}}

\long\def\comment#1{}
\def\phat{\hat{p}} % added by min
\def\disps{\displaystyle} % added by min
\def\fs{\footnotesize}
\def\F{{\mathcal F}}
\def\kappa{t}
\def\k{\kappa}
\def\off{\text{\sc off}}
\def\on{\text{\sc on}}
\def\ss{\vskip5pt}
\def\cS{{\mathcal{S}}} % added by min
\def\cD{{\mathcal{D}}} %added by min
\def\emcite{\cite}
\def\sgn{\mathop{\rm sign}}
\newenvironment{packed_enum}{
\begin{enumerate}
  \setlength{\itemsep}{1pt}
  \setlength{\parskip}{0pt}
  \setlength{\parsep}{0pt}
}{\end{enumerate}}
\begin{document}

\begin{frontmatter}
\title{Forest Density Estimation}%\protect\thanksref{T1}}
\runtitle{Forest Density Estimation}

\begin{aug}
\author{\fnms{Han} \snm{Liu}\ead[label=e1]{hanliu@cs.cmu.edu}}
\comma
\author{\fnms{Min} \snm{Xu}\ead[label=e1]{minxu@cs.cmu.edu}}
\comma
\author{\fnms{Haijie} \snm{Gu}\ead[label=e1]{haijie@cmu.edu}}
\comma
\author{\fnms{Anupam} \snm{Gupta}\ead[label=e1]{anupamg@cs.cmu.edu}}
\comma
\\
\author{\fnms{John} \snm{Lafferty}\ead[label=e2]{lafferty@cs.cmu.edu}}
\and
\author{\fnms{Larry} \snm{Wasserman}\ead[label=e3]{larry@stat.cmu.edu}}

\address{Carnegie Mellon University\\[0pt]
\today\\[5pt]
%\printead{e1,e2,e3}}
}
\end{aug}

\begin{abstract}
  We study graph estimation and density estimation in high dimensions,
  using a family of density estimators based on forest structured
  undirected graphical models.  For density estimation, we do not
  assume the true distribution corresponds to a forest; rather, we
  form kernel density estimates of the bivariate and univariate
  marginals, and apply Kruskal's algorithm to estimate the optimal
  forest on held out data.  We prove an oracle inequality on the
  excess risk of the resulting estimator relative to the risk of the
  best forest.  For graph estimation, we consider the problem of
  estimating forests with restricted tree sizes.  We prove that
  finding a maximum weight spanning forest with restricted tree size
  is NP-hard, and develop an approximation algorithm for this problem.
  Viewing the tree size as a complexity parameter, we then select a
  forest using data splitting, and prove bounds on excess risk and
  structure selection consistency of the procedure.  Experiments
  with simulated data and microarray data indicate that the methods
  are a practical alternative to Gaussian graphical models.
\end{abstract}


\tableofcontents
\end{frontmatter}

\section{Introduction}
\label{intro}

One way to explore the structure of a high dimensional distribution
$P$ for a random vector $X=(X_1,\ldots,X_d)$ is to estimate its
undirected graph.  The undirected graph $G$ associated with $P$ has
$d$ vertices corresponding to the variables $X_1, \ldots, X_d$, and
omits an edge between two nodes $X_i$ and $X_j$ if and only if $X_i$
and $X_j$ are conditionally independent given the other variables.
Currently, the most popular methods for estimating $G$ assume that the
distribution $P$ is Gaussian. Finding the graphical structure in this case
amounts to estimating the inverse covariance matrix $\Omega$;
the edge between $X_j$ and $X_k$ is missing if and only if
$\Omega_{jk}=0$.  Algorithms for optimizing the $\ell_1$-regularized
log-likelihood have recently been proposed that efficiently produce
sparse estimates of the inverse covariance matrix and the underlying graph
\citep{Banerjee:08,FHT:07}.

In this paper our goal is to relax the Gaussian assumption and
to develop nonparametric methods for estimating the graph of a
distribution.  Of course, estimating a high dimensional
distribution is impossible without making any assumptions.  The
approach we take here is to force the graphical structure to be a
forest, where each pair of vertices is connected by at most one path.
Thus, we relax the distributional assumption of normality but we
restrict the family of undirected graphs that are allowed.
%\footnote{Throughout
%the paper we use the term \textit{tree} to mean an acyclic graph; the
%graph is not necessarily connected.  This differs from the convention
%in much of the computer science literature, where a disconnected tree
%is referred to as a forest.}

If the graph for $P$ is a forest, then a simple conditioning argument
shows that 
its density $p$ can be written as
\begin{equation}
p(x) = \prod_{(i,j)\in E}\frac{ p(x_i,x_j)}{p(x_i)p(x_j)}\prod_{k=1}^d p(x_k)
\end{equation}
where $E$ is the set of edges in the forest \citep{Lauritzen:1996}. Here $p(x_i,x_j)$ is the bivariate marginal density of variables $X_{i}$ and $X_{j}$, and $p(x_{k})$ is the univariate marginal density of the variable $X_{k}$. With this factorization, we see that it is only necessary to 
estimate the bivariate and univariate marginals.
Given any distribution $P$ with density $p$,
there is a tree $T$ and a density $p_T$
whose graph is $T$
and which is closest in Kullback-Leibler divergence to $p$.
When $P$ is known, then the best fitting tree
distribution can be obtained by Kruskal's algorithm \citep{Kruskal:1956}, or other
algorithms for finding a maximum weight spanning tree.
In the discrete case, the algorithm can be applied to
the estimated probability mass function, resulting
in a procedure originally proposed by \cite{chow68}.
Here we are concerned with continuous random variables,
and we estimate the bivariate marginals
with nonparametric kernel density estimators.

In high dimensions, fitting a fully connected spanning tree can be
expected to overfit.  We regulate the complexity of the forest by
selecting the edges to include using a data splitting scheme, a simple
form of cross validation.  In particular, we consider the family of
forest structured densities that use the marginal kernel density
estimates constructed on the first partition of the data, and estimate
the risk of the resulting densities over a second, held out partition.
The optimal forest in terms of the held out risk is then obtained by
finding a maximum weight spanning forest for an appropriate set of
edge weights.

A closely related approach is proposed by \cite{Bach03}, where a tree
is estimated for the random vector $Y=WX$ instead of $X$, where $W$ is
a linear transformation, using an algorithm that alternates between
estimating $W$ and estimating the tree $T$.  Kernel density estimators
are used, and a regularization term that is a function of the number
of edges in the tree is included to bias the optimization toward
smaller trees.  We omit the transformation $W$, and we use a data
splitting method rather than penalization to choose the complexity of
the forest.

While tree and forest structured density estimation has been long
recognized as a useful tool, there has been little theoretical
analysis of the statistical properties of such density estimators.  The main
contribution of this paper is an analysis of the asymptotic properties
of forest density estimation in high dimensions.  We allow both the
sample size $n$ and dimension $d$ to increase, and prove oracle
results on the risk of the method.  In particular, we assume that the
univariate and bivariate marginal densities lie in a H\"older class
with exponent $\beta$ (see Section~\ref{sec.theory} for details), and
show that
\begin{equation}
  R(\hat{p}_{\hat{F}}) - \min_{F}R(\hat{p}_{F})
 = O_{P}\left(\sqrt{\log(nd)} \left(\frac{k^{*} + \hat{k}}{n^{\beta/(2+2\beta)}}
+ \frac{d}{n^{\beta/(1+2\beta)}}\right)\right)
\end{equation}
where $R$ denotes the risk, the expected negative log-likelihood, 
$\hat k$ is the number of edges in the estimated forest $\hat F$, and
$k^*$ is the number of edges in the optimal forest $F^*$ that
can be constructed in terms of the kernel density estimates
$\hat p$.

In addition to the above results on risk consistency, we establish conditions under which
\begin{equation}
\mathbb{P}\left( \hat{F}^{(k)}_{d}=F^{*(k)}_{d}\right) \rightarrow 1 
\end{equation}
as $n\to\infty$, where $F^{*(k)}_d$ is the \textit{oracle forest}---the
best forest with $k$ edges; this result allows the dimensionality $d$ to increase as fast as
$o\left(\exp(n^{\beta/(1+\beta)}) \right)$, while still having
consistency in the selection of the oracle forest.

Among the only other previous work analyzing
tree structured graphical models is \cite{Tan:09a} and
\cite{Chechetka+Guestrin:nips07tjtpac}. \cite{Tan:09a} analyze the
error exponent in the rate of decay of the error probability for
estimating the tree, in the fixed dimension setting, and
\cite{Chechetka+Guestrin:nips07tjtpac} give a PAC analysis.  An
extension to the Gaussian case is given by \cite{Tan:09b}.

We also study the problem of estimating forests with restricted tree
sizes.  In many applications, one is interested in obtaining a
graphical representation of a high dimensional distribution to aid in
interpretation.  For instance, a biologist studying gene interaction
networks might be interested in a visualization that groups together
genes in small sets.  Such a clustering approach through density
estimation is problematic if the graph is allowed to have cycles, as
this can require marginal densities to be estimated with many
interacting variables.  Restricting the graph to be a forest
circumvents the curse of dimensionality by requiring only univariate
and bivariate marginal densities.  The problem of clustering the
variables into small interacting sets, each supported by a
tree-structured density, becomes the problem of estimating a
maximum weight spanning forest with a restriction on the size of each
component tree.  As we demonstrate, estimating restricted tree size
forests can also be useful in model selection for the purpose of risk
minimization.  Limiting the tree size gives another way of regulating
tree complexity that provides larger family of forest to select from
in the data splitting procedure.

While the problem of finding a maximum weight forest with restricted
tree size may be natural, it appears not to have been studied
previously.  We prove that the problem is NP-hard through a reduction
from the problem of Exact 3-Cover \citep{Garey:79}, where we are given
a set $X$ and a family $\cS$ of 3-element subsets of $X$, and must
choose a subfamily of disjoint 3-element subsets to cover $X$.  While
finding the exact optimum is hard, we give a practical
$4$-approximation algorithm for finding the optimal tree restricted
forest; that is, our algorithm outputs a forest whose weight is
guaranteed to be at least $\frac{1}{4}w(F^*)$, where $w(F^*)$ is the
weight of the optimal forest.  This approximation guarantee translates
into excess risk bounds on the constructed forest using our previous
analysis.  Our experimental results with this approximation algorithm
show that it can be effective in practice for forest density
estimation.

In Section \ref{sec.notation} we review some background and notation.
In Section \ref{sec.method} we present a two-stage algorithm for
estimating high dimensional densities supported by forests, and we
provide a theoretical analysis of the algorithm in
Section~\ref{sec.theory}, with the detailed proofs collected in an
appendix. In Section \ref{sec.restricted}, we explain how to estimate
maximum weight forests with restricted tree size.  In
Section~\ref{sec.experiments} we present experiments with both
simulated data and gene microarray datasets, where the problem is to
estimate the gene-gene association graphs.

\section{Preliminaries and Notation}
\label{sec.notation}

Let $\truep(x)$ be a probability density with respect to Lebesgue
measure $\mu(\cdot)$ on $\mathbb{R}^{d}$ and let $X^{(1)}, \ldots, X^{(n)}$
be $n$ independent identically distributed $\mathbb{R}^{d}$-valued
data vectors sampled from $\truep(x)$
where $X^{(i)} = (X^{(i)}_1,\ldots, X^{(i)}_d)$.
Let $\X_{j}$ denote the range of $X_j^{(i)}$ and let
$\X = \X_{1} \times \cdots \times \X_{d}$.

A graph is a forest if it is acyclic.  If $F$
is a $d$-node undirected forest with vertex set $V_{F} = \{1,\ldots,d \}$ 
and edge set 
$E({F}) \subset \{1, \ldots, d \} \times \{1,\ldots, d\}$,
the number of edges satisfies $|E({F})|<d$, noting
that we do not restrict the graph to be connected.
We say that a probability density function $p(x)$ is \textit{supported
by a forest $F$} if the density can be written as
\begin{equation}
p_F(x) =  \prod_{(i,j)\in E({F})}
\frac{ p(x_{i}, x_{j})}{p(x_{i})\,p(x_{j})} \prod_{k\in V_{F}}p(x_{k}), \label{eq.treedensity}
\end{equation}
where each $p(x_{i}, x_{j})$ is a bivariate density on $\X_i\times \X_j$, 
and each $p(x_{k})$ is a univariate density on $\X_k$. More details can be found in \cite{Lauritzen:1996}.

Let $\F_{d}$ be the family of forests with $d$ nodes,
and let $\mathcal{P}_{d}$ be the corresponding family
of densities:
\begin{equation}
\mathcal{P}_{d} = 
\left\{ p\geq 0: \ \int_{\X} p(x)\,d\mu(x) = 1,\; \text{and $p(x)$
    satisfies \eqref{eq.treedensity} for some $F\in\F_d$}\right\}.
\label{eq.Pd}
\end{equation}
To bound the number of labeled spanning forests on $d$ nodes, note that each
such forest can be obtained by forming a labeled tree on $d+1$ nodes,
and then removing node $d+1$.  From Cayley's formula
\citep{Cayley:1889,THEBOOK:98}, we then obtain
the following.
\begin{proposition}
The size of the collection $\F_d$ of labeled forests on $d$ nodes satisfies
\begin{equation}
|\F_d| < (d+1)^{d-1}.
\end{equation}
\end{proposition}

Define the oracle forest density
\begin{equation}
q^{*}= \argmin_{q \in \mathcal{P}_{d}} D(\truep \div q) \label{eq.oracle}
\end{equation} 
where the Kullback-Leibler divergence 
$D(p\div q)$ between two densities $p$ and $q$ is
\begin{equation}
D(p\div q) = \int_{\X} p(x) \log \frac{p(x)}{q(x)} dx, 
\end{equation}
under the convention that $0\log(0/q) = 0$, and $p\log(p/0) = \infty$ for $p\neq 0$. 
The following is proved by \cite{Bach03}.

\begin{proposition}\label{prop.oracle} 
Let $q^{*}$ be defined as in \eqref{eq.oracle}.
There exists a forest $F^{*}\in \mathcal{F}_{d}$, such that
\begin{equation}
q^{*} = \truep_{F^{*}} = 
\prod_{(i,j)\in E({F^{*}})}\frac{ \truep(x_{i}, x_{j})}{\truep(x_{i})\,\truep(x_{j})} 
\prod_{k\in V_{F^{*}}}\truep(x_{k})  \label{eq.Tstar}
\end{equation}
where $\truep(x_{i}, x_{j})$ and $\truep(x_{i})$ are the bivariate and
univariate marginal densities of $\truep$.
\end{proposition}

%\begin{proof}
%According to \citep{Bach03}, for a given tree $T$ and a target
%distribution $\truep(x)$, we have, for all distributions $q_{T}$
%factors according to $T$ with $T\in\mathcal{T}_{d}$,
%\begin{equation}
%D(\truep\| q_{T}) = D(\truep \| \truep_{T}) + D(\truep_{T} \| q_{T}). \label{eq.bach}
%\end{equation}
%The claim then follows from the fact that $D(\truep_{T} \| q_{T})\geq 0$ and the equality is achieved only when $\truep_{T}(x) =
%q_{T}(x)$ for $\mu$-almost surely.
%\end{proof}

For any density $q(x)$, the negative log-likelihood risk $R(q)$ is defined as
\begin{equation}
R(q) = -\mathbb{E}\log q(X) = -\int_{\X} \truep(x)\log q(x) \,dx
\end{equation}
where the expectation is defined with respect to the distribution of $X$.

It is straightforward
to see that the density $q^{*}$ defined in \eqref{eq.oracle} also minimizes the negative log-likelihood loss:
\begin{equation}
q^{*} = \argmin_{q\in\mathcal{P}_d} D(\truep\div q) = \argmin_{q \in
  \mathcal{P}_{d}}R(q). 
\end{equation}
Let $\hat{p}(x)$ be the kernel density estimate, we also define
\begin{equation}
\hat{R}(q) =  -\int_{\X} \hat{p}(x)\log q(x) \,dx. 
\end{equation}

%\begin{proof}
%The desired result follows from the fact that
%\begin{equation}
%D(\truep \| q) = \int_{\X} \truep(x)\log \truep(x) - 
%\int_{\X} \truep(x)\log q(x) dx  = \int_{\X} \truep(x)\log \truep(x)dx + R(q), \label{eq.equivresult}
%\end{equation}
%where the first term in the last equality does not depend on $q$ at all.
%\end{proof}

We thus define the oracle risk as $R^{*}=R(q^{*})$.  
Using Proposition \ref{prop.oracle} and equation \eqref{eq.treedensity}, we have
\begin{eqnarray}
\nonumber
R^{*} &=& R(q^{*}) \;=\; R(\truep_{F^{*}}) \\
& = & -\int_{\X} \truep(x)\biggl( \sum_{(i,j)\in E({F^{*}})} 
\log\frac{ \truep(x_{i}, x_{j})}{\truep(x_{i})\truep(x_{j})} + 
\sum_{k\in V_{F^{*}}}\log\left( \truep(x_{k})\right)\biggr)dx \nonumber\\
&= & -\sum_{(i,j)\in E({F^{*}})} \int_{\X_{i}\times \X_{j}} 
\truep(x_{i}, x_{j})\log\frac{ \truep(x_{i}, x_{j})}{\truep(x_{i})\truep(x_{j})}dx_{i}dx_{j} - 
\sum_{k\in V_{F^{*}}}\int_{\X_{k}}\truep(x_{k})\log \truep(x_{k}) dx_{k} \nonumber\\
& = &-\sum_{(i,j)\in E({F^{*}})}  I(X_{i}; X_{j})  + \sum_{k\in V_{F^{*}}} H(X_{k}),\label{eq.oraclerisk}
\end{eqnarray}
where 
\begin{equation}
I(X_{i}; X_{j})  = 
\int_{\X_{i}\times \X_{j}} \truep(x_{i}, x_{j})
\log \frac{\truep(x_{i}, x_{j}) }{\truep(x_{i})\, \truep(x_{j})} \,dx_{i}dx_{j} 
\end{equation} 
is the mutual information between the pair of variables $X_{i}$, $X_{j}$
and
\begin{equation}
H(X_{k}) = -\ds \int_{\X_{k}} \truep(x_{k}) \log\truep(x_{k})\, dx_{k}
\end{equation}
is the entropy. 
%For details about these information theoretic
%quantities,  see \citep{Cover:91}.
While the best forest will in fact be a spanning tree, the
densities $p^*(x_i,x_j)$ are in practice not known.  We estimate the
marginals using finite data, in terms of a kernel density estimates
$\hat p_{n_1}(x_i,x_j)$ over a training set of size $n_1$.   With these estimated marginals, we
consider all forest density estimates of the form
\begin{equation}
\hat p_F(x) =  \prod_{(i,j)\in E({F})}
\frac{ \hat p_{n_1}(x_{i}, x_{j})}{\hat p_{n_1}(x_{i})\,\hat  p_{n_1}(x_{j})} 
\prod_{k\in V_{F}} \hat p_{n_1}(x_{k}). \label{eq.esttreedensity}
\end{equation}
Within this family, the best density estimate may not be supported on 
a full spanning tree, since a full tree will in general be subject to
overfitting.   Analogously, in high dimensional linear regression, 
the optimal regression model will generally be a full $d$-dimensional
fit, with a nonzero parameter for each variable. However, when
estimated on finite data the variance of a full model will dominate
the squared bias, resulting in overfitting.  In our setting of
density estimation we will regulate the complexity of the forest
by cross validating over a held out set.

There are several different ways to judge the quality of a forest structured
density estimator.  In this paper we concern ourselves with
prediction and structure estimation.

\begin{definition}[(Risk consistency)]   For an 
estimator $\hat{q}_n\in\mathcal{P}_{d}$,
the {\it excess risk} is defined as $R(\hat{q}_n) - R^{*}$.
The estimator $\hat{q}_n$ is \textit{risk consistent with convergence rate $\delta_{n}$} if
\begin{equation}
\lim_{M\rightarrow \infty}\limsup_{n\rightarrow \infty}\;\mathbb{P}\left(R(\hat{q}_n) - R^{*} \geq M\delta_{n} \right) = 0.
\end{equation}
In this case we write $R(\hat{q}_n) - R^{*}  = O_{P}(\delta_{n})$.
\end{definition}

\begin{definition}[(Estimation consistency)]
An estimator $\hat{q}_n\in \mathcal{P}_{d}$ 
is \textit{estimation consistent with convergence rate $\delta_{n}$},
with respect to the Kullback-Leibler divergence, if 
\begin{eqnarray}
\lim_{M\rightarrow \infty}\limsup_{n\rightarrow \infty}
\;\mathbb{P}\left( D(\truep_{F^{*}}\div \hat{q}_n)  \geq M\delta_{n}\right) = 0.
\end{eqnarray}
\end{definition}

\begin{definition}[(Structure selection consistency)]
An estimator $\hat{q}_n\in \mathcal{P}_{d}$ supported by a forest $\hat{F}_n$  is \textit{structure selection consistent} if 
\begin{eqnarray}
\mathbb{P}\left(E({\hat{F}_n}) \neq E({F^{*}})  \right)\rightarrow 0,
\end{eqnarray}
as $n$ goes to infinity, where $F^{*}$ is defined in \eqref{eq.Tstar}.
\end{definition}

Later we will show that estimation consistency is almost equivalent to
risk consistency. If the true density is given, these two criteria are
exactly the same; otherwise, estimation consistency requires
stronger conditions than risk consistency.

It is important to note that risk consistency is an oracle
property, in the sense that the true density
$\truep(x)$ is not restricted to be supported by a forest;
rather, the property assesses how well a given estimator $\hat{q}$ approximates the
best forest density (the oracle) within a class.

\section{Kernel Density Estimation For Forests}
\label{sec.method}

If the true density $\truep(x)$ were known,  
by Proposition \ref{prop.oracle}, the density estimation problem 
would be reduced to finding the best forest structure $F^{*}_{d}$, satisfying
\begin{equation}
F^{*}_{d} = \argmin_{F\in\mathcal{F}_{d}}R(\truep_{F}) = 
\argmin_{F\in\mathcal{F}_{d}} D(\truep\div \truep_{F}).
\end{equation}
The optimal forest $F^{*}_{d}$ can be found by minimizing the right hand
side of \eqref{eq.oraclerisk}. Since the entropy term 
$H(X) = \sum_{k} H(X_{k})$ is constant across all forests, this
can be recast as the problem of
finding the maximum weight spanning forest for a weighted graph, where the
weight of the edge connecting nodes $i$ and $j$ is $I(X_{i};
X_{j})$.  Kruskal's algorithm \citep{Kruskal:1956} is a greedy
algorithm that is guaranteed to find a maximum weight spanning tree of
a weighted graph.  In the setting of density estimation,
this procedure was proposed by \cite{chow68} as a way of constructing
a tree approximation to a distribution.  At each stage the algorithm adds an edge
connecting that pair of variables with maximum mutual information
among all pairs not yet visited by the algorithm, if doing
so does not form a cycle.  When stopped early, after $k<d-1$ edges have been
added, it yields the best $k$-edge weighted forest.

Of course, the above procedure is not practical since the true density
$\truep(x)$ is unknown.
We replace the population mutual information
$I(X_{i}; X_{j})$ 
in \eqref{eq.oraclerisk} by the plug-in estimate
$\hat{I}_n(X_{i}, X_{j})$, defined as
\begin{eqnarray}
\hat{I}_n(X_{i}, X_{j})  = \int_{\X_{i}\times \X_{j}} 
\hat{p}_n(x_{i}, x_{j}) \log \frac{\hat{p}_n(x_{i}, x_{j}) }{\hat{p}_n(x_{i})\, \hat{p}_n(x_{j})} \,dx_{i}dx_{j} 
%\hat{h}(X_{k}) = -\int_{\X_{k}} \hat{p}(x_{k}) \log\hat{p}(x_{k}) dx_{k},\label{eq.sampleMI}
\end{eqnarray} 
where $\hat{p}_n(x_{i}, x_{j})$ and $\hat{p}_n(x_{i})$ are 
bivariate and univariate kernel density estimates.
Given this estimated mutual information matrix $\hat{M}_n =
\left[\hat{I}_n(X_{i}, X_{j}) \right]$, we can then apply 
Kruskal's algorithm (equivalently, the Chow-Liu algorithm)
to find the best forest structure
$\hat{F}_{n}$.  

Since the number of edges of $\hat{F}_{n}$ controls the
number of degrees of freedom in the final density estimator, we need an automatic
data-dependent way to choose it. We adopt the following two-stage procedure.
First, randomly partition the data into two sets $\mathcal{D}_{1}$ and
$\mathcal{D}_{2}$ of sizes $n_{1}$ and $n_{2}$; then, apply the following
steps:

\begin{enumerate}
\item Using $\mathcal{D}_{1}$, construct kernel density estimates of
  the univariate and bivariate marginals and calculate
  $\hat{I}_{n_1}(X_{i}, X_{j})$ for $i,j \in \{1, \ldots, d \}$ with
  $i\neq j$. Construct a full tree $\hat{F}^{(d-1)}_{n_1}$ with $d-1$
  edges, using the Chow-Liu algorithm.

\item Using $\mathcal{D}_{2}$, prune the tree $\hat{F}^{(d-1)}_{n_1}$
  to find a forest $\hat{F}^{(\hat k)}_{n_1}$ with 
  $\hat{k}$ edges, for $0\leq \hat{k} \leq d-1$.
\end{enumerate}

Once $\hat{F}^{(\hat k)}_{n_1}$ is obtained in Step 2, we can calculate
$\hat{p}_{\hat{F}^{(\hat k)}_{n_1}}$ according to \eqref{eq.treedensity}, using the
kernel density estimates constructed in Step 1.

\subsection{Step 1:  Estimating the marginals} %$\hat{F}^{(d-1)}_{n_1}$ on $\mathcal{D}_{1}$}
\label{subsec.tde.step1}

Step 1 is carried out on the dataset $\mathcal{D}_{1}$.
Let $K(\cdot)$ be a univariate kernel function. Given an
evaluation point $(x_{i}, x_{j})$, the bivariate kernel density
estimate for $(X_{i}, X_{j})$ based on the observations $\{X^{(s)}_{i},
X^{(s)}_{j}\}_{s\in\mathcal{D}_{1}}$ is defined as
\begin{equation}
\hat{p}_{n_1}(x_{i}, x_{j}) = 
\frac{1}{n_{1}}\sum_{s \in \mathcal{D}_{1}} \frac{1}{h^{2}_{2}}\,
K\left( \frac{X^{(s)}_{i} - x_{i}}{h_{2}} \right)
K\left( \frac{X^{(s)}_{j} - x_{j}}{h_{2}} \right), \label{eq.bivariatekde}
\end{equation}
where we use a product kernel with $h_{2} > 0$ be the bandwidth
parameter. 
The univariate kernel density estimate $\hat{p}_{n_1}(x_{k})$ for $X_{k}$
is 
\begin{equation}
\hat{p}_{n_1}(x_{k}) = \frac{1}{n_{1}}\sum_{s \in \mathcal{D}_{1}} 
\frac{1}{h_{1}}  K\left( \frac{X^{(s)}_{k} - x_{k}}{h_{1}} \right), \label{eq.univariatekde}
\end{equation}  
where $h_{1}>0$ is the univariate bandwidth.
Detailed specifications for $K(\cdot)$ and $h_1$, $h_{2}$ will
be discussed in the next section.

We assume that the data lie in a $d$-dimensional unit cube $\X =
[0,1]^{d}$. To calculate the empirical mutual information
$\hat{I}_{n_1}(X_{i}, X_{j})$, we need to numerically evaluate a two-dimensional
integral. To do so, we calculate the kernel density
estimates on a grid of points. We choose $m$ evaluation points on each
dimension,  $x_{1i} < x_{2i} < \cdots < x_{mi}$ for the $i$th variable.
The mutual information $\hat{I}_{n_1}(X_{i}, X_{j})$ is then approximated as
\begin{equation}
\hat{I}_{n_1}(X_{i}, X_{j}) =
\frac{1}{m^{2}}\sum_{k=1}^{m}\sum_{\ell=1}^{m}\hat{p}_{n_1}(x_{ki},x_{\ell j}) 
\log \frac{\hat{p}_{n_1}(x_{ki}, x_{\ell j}) }{\hat{p}_{n_1}(x_{ki})\,
  \hat{p}_{n_1}(x_{\ell j})}.  \label{eq.empiricalMI}
\end{equation}
The approximation error can be made arbitrarily small by choosing
$m$ sufficiently large.  As a practical concern,
care needs to be taken that the factors $\hat{p}_{n_1}(x_{ki})$ and
$\hat{p}_{n_1}(x_{\ell j})$ in the denominator are not too small; a truncation procedure
can be used to ensure this.
Once the $d\times d$ mutual information matrix
$\hat M_{n_1} = \left[\hat{I}_{n_1}(X_{i}, X_{j}) \right]$ is obtained, we
can apply the Chow-Liu (Kruskal) algorithm to find a maximum weight spanning tree.

\begin{algorithm}[t]
   \caption{\ \ Chow-Liu (Kruskal)\label{algorithm:naiveMI}}
%   {\bfseries Input:} Data set $\mathcal{D}_{1}$ and the  bandwidths
%   $h_{1}$, $h_{2}$
\vskip5pt
   \begin{algorithmic}[1] 
\STATE \textbf{Input} data $\mathcal{D}_1 = \{X^{(1)}, \ldots, X^{(n_1)}\}$.
\\[5pt]
\STATE Calculate $\hat M_{n_1}$, according to \eqref{eq.bivariatekde},
\eqref{eq.univariatekde}, and \eqref{eq.empiricalMI}.
\\[5pt]
\STATE Initialize $E^{(0)}=\emptyset$
\vskip5pt
 \FOR{$k = 1, \ldots, d-1$}
\vskip5pt
 \STATE $(i^{(k)}, j^{(k)}) \leftarrow \argmax_{(i,j)} \hat M_{n_1}(i,j)$ such
  that $E^{(k-1)} \cup\{(i^{(k)}, j^{(k)}) \}$ does not contain a cycle
\vskip5pt
 \STATE $E^{(k)} \leftarrow E^{(k-1)} \cup \{(i^{(k)}, j^{(k)})\}$
 \ENDFOR
\vskip5pt
 \STATE \textbf{Output} tree $\hat{F}^{(d-1)}_{n_1}$ with edge set $E^{(d-1)}$.
\end{algorithmic}
\end{algorithm}

\subsection{Step 2:  Optimizing the forest}\label{subsubsec.stage2}
\label{subsec.tde.step2} 
 
The full tree $\hat{F}^{(d-1)}_{n_1}$ obtained in Step 1 might have
high variance when the dimension $d$ is large, leading to overfitting
in the density estimate.  In order to reduce the variance, we prune the
tree; that is, we choose forest with $k\leq d-1$ edges.  The
number of edges $k$ is a tuning parameter that 
induces a bias-variance tradeoff.

In order to choose $k$, note that in stage $k$ of the Chow-Liu algorithm
we have an edge set $E^{(k)}$ (in the
notation of the Algorithm~\ref{algorithm:naiveMI}) 
which corresponds to a forest $\hat{F}^{(k)}_{n_1}$ with $k$ edges, where
$\hat{F}^{(0)}_{n_1}$ is the union of $d$ disconnected nodes.  To
select $k$, we choose among
the $d$ trees $\hat{F}^{(0)}_{n_1},\ \hat{F}^{(1)}_{n_1},\ \ldots\,, \hat{F}^{(d-1)}_{n_1}$.

Let $\hat{p}_{n_2}(x_{i}, x_{j})$ and $\hat{p}_{n_2}(x_{k})$ be
defined as in \eqref{eq.bivariatekde} and \eqref{eq.univariatekde},
but now evaluated solely based on the held-out data in
$\mathcal{D}_{2}$. For a density $p_{F}$ that is supported by a 
forest $F$, we define the held-out negative log-likelihood
risk as
\begin{eqnarray}
\lefteqn{\hat{R}_{n_2}(p_{F})} \label{eq.heldoutrisk} \\
& =& -\sum_{(i,j)\in E_{F}} \int_{\X_{i}\times \X_{j}} 
\hat{p}_{n_2}(x_{i}, x_{j})\log\frac{ p(x_{i},
  x_{j})}{p(x_{i})\,p(x_{j})} \,dx_{i}dx_{j} - 
\sum_{k\in V_{F}}\int_{\X_{k}}\hat{p}_{n_2}(x_{k})\log p(x_{k}) \,dx_{k}.\nonumber
\end{eqnarray}
The selected forest is then $\hat{F}^{(\hat k)}_{n_1}$ where
\begin{eqnarray}
\hat{k}=\argmin_{k \in\{0, \ldots, d-1 \}}\hat{R}_{n_2}\left(\hat{p}_{\hat{F}^{(k)}_{n_1}}\right)
\end{eqnarray}
and where $\hat{p}_{\hat{F}^{(k)}_{n_1}}$ is computed using the density
estimate $\hat{p}_{n_1}$ constructed on $\mathcal{D}_1$.

For computational simplicity, we can also estimate $\hat{k}$ as
\begin{eqnarray}
\hat{k} 
&=& \argmax_{k \in\{0, \ldots, d-1 \}}\;
\frac{1}{n_{2}}\sum_{s\in\mathcal{D}_{2}}\log 
\left( \prod_{(i,j)\in E^{(k)}}
\frac{ \hat{p}_{n_1}(X^{(s)}_{i}, X^{(s)}_{j})}{\hat{p}_{n_1}(X^{(s)}_{i})\,\hat{p}_{n_1}(X^{(s)}_{j})} 
\prod_{k\in V_{\hat{F}_{n_1}^{(k)}}}\hat{p}_{n_1}(X^{(s)}_{k}) \right) \\
&=& \argmax_{k \in\{0, \ldots, d-1 \}}\;
\frac{1}{n_{2}}\sum_{s\in\mathcal{D}_{2}}\log 
\left( \prod_{(i,j)\in E^{(k)}}
\frac{ \hat{p}_{n_1}(X^{(s)}_{i}, X^{(s)}_{j})}{\hat{p}_{n_1}(X^{(s)}_{i})\, \hat{p}_{n_1}(X^{(s)}_{j})}\right) .
\end{eqnarray}
This minimization can be  efficiently carried out by iterating
over the $d-1$ edges in $\hat{F}^{(d-1)}_{n_1}$.

Once $\hat{k}$ is obtained, the final forest density
estimate is given by
\begin{eqnarray}
\hat{p}_{n}(x) =  
\prod_{(i,j)\in E^{(\hat k)}}\frac{ \hat{p}_{n_1}(x_{i}, x_{j})}{\hat{p}_{n_1}(x_{i})\,\hat{p}_{n_1}(x_{j})} 
\prod_{k}\hat{p}_{n_1}(x_{k}). \label{eq.tkde}
\end{eqnarray}

\begin{remark}
For computational efficiency, Step 1 can be carried out
simultaneously with Step 2.  In particular, during the Chow-Liu
iteration, whenever an edge is added to $E^{(k)}$, the
log-likelihood of the resulting density estimator on
$\mathcal{D}_{2}$ can be immediately computed. A more efficient
algorithm to speed up the computation of the mutual information
matrix is discussed in Appendix \ref{subsection.efficientMI}.
\end{remark}

\subsection{Building a forest on held-out data}
\label{sec:hoforest}

Another approach to estimating the forest structure is to estimate the
marginal densities on the training set, but only build graphs
on the held-out data.   To do so, we first estimate the univariate and bivariate kernel density estimates using $\cD_{1}$, denoted by $\hat{p}_{n_1}(x_i)$ and $\hat{p}_{n_1}(x_i,x_j)$. We also construct a
new set of univariate and bivariate kernel density estimates using $\cD_2$, 
$\hat{p}_{n_2}(x_i)$ and $\hat{p}_{n_2}(x_i,x_j)$. We then
estimate the ``cross-entropies'' of the kernel density estimates
$\hat{p}_{n_1}$ for each pair of variables by computing
\begin{eqnarray}
\hat{I}_{n_2,n_1}(X_i,X_j) &=& \int \hat{p}_{n_2}(x_i,x_j) \log
\frac{\hat{p}_{n_1}(x_i,x_j)}{\hat{p}_{n_1}(x_i)\hat{p}_{n_1}(x_j)} \,
dx_i \, dx_j\\
&\approx& \frac{1}{m^{2}}\sum_{k=1}^{m}\sum_{\ell=1}^{m}\hat{p}_{n_2}(x_{ki},x_{\ell j}) 
\log \frac{\hat{p}_{n_1}(x_{ki}, x_{\ell j}) }{\hat{p}_{n_1}(x_{ki})\,
  \hat{p}_{n_1}(x_{\ell j})}.  \label{eq.empiricalXMI}
\end{eqnarray}
Our method is to use $\hat{I}_{n_2, n_1}(X_i,X_j)$ as edge weights on a
full graph and run Kruskal's algorithm until we encounter edges with
negative weight. Let $\cF$ be the set of all forests and $\hat{w}_{n_2}(i,j) =
\hat{I}_{n_2, n_1}(X_i,X_j)$.  The final forest is then
\begin{eqnarray}
\hat{F}_{n_2} = \argmax_{F \in \cF} \hat{w}_{n_2} (F) = \argmin_{F \in \cF} \hat{R}_{n_2} (\hat{p}_F)
\end{eqnarray}
By building a forest on held-out data,
we directly cross-validate over {\it all} forests.

\section{Statistical Properties}
\label{sec.theory}

In this section we present our theoretical results on risk
consistency, structure selection consistency, and estimation consistency 
of the forest density estimate
$\hat{p}_n = \hat{p}_{\hat{F}^{(\hat{k})}_{d}}$.

To establish some notation, we write $a_{n} = \Omega(b_{n})$ if there
exists a constant $c$ such that $a_{n} \geq c b_{n}$ for sufficiently
large $n$.  We also write $a_{n}\asymp b_{n}$ if there exists a
constant $c$ such that $ a_{n} \leq c\,b_{n}$ and $b_{n} \leq c\,
a_{n}$ for sufficiently large $n$.  Given a $d$-dimensional function
$f$ on the domain $\X$, we denote its $L_{2}(P)$-norm and sup-norm as
\begin{equation}
\| f\|_{L_{2}(P)} = \sqrt{\int_{\X} f^{2}(x) dP_{X}(x)}, \qquad \| f\|_{\infty} = \sup_{x \in \X}|f(x)|
\end{equation}
where $P_{X}$ is the probability measure induced by $X$. Throughout
this section, all constants are treated as generic values,
and as a result they can change from line to line.

In our use of a data splitting scheme, we always adopt equally sized
splits for simplicity, so that $n_{1} = n_{2} = n/2$, noting that this
does not affect the final rate of convergence.

\subsection{Assumptions on the density}

%In the sequel, we use $\lfloor \beta \rfloor$ to denote the largest
%integer that is strictly smaller than a real number $\beta$. Following
%the same notation as in \citep{Rig09}, we recall definition of the
%H\"{o}lder class.

Fix $\beta > 0$.  For any $d$-tuple $\alpha = (\alpha_{1}, \ldots,
\alpha_{d}) \in \mathbb{N}^{d}$ and $x=(x_{1}, \ldots, x_{d}) \in \X$,
we define $x^{\alpha} = \prod_{j=1}^{d}x_{j}^{\alpha_{j}}$. Let
$D^{\alpha}$ denote the differential operator
\begin{eqnarray}
 D^{\alpha} = \frac{\partial^{\alpha_{1} + \cdots + \alpha_{d}}}{\partial x_{1}^{\alpha_{1}} \cdots \partial x_{d}^{\alpha_{d}}}.
\end{eqnarray}
For any real-valued $d$-dimensional function $f$ on $\X$ that is
$\lfloor \beta \rfloor$-times continuously differentiable at point
$x_{0} \in \X$, let $P^{(\beta)}_{f,x_{0}}(x) $ be its Taylor
polynomial of degree $\lfloor \beta \rfloor$ at point $x_{0}$:
\begin{eqnarray}
P^{(\beta)}_{f,x_{0}}(x) = \sum_{\alpha_{1} + \cdots +\alpha_{d} \leq \lfloor \beta \rfloor} \frac{(
x-x_{0})^{\alpha}}{\alpha_{1}!\cdots \alpha_{d}!} D^{\alpha} f(x_{0}).
\end{eqnarray}
Fix $L > 0$, and denote by $\Sigma(\beta, L, r, x_{0})$ the set of
functions $f: \X \rightarrow \mathbb{R}$ that are $\lfloor \beta
\rfloor$-times continuously differentiable at $x_{0}$ and satisfy
\begin{equation}
\left|f(x) - P^{(\beta)}_{f,x_{0}}(x)\right| \leq L\|x - x_{0} \|^{\beta}_{2},~~\forall x \in \mathcal{B}(x_{0}, r)
\end{equation}
where $\mathcal{B}(x_{0}, r) = \{x: \|x - x_{0} \|_{2} \leq r \}$ is
the $L_{2}$-ball of radius $r$ centered at $x_{0}$. The set
$\Sigma(\beta, L, r, x_{0})$ is called the $(\beta, L, r, x_{0})$-locally
H\"{o}lder class of functions. Given a set $A$, we define
\begin{equation}
\Sigma(\beta, L, r, A) = \cap_{x_{0} \in A}\Sigma(\beta, L, r, x_{0}).
\end{equation}

The following are the regularity assumptions we make on the true density function $\truep(x)$.
\begin{assumption}\label{assump.density}
For any $1 \leq i < j \leq d$, we assume
\begin{packed_enum}
\item[(D1)] there exist $L_{1}>0$ and $L_{2}>0$ such that for any
  $c>0$  the true bivariate and univariate densities satisfy
 \begin{equation}
 \truep(x_{i}, x_{j}) \in \Sigma\left(\beta, L_{2}, c\left( {\log n}/{ n} \right)^{\frac{1}{2\beta+2}}, \X_{i}\times \X_{j}\right)
 \end{equation}
 and
 \begin{equation}
 \truep(x_{i}) \in \Sigma\left(\beta, L_{1}, c\left( {\log n}/{ n} \right)^{\frac{1}{2\beta+1}}, \X_{i}\right);
 \end{equation}
\item[(D2)] there exists two constants $c_{1}$ and $c_{2}$ such that
\begin{equation}
c_{1} \leq \inf_{x_{i}, x_{j} \in \X_{i}\times\X_{j} }\truep(x_{i}, x_{j}) \leq  \sup_{x_{i}, x_{j} \in \X_{i}\times\X_{j} }\truep(x_{i}, x_{j})  \leq c_{2}
\end{equation} 
$\mu$-almost surely.
\end{packed_enum}
\end{assumption}
These assumptions are mild, in the sense that instead of adding
constraints on the joint density $\truep(x)$, we only add regularity
conditions on the bivariate and univariate marginals.

\subsection{Assumptions on the kernel}

An important ingredient in our analysis is an exponential
concentration result for the kernel density estimate, due to
\cite{Gine:2002}. We first specify the requirements on the kernel
function $K(\cdot)$.

Let $(\Omega, \mathcal{A} )$ be a measurable space and let
$\mathcal{F}$ be a uniformly bounded collection of measurable
functions.

\begin{definition}
$\mathcal{F}$ is a bounded measurable \textit{VC
class of functions with characteristics $A$ and $v$} if it is
separable and for
every probability measure $P$ on $(\Omega, \mathcal{A} )$ and any
$0<\epsilon <1$,
\begin{eqnarray}
N\left(\epsilon\|F \|_{L_{2}(P)}, \mathcal{F}, \|\cdot \|_{L_{2}(P)} \right) \leq \left(\frac{A}{\epsilon} \right)^{v},
\end{eqnarray}
where $F(x) = \sup_{f \in \mathcal{F}}|f(x)|$ and
$N(\epsilon, \mathcal{F}, \|\cdot \|_{L_{2}(P)})$
denotes the $\epsilon$-covering number of the metric space $(\Omega,
\|\cdot \|_{L_{2}(P)})$; that
is, the smallest number of balls of radius no larger than $\epsilon$
(in the norm $\|\cdot \|_{L_{2}(P)}$) needed to cover $\mathcal{F}$.
\label{def:vc}
\end{definition}

The one-dimensional density estimates
are constructed using a kernel $K$, and the two-dimensional estimates
are constructed using the product kernel
\begin{eqnarray}
K_{2}(x, y) = K(x)\cdot K(y).
\end{eqnarray}

\begin{assumption}\label{assump.kernel} 
The kernel $K$ satisfies the following properties.
\begin{itemize}
 \item[(K1)] $\ds \int K(u)\, du  =1$, $\ds \int_{-\infty}^\infty K^{2}(u)\,du < \infty$ and $\ds \sup_{u\in\mathbb{R}} K(u) \leq c$
     for some constant $c$.
\vskip5pt
   \item[(K2)] $K$ is a finite linear combination of functions $g$
     whose epigraphs
     $\text{epi}(g) = \{(s,u): g(s) \geq u \}$, can be represented as a
     finite number of Boolean operations (union and intersection) among sets of the form
     $\{(s,u): Q(s,u)\geq \phi(u) \}$, where $Q$ is a polynomial on
     $\mathbb{R}\times \mathbb{R}$ and $\phi$ is an arbitrary real
     function.
\vskip5pt
 \item[(K3)] $K$ has a compact support and for  any $\ell \geq 1$ and $1 \leq \ell' \leq \lfloor \beta \rfloor$
 \begin{equation}
 \int |t|^{\beta} \left| K(t)\right| dt < \infty, ~ \mathrm{and}~\int |K(t)|^{\ell} dt < \infty,~~\int t^{\ell'}K(t)dt = 0.
 \end{equation}
\end{itemize}
\end{assumption}

Assumptions (K1), (K2) and (K3) are mild.  As pointed out by
\cite{nolan:1987}, both the pyramid (truncated or not) kernel and the
boxcar kernel satisfy them.  It follows from (K2) that the classes
of functions
\begin{eqnarray}
\mathcal{F}_{1} &=& \left\{ \frac{1}{h_{1}}K\left( \frac{u - \cdot}{h_{1}}\right) : u \in\mathbb{R},\; h_{1} > 0 \right\} \label{eq.F1}\\
\mathcal{F}_{2} &=& \left\{ \frac{1}{h^{2}_{2}}K\left( \frac{u - \cdot}{h_{2}}\right)K\left(\frac{t - \cdot}{h_{2}}\right) : u,t \in\mathbb{R},\; h_{2} > 0 \right\} \label{eq.F2}
\end{eqnarray}
are bounded VC classes, in the sense
of Definition~\ref{def:vc}.
Assumption (K3)
essentially says that the kernel $K(\cdot)$ should be {\it
  $\beta$-valid}; see \cite{Tsybakov09} and Definition 6.1 in
\cite{Rig09} for further details about this
assumption.

We choose the bandwidths $h_1$ and $h_2$ used in the one-dimensional and
two-dimensional kernel density estimates to satisfy
\begin{eqnarray}
\label{eq:band1}
h_{1} &\asymp& \left(\frac{\log n}{n}\right)^{\frac{1}{1+2\beta}} \\
\label{eq:band2}
\ds h_{2} &\asymp& \left(\frac{\log n}{n}\right)^{\frac{1}{2+2\beta}}.
\end{eqnarray}
This choice of bandwidths ensures the optimal rate of convergence.

\subsection{Risk consistency}\label{subsec.risk}

Given the above assumptions, we first present a key lemma that
establishes the rates of convergence of 
bivariate and univariate kernel density estimates in the $\sup$ norm.
The proof of this and our other technical results are
provided in Appendix \ref{sec.proofs}.

\begin{lemma}
  \label{lemma.key} Under Assumptions \ref{assump.density}
  and~\ref{assump.kernel}, and choosing bandwidths satisfying
  \eqref{eq:band1} and \eqref{eq:band2}, the bivariate and univariate
  kernel density estimates $\hat{p}(x_{i}, x_{j})$ and
  $\hat{p}(x_{k})$ in \eqref{eq.bivariatekde} and
  \eqref{eq.univariatekde} satisfy
\begin{equation}
\max_{(i,j)\in \{1,\ldots, d \} \times \{1, \ldots, d \}}\sup_{(x_{i}, x_{j})\in\X_{i} \times \X_{j} } |\hat{p}(x_{i}, x_{j})  - \truep(x_{i}, x_{j})| = O_{P}\left(\sqrt{\frac{\log n + \log d}{n^{\beta/(1+\beta)}}} \right) \label{eq.bivariatesup}
\end{equation}
and
\begin{equation}
\max_{k \in \{1,\ldots, d\}}\sup_{x_{k}\in\X_{k}}|\hat{p}(x_{k}) - \truep(x_{k}) | =  O_{P}\left(\sqrt{\frac{\log n + \log d}{n^{2\beta/(1+2\beta)}}} \right). \label{eq.univariatesup}
\end{equation}
\end{lemma}

To describe the risk consistency result, let
$\mathcal{P}^{(d-1)}_{d} = \mathcal{P}_{d}$ be the family of densities
that are supported by forests with at most $d-1$ edges, as already
defined in \eqref{eq.Pd}.  For $0\leq k \leq d-1$, we define
$\mathcal{P}^{(k)}_{d}$ as the family of $d$-dimensional densities
that are supported by forests with at most $k$ edges. Then
\begin{equation}
\mathcal{P}^{(0)}_{d} \subset \mathcal{P}^{(1)}_{d} \subset \cdots \subset \mathcal{P}^{(d-1)}_{d}. \label{eq.nestedclass}
\end{equation}
Now, due to the nesting property \eqref{eq.nestedclass}, we have
\begin{equation}
\inf_{q_{F} \in \mathcal{P}^{(0)}_{d}} R(q_{F}) \geq  
\inf_{q_{F} \in \mathcal{P}^{(1)}_{d}} R(q_{F}) \geq \cdots \geq 
\inf_{q_{F} \in \mathcal{P}^{(d-1)}_{d}} R(q_{F}). 
\end{equation}

We first analyze the forest density estimator obtained using a
fixed number of edges $k < d$; specifically, consider stopping the
Chow-Liu algorithm in Stage 1 after $k$ iterations.  This is in
contrast to the algorithm described in \ref{subsubsec.stage2}, where
the pruned tree size is automatically determined on the held out data.
While this is not very realistic in applications, since the tuning
parameter $k$ is generally hard to choose, the analysis in this case
is simpler, and can be directly exploited to analyze the more
complicated data-dependent method.

\begin{theorem}[Risk consistency]\label{thm.persistency} 
  Let
  $\hat{p}_{\hat{F}^{(k)}_{d}}$ be the forest density
  estimate with $|E({\hat{F}^{(k)}_{d}})| = k$, obtained after the
  first $k$ iterations of the Chow-Liu algorithm, for some $k \in \{ 0, \ldots, d-1\}$.  Under
  Assumptions \ref{assump.density} and~\ref{assump.kernel},
  we have
\begin{equation}
  R(\hat{p}_{\hat{F}^{(k)}_{d}}) -\inf_{q_{F} \in \mathcal{P}^{(k)}_{d}} 
R(q_{F}) = O_{P}\left(k\sqrt{\frac{\log n + \log d}{n^{\beta/(1+\beta)}}} + d\sqrt{\frac{\log n + \log d}{n^{2\beta/(1+2\beta)}}} \right). \label{eq.riskrate}
\end{equation}
\end{theorem}

Note that this result allows the dimension $d$ to increase at a rate $ o\left(
\sqrt{ n ^{2\beta/(1+2\beta)} /\log n}  \right)$ and the number of
edges $k$ to increase at a rate $ o\left(
\sqrt{ n ^{\beta/(1+\beta)} /\log n}  \right)$, with the excess risk
still decreasing to zero asymptotically.

\comment{
It's possible that there exists a $k < d-1$ such that
that
\begin{equation}
\inf_{q_{F} \in \mathcal{P}^{(k)}_{d}} R(q_{F}) = \inf_{q_{F} \in \mathcal{P}^{(d-1)}_{d}} R(q_{F}). \label{eq.disconnecttree}
\end{equation}
In this case, we obtain the following corollary.

\begin{corollary}
Assume there exists some $k \in \{ 0, \ldots, d-2\}$, such that
\eqref{eq.disconnecttree} holds. Under the same assumptions in Theorem \ref{thm.persistency}, 
we have
\begin{equation}
R(\hat{p}_{\hat{F}^{(k)}_{d}}) -  R^{*} = O_{P}\left(k\sqrt{\frac{\log n + \log d}{n^{\beta/(1+\beta)}}} + d\sqrt{\frac{\log n + \log d}{n^{2\beta/(1+2\beta)}}} \right). 
\end{equation}
where $R^{*}$ is defined in \eqref{eq.oraclerisk}.
\end{corollary}
}

The above results can be used to prove a risk consistency result
for the data-dependent pruning method using the data-splitting scheme described in Section~\ref{subsubsec.stage2}.

\begin{theorem}\label{thm.randompersistency} 
Let $\hat{p}_{\hat{F}_{d}^{(\hat{k})}}$ be the forest  density
estimate using the data-dependent pruning method in
Section~\ref{subsubsec.stage2}, and let
$\hat{p}_{\hat{F}^{(k)}_{d}}$ be the estimate
with $|E({\hat{F}^{(k)}_{d}})| = k$ obtained after the
first $k$ iterations of the Chow-Liu algorithm.
Under Assumptions
  \ref{assump.density} and~\ref{assump.kernel}, we have
\begin{equation}
  R(\hat{p}_{\hat{F}^{(\hat{k})}_{d}}) - \min_{0\leq k \leq d-1}R(\hat{p}_{\hat{F}^{(k)}_{d}})
%\inf_{q_{F} \in \mathcal{P}^{(k)}_{d}} R(q_{F}) 
= O_{P}\left((k^* + \hat{k})\sqrt{\frac{\log n + \log d}{n^{\beta/(1+\beta)}}} + d\sqrt{\frac{\log n + \log d}{n^{2\beta/(1+2\beta)}}} \right) \label{eq.randomriskrate}
\end{equation}
where $k^*=\argmin_{0\leq k \leq d-1}R(\hat{p}_{\hat{F}^{(k)}_{d}})$.
\end{theorem}

The proof of this theorem is given in the appendix.
A parallel result can be obtained for the method described in
Section~\ref{sec:hoforest}, which builds the forest by running
Kruskal's algorithm on the heldout data.
\begin{theorem}
\label{thm.heldoutforest}
Let $\hat{F}_{n_2}$ be the forest obtained using Kruskal's algorithm on held-out data, and let $\hat{k} = |\hat{F}_{n_2}|$ be the number of edges in $\hat{F}_{n_2}$. Then
\begin{eqnarray}
R(\hat{p}_{\hat{F}_{n_2}}) - \min_{F \in \cF} R(\hat{p}_F) = O_{P}\left( (k^* +\hat{k}) \sqrt{\frac{\log n + \log d}{n^{\beta/(1+\beta)}}} + d \sqrt{\frac{\log n + \log d}{n^{2\beta/(1+2\beta)}}}\right)
\end{eqnarray}
where $k^* = |F^*|$ is the number of edges in the optimal forest $F^* = \argmin_{F \in \cF} R(\hat{p}_F)$.
\end{theorem}

\subsection{Structure selection consistency}\label{subsec.structure}

In this section, we provide conditions guaranteeing
that the procedure is structure selection consistent. Again, we do not
assume the true density $\truep(x)$ is consistent with a forest; rather,
we are interested in comparing the estimated forest structure
to the oracle forest which minimizes the risk.
In this way our result differs from that in \cite{Tan:09a}, although
there are similarities in the analysis.

By Proposition \ref{prop.oracle}, we can define
\begin{equation}
\truep_{F^{(k)}_{d}} = \argmin_{q_{F} \in \mathcal{P}^{(k)}_{d}} R(q_{F}). 
\end{equation}
Thus $F^{(k)}_{d}$ is the optimal forest within $\mathcal{P}^{(k)}_{d}$
that minimizes the negative log-likelihood loss. Let $\hat{F}^{(k)}_{d}$ be
the estimated forest structure, fixing the number of edges at $k$; we want to
study conditions under which
\begin{equation}
\mathbb{P}\left( \hat{F}^{(k)}_{d}= F^{(k)}_{d}\right) \rightarrow 1. 
\end{equation}

Let's first consider the population version of the algorithm---if the
algorithm cannot recover the best forest $F^{(k)}_{d}$ in this ideal
case, there is no hope for stable recovery in the data version. The
key observation is that the graph selected by the 
Chow-Liu algorithm only depends on the relative order
of the edges with respect to mutual information, not on the specific
mutual information values.  Let
\begin{equation}
\mathcal{E} = \biggl\{ \left\{(i,j), (k,\ell)\right\} : 
i<j ~\mathrm{and}~ k<\ell ,~j \neq \ell~\mathrm{and}~ i,j,k,\ell \in \{1,\ldots, d \} \biggr\}.
\end{equation}
The cardinality of $\mathcal{E}$ is  
\begin{equation}
|\mathcal{E}| =  O(d^{4}). 
\end{equation}
Let $e = (i,j)$ be an edge; the corresponding
mutual information associated with $e$ is denoted as $I_{e}$.  If
for all $(e,e')\in\mathcal{E}$, we have $I_{e} \neq
I_{e'}$, the population version of the Chow-Liu algorithm will always obtain
the unique solution $F^{(k)}_{d}$. However, this condition is, in a
sense, both too weak and too strong.  It is too weak because
the sample estimates of the mutual information values will only approximate
the population values, and could change the relative ordering of some edges.
However, the assumption is too strong because, in fact, the relative order of
many edge pairs might be
changed without affecting the graph selected by the algorithm.
For instance, when $k\geq 2$ and
$I_{e}$ and $I_{e'}$ are the largest two mutual
information values, it's guaranteed that $e$ and $e'$ will both be
included in the learned forest $F^{(k)}_{d}$ whether $I_{e} >
I_{e'}$ or $I_{e} < I_{e'}$.

Define the 
{\it crucial set} $\mathcal{J}\subset \mathcal{E}$ to be 
a set of pairs of edges $(e,e')$ such that $I_e \neq I_{e'}$ and
flipping the relative order of $I_{e}$
and $I_{e'}$ changes the learned forest structure in the
population Chow-Liu algorithm, with positive probability. Here, we
assume that the Chow-Liu algorithm randomly selects an edge when
a tie occurs. 

The cardinality $|\mathcal{J}|$ of the crucial set is a function of the true
density $\truep(x)$, and we can expect $|\mathcal{J} | \ll
|\mathcal{E}|$.  The next assumption provides a sufficient
condition for the two-stage procedure to be structure selection
consistent.

\begin{assumption}\label{assump.sparsistency}
Let the crucial set $\mathcal{J}$ be defined as before.
Suppose that
\begin{eqnarray}
\min_{\left( (i, j), (k, \ell)\right)\in \mathcal{J}}|I(X_{i}; X_{j}) - I(X_{k}; X_{\ell}) | \geq 2L_{n}
\end{eqnarray} 
where $L_{n} = \Omega\left(\ds \sqrt{\frac{\log n + \log d}{n^{\beta/(1+\beta)}}} \right).$
\end{assumption} 

This assumption is satisfied in many cases. For example, in a graph with 
population mutual informations differing by a constant,
the assumption holds.  Assumption
\ref{assump.sparsistency} is trivially satisfied if $\ds
\frac{n^{\beta/(1+\beta)}}{\log n + \log d}\rightarrow \infty.$  However, if two
  pairs of edges belonging $\mathcal{J}$  have the same marginal distributions, the assumption may fail.

\begin{theorem}[Structure selection consistency] \label{thm.sparsistency} 
Let $F^{(k)}_{d}$ be the optimal forest within
$\mathcal{P}^{(k)}_{d}$ that minimizes the negative log-likelihood
loss. Let $\hat{F}^{(k)}_{d}$ be the estimated forest with $|E_{\hat{F}^{(k)}_{d}}| =k$. Under Assumptions
\ref{assump.density}, \ref{assump.kernel}, and
\ref{assump.sparsistency}, we have
\begin{equation}
\mathbb{P}\left( \hat{F}^{(k)}_{d}=F^{(k)}_{d}\right) \rightarrow 1  \label{eq::sparsistent}
\end{equation}
as $n\to\infty$.
\end{theorem}

The proof shows that our method is structure selection consistent as long as
the dimension increases as $d = o\left(\exp(n^{\beta/(1+\beta)})
\right)$; in this case the error decreases at the rate
$o\left(\exp \left( 4\log d-{c} (\log n)^{\frac{1}{1+\beta}}\log d \right)\right)$.

\subsection{Estimation consistency}\label{subsec.estimate}

Estimation consistency can be easily established
using the structure selection consistency result above.
Define the event $\mathcal{M}_k = \{ \hat{F}^{(k)}_{d} = F^{(k)}_{d}
\}$.  Theorem \ref{thm.sparsistency} shows that 
$\mathbb{P}(\mathcal{M}_k^{c})\rightarrow 0$ as $n$ goes to infinity.

\begin{lemma}\label{lemma.equiv}
Let  $\hat{p}_{\hat{F}^{(k)}_{d}}$ be the forest-based kernel density
estimate for some fixed $k \in \{ 0, \ldots, d-1\}$, and let
\begin{equation}
\truep_{F^{(k)}_{d}} =  \argmin_{q_{F} \in \mathcal{P}^{(k)}_{d}} R(q_{F}). 
\end{equation}
Under the assumptions of Theorem \ref{thm.sparsistency},  
 \begin{equation}
 D( \truep_{F^{(k)}_{d}}  \div  \hat{p}_{\hat{F}^{(k)}_{d}} ) = R(\hat{p}_{\hat{F}^{(k)}_{d}}) - R(\truep_{F^{(k)}_{d}} )
 \end{equation}
on the event $\mathcal{M}_k$.
\end{lemma}

\begin{proof}
According to \cite{Bach03},  for a given forest $F$ and a target distribution $\truep(x)$, 
\begin{equation}
D(\truep\div q_{F}) = D(\truep \div \truep_{F}) + D(\truep_{F} \div q_{F}) \label{eq.bach}
\end{equation}
for all distributions $q_{F}$ that are supported by $F$.
We further have
\begin{equation}
D(\truep \div q) = \int_{\X} \truep(x)\log \truep(x) - \int_{\X}
              \truep(x)\log q(x) dx  = 
\int_{\X} \truep(x)\log \truep(x)dx + R(q) \label{eq.equivresult}
\end{equation}
for any distribution $q$.
Using \eqref{eq.bach} and  \eqref{eq.equivresult}, and
conditioning on the event $\mathcal{M}_k$,  
we have
\begin{eqnarray}
D( \truep_{F^{(k)}_{d}}  \div \hat{p}_{\hat{F}^{(k)}_{d}} )  &= &
D(\truep\div \hat{p}_{\hat{F}^{(k)}_{d}}) -  D(\truep \div \truep_{F^{(k)}_{d}} )  \\
& = & \int_{\X}\truep(x)\log \truep (x)dx +
R(\hat{p}_{\hat{F}^{(k)}_{d}} ) - \int_{\X}\truep(x)\log \truep(x) dx
-R(\truep_{F^{(k)}_{d}})   \nonumber \\
& = & R(\hat{p}_{\hat{F}^{(k)}_{d}}) - R(\truep_{F^{(k)}_{d}} ), \nonumber
\end{eqnarray}
which gives the desired result.
\end{proof}

The above lemma combined with Theorem \ref{thm.persistency} allows us
to obtain the following estimation consistency result, the proof of
which is omitted.

\begin{corollary}[Estimation consistency]
Under Assumptions \ref{assump.density}, \ref{assump.kernel}, and \ref{assump.sparsistency}, we have
\begin{equation}
D( \truep_{F^{(k)}_{d}}  
\div \hat{p}_{\hat{F}^{(k)}_{d}} )= O_{P}\left(k\sqrt{\frac{\log n + \log d}{n^{\beta/(1+\beta)}}} + d\sqrt{\frac{\log n + \log d}{n^{2\beta/(1+2\beta)}}}  \right). \label{eq.riskrate2}
\end{equation}
\end{corollary}

\section{Tree Restricted Forests}
\label{sec.restricted}

We now turn to the problem of estimating forests with restricted tree
sizes.  As discussed in the introduction, clustering problems motivate
the goal of constructing forest structured density estimators where each
connected component has a restricted number of edges.  But estimating
restricted tree size forests can also be useful in model selection for
the purpose of risk minimization, since the maximum subtree size can be
viewed as an additional complexity parameter.

\begin{definition}
\label{def.krf}
A
\textit{$\k$-restricted forest} of a graph $G$ is a subgraph $F_\k$
such that
\begin{packed_enum}
\item $F_\k$ is the disjoint union of connected components $\{T_1, ..., T_m\}$, each of which is a tree;
\item $|T_i| \leq \k$ for each $i\leq m$, where $|T_i|$ denotes the number
  of edges in the $i$th component.
\end{packed_enum}
Given a weight $w_e$ assigned to each edge of $G$, an 
\textit{optimal $\k$-restricted forest} $F^*_\k$ satisfies
\begin{equation}
w(F_\k^*) = \max_{F \in \F_\k(G)} w(F)
\end{equation}
where $w(F) = \sum_{e \in F} w_e $ is the weight of a forest $F$ and 
$\F_{\k}(G)$ denotes the collection of all $\k$-restricted
forests of $G$.
\end{definition}

For $t=1$, the problem is maximum weighted matching. However, for $t
\geq 7$, we show that finding an optimal $\k$-restricted forest is an
NP-hard problem; however, this problem appears not to have been
previously studied. Our reduction is from Exact 3-Cover (X3C), shown to
be NP-complete by \cite{Garey:79}). In X3C, we are given a set $X$, a
family $\cS$ of 3-element subsets of $X$, and we must choose a subfamily
of disjoint 3-element subsets to cover $X$. Our reduction constructs a
graph with special tree-shaped subgraphs called \textit{gadgets}, such
that each gadget corresponds to a 3-element subset in $\cS$.  We show
that finding a maximum weight $\k$-restricted forest on this graph would
allow us to then recover a solution to X3C by analyzing how the optimal
forest must partition each of the gadgets.

Given the NP-hardness for finding optimal $\k$-restricted forest, it is
of interest to study approximation algorithms for the problem. Our first
algorithm is Algorithm~\ref{alg.approx}, which runs in two stages.  In
the first stage, a forest is greedily constructed in such a way that
each node has degree no larger than $t$ (a property that is satisfied by
all $t$-restricted forests). However, the trees in the forest may have
more than $t$ edges; hence, in the second stage, each tree in the forest
is partitioned in an optimal way by removing edges, resulting in a
collection of trees, each of which has size at most $t$.  The second
stage employs a procedure we call {\tt TreePartition} that takes a tree
and returns the optimal $\k$-restricted subforest.  {\tt TreePartition}
is a divide-and-conquer procedure of \cite{Lukes:74} that finds a
carefully chosen set of forest partitions for each child subtree. It
then merges these sets with the parent node one subtree at a time.  The
details of the {\tt TreePartition} procedure are given in Appendix
\ref{sec.proofs}.

\begin{algorithm}[t]
\caption{\ \ Approximate Max Weight $\k$-Restricted Forest}
\begin{algorithmic}[1]
\vskip5pt
\STATE \textbf{Input} graph $G$ with positive edge weights, and positive
  integer $\k \geq 2$.\\[5pt]
\STATE Sort edges in decreasing order of weight. \\[5pt]
\STATE Greedily add edges in decreasing order of weight such that
 \begin{enumerate}
  \item[(a)] the degree of any node is at most $\k+1$;
  \item[(b)] no cycles are formed.
 \end{enumerate}
 The resulting forest is $F' = \{T_{1},T_{2},\ldots, T_{m}\}$.
\\[5pt]
\STATE \textbf{Output} $F_t = \cup_j \mbox{\tt TreePartition}(T_j,t)$.
\end{algorithmic}
\label{alg.approx}
\end{algorithm}

\begin{theorem}
\label{thm.approx}
Let $F_t$ be the output of Algorithm~\ref{alg.approx}, 
and let $F^*_t$ be the optimal $\k$-restricted forest.  
Then $w(F_t) \geq \ds \frac{1}{4} w(F^*_t)$.
\end{theorem}

In Appendix \ref{sec.approxproof}, we present a proof of the above
result. In that section, we also present an improved approximation
algorithm, one based on solving linear programs, that finds a
$\k$-restricted forest $F_t'$ such that $w(F_t') \geq \frac12 w(F_t^*)$.

%fomalhaut

\subsection{Pruning Based on $\k$-Restricted Forests}
\label{subsec.cvapprox}

For a given $\k$, after producing an approximate maximum weight $\k$-restricted
forest $\hat F_\k$ using $\cD_1$, we prune away edges using $\cD_2$.  To do so, we first construct a
new set of univariate and bivariate kernel density estimates using $\cD_2$, as before,
$\hat{p}_{n_2}(x_i)$ and $\hat{p}_{n_2}(x_i,x_j)$. Recall that we 
define the ``cross-entropies'' of the kernel density estimates
$\hat{p}_{n_1}$ for each pair of variables as
\begin{eqnarray}
\hat{I}_{n_2,n_1}(X_i,X_j) &=& \int \hat{p}_{n_2}(x_i,x_j) \log
\frac{\hat{p}_{n_1}(x_i,x_j)}{\hat{p}_{n_1}(x_i)\hat{p}_{n_1}(x_j)} \,
dx_i \, dx_j\\
&\approx& \frac{1}{m^{2}}\sum_{k=1}^{m}\sum_{\ell=1}^{m}\hat{p}_{n_2}(x_{ki},x_{\ell j}) 
\log \frac{\hat{p}_{n_1}(x_{ki}, x_{\ell j}) }{\hat{p}_{n_1}(x_{ki})\,
  \hat{p}_{n_1}(x_{\ell j})}.  \label{eq.empiricalXMI}
\end{eqnarray}
We then eliminate all edges $(i,j)$ in $\hat F_\k$ for which
$\hat{I}_{n_2,n_1}(X_i,X_j) \leq 0$.  For notational simplicity, we denote the
resulting pruned forest again by $\hat F_\k$.

To estimate the risk, we simply use $\hat{R}_{n_2}(\hat p_{\hat
  F_\k})$ as defined in \eqref{eq.heldoutrisk}, and select the forest $\hat F_{\hat
  \k}$ according to 
\begin{equation}
\hat \k = \argmin_{0\leq \k\leq d-1} \hat{R}_{n_2}(\hat p_{\hat F_\k}).
\end{equation}
The resulting procedure is summarized in Algorithm~\ref{alg.krestrict}.

Using the approximation guarantee and our previous analysis, we have
that the population weights of the approximate $t$-restricted forest
and the optimal forest satisfy the following inequality.
We state the result for a general $c$-approximation algorithm; for
the algorithm given above, $c=4$, but tighter approximations are possible.

\begin{theorem}
\label{thm.approxpersist}
Assume the conditions of Theorem \ref{thm.persistency}. For
$\k\geq 2$, let $\hat F_t$ be the forest constructed using
a $c$-approximation algorithm, and let $F_t^*$ be the optimal
forest; both constructed with respect to finite sample edge weights $\hat w_{n_1} =
\hat I_{n_1}$.   Then
\begin{equation}
w(\hat F_t)  \geq \frac{1}{c} w(F_t^*) + O_{P}\left((k^* + \hat{k})\sqrt{\frac{\log n + \log
      d}{n^{\beta/(1+\beta)}}} \right)
\end{equation}
where $\hat k$ and $k^*$ are the number of edges in $\hat F_t$ and
$F^*_t$, respectively, and $w$ denotes the population weights, given 
by the mutual information.
\end{theorem}

As seen below, although the approximation algorithm has weaker
theoretical guarantees, it out-performs other approaches in
experiments.

%\begin{theorem}
%Assume the conditions of Theorem \ref{thm.persistency} holds. Let $F_A$ be the $k$-restricted %forest that minimizes heldout risk. Let $F^*$ be the forest that 
%\end{theorem}

%% betelgeuse
\begin{algorithm}[t]
\caption{\ \ $\k$-Restricted Forest Density Estimation}
\vskip5pt
\begin{algorithmic}[1]
\STATE Divide data into two halves $\cD_1$ and $\cD_2$.
\ss\STATE Compute kernel density estimators $\hat{p}_{n_1}$ and
 $\hat{p}_{n2}$ for all pairs and single variable marginals.
\ss\STATE For all pairs $(i,j)$ compute $\disps \hat I_{n_1}(X_i,X_j)$
according to \eqref{eq.empiricalMI} and $\disps\hat I_{n_2,
  n_1}(X_i,X_j)$ according to \eqref{eq.empiricalXMI}.
\ss\STATE For $\k = 0,\ldots, \k_{\text{final}}$ where $\k_{\text{final}}$ is chosen
based on the application
 		\begin{enumerate}
 		\item Compute or approximate (for
                  $\k \geq 2$) the
                  optimal $\k$-restricted forest $\hat F_\k$
                  using $\hat I_{n_1}$ as edge weights. 
 		\item Prune $\hat F_t$ to eliminate all edges with
                  negative weights $\hat I_{n_2,n_1}$.
 		\end{enumerate}
\ss\STATE Among all pruned forests $\phat_{F^\k}$, select
$\hat \k = \argmin_{0\leq \k\leq \k_{\text{final}}} \hat{R}_{n_2}(\hat p_{\hat F_\k})$.
\end{algorithmic}
\label{alg.krestrict}
\end{algorithm}

\section{Experimental Results}
\label{sec.experiments}

In this section, we report numerical results on both synthetic
datasets and microarray data.  We mainly compare
the forest density estimator with sparse Gaussian graphical models,
fitting a multivariate Gaussian with a sparse inverse covariance
matrix.  The sparse Gaussian models are estimated using the graphical
lasso algorithm (glasso) of \cite{FHT:07}, which is a refined version
of an algorithm first derived by \cite{Banerjee:08}.  Since the glasso
typically results in a large parameter bias as a consequence of the
$\ell_1$ regularization, we also compare with a method that we call
the \textit{refit glasso}, which is a two-step procedure---in the
first step, a sparse inverse covariance matrix is obtained by the
glasso; in the second step, a Gaussian model is refit without $\ell_1$
regularization, but enforcing the sparsity pattern obtained in the
first step.

To quantitatively compare the performance of these estimators, we calculate the log-likelihood of
all methods on a held-out dataset $\mathcal{D}_{2}$. With $\hat{\mu}_{n_1}$
and $\hat{\Omega}_{n_1}$ denoting the estimates from the Gaussian model, the
held-out log-likelihood can be explicitly
evaluated as
\begin{equation}
  \ell_{\rm gauss} = -\frac{1}{n_{2}}\sum_{s\in \mathcal{D}_{2}}
 \left\{ \frac{1}{2} (X^{(s)} - \hat{\mu}_{n_1})^{F}
\hat{\Omega}_{n_1}(X^{(s)} - \hat{\mu}_{n_1}) + \frac{1}{2}\log\left(\frac{|\hat{\Omega}_{n_1}|}{(2\pi)^{d}} \right)\right\}.
\end{equation}
For a given tree structure $\hat{F}$,  the held-out log-likelihood for the forest  density estimator is
\begin{eqnarray}
\ell_{\rm fde} = \frac{1}{n_{2}}\sum_{s \in\mathcal{D}_{2}}\log \left( \prod_{(i,j)\in E({\hat{F}})}\frac{ \hat{p}_{n_1}(X^{(s)}_{i}, X^{(s)}_{j})}{\hat{p}_{n_1}(X^{(s)}_{i})\hat{p}_{n_1}(X^{(s)}_{j})} \prod_{k\in V_{\hat{F}}}\hat{p}_{n_1}(X^{(s)}_{k}) \right), 
\end{eqnarray}
where $\hat{p}_{n_1}(\cdot)$ are the corresponding kernel
density estimates using the plug-in bandwidths.  

Since the held-out
log-likelihood of the forest  density estimator is indexed
by the number of edges included in the tree, while the held-out
log-likelihoods of the glasso and the refit glasso are indexed by a
continuously varying regularization parameter, we need to find a way
to calibrate them. To address this issue, we plot the held-out
log-likelihood of the forest  density estimator as a step
function indexed by the tree size.  We then run the full path of the
glasso and discretize it according to the corresponding sparsity
level, i.e., how many edges are selected for each value of the regularization
parameter.  The size of the forest  density estimator
and the sparsity level of the glasso (and the refit glasso) can then be
aligned for a fair comparison.

\subsection{Synthetic data}

We use a procedure to generate high dimensional Gaussian and
non-Gaussian data which are consistent with an undirected graph. 
We generate high dimensional
graphs that contain cycles, and so are not forests.  In dimension $d=100$, we
sample $n_{1}= n_{2} = 400$ data points from a multivariate Gaussian
distribution with mean vector $\mu = (0.5, \ldots, 0.5) $ and 
inverse covariance matrix $\Omega$. The diagonal elements of $\Omega$
are all 62.  We then randomly generate many connected subgraphs containing no more
than eight nodes each, and set the corresponding non-diagonal elements in
$\Omega$ at random, drawing values uniformly from $-30$ to $-10$.  To obtain non-Gaussian data,
we simply transform each dimension of the data by its empirical
distribution function; such a transformation
preserves the graph structure but the joint distribution is no longer
Gaussian (see \emcite{npn:09}).

To calculate the pairwise mutual information $\hat{I}(X_{i}; X_{j})$,
we need to numerically evaluate two-dimensional integrals.  We first
rescale the data into $[0,1]^{d}$ and calculate the kernel density
estimates on a grid of points; we choose $m=128$ evaluation points
$x^{(1)}_{i} < x^{(2)}_{i} < \cdots < x^{(m)}_{i}$ for each dimension $i$,
and then evaluate the bivariate and the univariate kernel density
estimates on this grid.

There are three different kernel density estimates that we use---the
bivariate kde, the univariate kde, and the marginalized bivariate kde.
Specifically, the bivariate kernel density estimate on $x_{i}, x_{j}$
based on the observations $\{X^{(s)}_{i}, X^{(s)}_{j}\}_{s
  \in\mathcal{D}_{1}}$ is defined as
\begin{equation}
\hat{p}(x_{i}, x_{j}) = \frac{1}{n_{1}}\sum_{s \in \mathcal{D}_{1}} \frac{1}{h_{2i}h_{2j}}  K\left( \frac{X^{(s)}_{i} - x_{i}}{h_{2i}} \right)K\left( \frac{X^{(s)}_{j} - x_{j}}{h_{2j}} \right), 
\end{equation}
using a product kernel.  The bandwidths $h_{2i}, h_{2j}$ are chosen as
\begin{equation}
h_{2k} = 1.06\cdot \min \left\{\hat{\sigma}_{k},
  \frac{\hat{q}_{k,0.75} - \hat{q}_{k,0.25}}{1.34}  \right\}\cdot
n^{-1/(2\beta+2)},
\label{eq:bandseta}
\end{equation}
where $\hat{\sigma}_{k}$ is the sample standard deviation of
$\{X^{(s)}_{k}\}_{s \in\mathcal{D}_{1}}$ and $\hat{q}_{k, 0.75}$,
$\hat{q}_{k, 0.25}$ are the $75\%$ and $25\%$ sample quantiles of
$\{X^{(s)}_{k}\}_{s \in\mathcal{D}_{1}}$.  

In all the experiments, we set $\beta=2$, such a choice of $\beta$ and the ``plug-in'' bandwidth $h_{2k}$ (and $h_{1k}$ in the following) is a very common practice in nonparametric Statistics.  For more details, see \cite{Fan:Gijb:1996} and \cite{Tsybakov09}.

Given an evaluation point $x_{k}$, the univariate kernel density
estimate $\hat{p}(x_{k})$ based on the observations $\{X^{(s)}_{k}\}_{s \in\mathcal{D}_{1}}$ is defined as
\begin{equation}
\hat{p}(x_{k}) = \frac{1}{n_{1}}\sum_{s \in \mathcal{D}_{1}} \frac{1}{h_{1k}}  K\left( \frac{X^{(s)}_{k} - x_{k}}{h_{1k}} \right),  
\end{equation}  
where $h_{1k}>0$ is defined as
\begin{equation}
h_{1k} = 1.06\cdot \min \left\{\hat{\sigma}_{k},
  \frac{\hat{q}_{k,0.75} - \hat{q}_{k,0.25}}{1.34}  \right\}\cdot
n^{-1/(2\beta+1)}. 
\label{eq:bandsetb}
\end{equation}
Finally, the marginal univariate kernel density estimate
$\hat{p}_{M}(x_{k})$ based on the observations $\{ X^{(s)}_{k}\}_{s
  \in\mathcal{D}_{1}}$ is defined by integrating the irrelevant
dimension out of the bivariate kernel density estimates
$\hat{p}(x_{j}, x_{k})$ on the unit square $[0,1]^{2}$.  Thus,
\begin{equation}
\hat{p}_{M}(x_{k}) = \frac{1}{m-1}\sum_{\ell=1}^{m} \hat{p}(x^{(\ell)}_{j}, x_{k}). 
\end{equation}  
With the above definitions of the bivariate and univariate
kernel density estimates, we consider estimating the mutual 
information $I(X_{i}; X_{j})$ in three
different ways, depending on which estimates for the
univariate densities are employed.

\begin{eqnarray}
\hat{I}_{\mathrm{fast}}(X_{i}, X_{j}) &=&
\frac{1}{(m-1)^{2}}\sum_{k'=1}^{m}\sum_{\ell'=1}^{m}\hat{p}(x^{(k')}_{i},
x^{(\ell')}_{j}) \log \hat{p}(x^{(k')}_{i}, x^{(\ell')}_{j}) \;- \\
&& \frac{1}{m-1}\sum_{k'=1}^{m} \hat{p} (x^{(k')}_{i})\log
\hat{p} (x^{(k')}_{i})  -\frac{1}{m-1}\sum_{\ell'=1}^{m} \hat{p}
(x^{(\ell')}_{j})\log  \hat{p}
(x^{(\ell')}_{j})  \label{eq.empiricalMI.fast}  \nonumber \\[5pt]
\hat{I}_{\mathrm{medium}}(X_{i}, X_{j}) &=&
\frac{1}{(m-1)^{2}}\sum_{k'=1}^{m}\sum_{\ell'=1}^{m}\hat{p}(x^{(k')}_{i},
x^{(\ell')}_{j}) \log \frac{\hat{p}(x^{(k')}_{i}, x^{(\ell')}_{j})
}{\hat{p}(x^{(k')}_{i})\,
  \hat{p}(x^{(\ell')}_{j})}.  \label{eq.empiricalMI.medium} \\[5pt]
\hat{I}_{\mathrm{slow}}(X_{i}, X_{j}) &=&
\frac{1}{(m-1)^{2}}\sum_{k'=1}^{m}\sum_{\ell'=1}^{m}\hat{p}(x^{(k')}_{i},
x^{(\ell')}_{j}) \log \hat{p}(x^{(k')}_{i}, x^{(\ell')}_{j}) \; -\\
&& \frac{1}{m-1}\sum_{k'=1}^{m} \hat{p}_{M} (x^{(k')}_{i})\log
\hat{p}_{M} (x^{(k')}_{i})  -\frac{1}{m-1}\sum_{\ell'=1}^{m} \hat{p}
_{M}(x^{(\ell')}_{j})\log  \hat{p}_{M} (x^{(\ell')}_{j}) .  \nonumber \label{eq.empiricalMI.slow}
\end{eqnarray}

The terms ``fast,'' ``medium'' and ``slow'' refer to the theoretical
statistical rates of convergence of the estimators.  The ``fast''
estimate uses one-dimensional univariate kernel density estimators
wherever possible.   The ``medium'' estimate uses the one-dimensional
kernel density estimates in the denominator of
$p(x_i,x_j)/(p(x_i)p(x_j)$, but averages with respect
to the bivariate density.  Finally, the ``slow'' estimate
marginalizes the bivariate densities to estimate the univariate
densities.   While the rate of convergence is the two-dimensional
rate, the ``slow'' estimate ensures
the  consistency of the bivariate and univariate densities.

\begin{figure}[htp!]
\begin{center}
\def\hs{\hskip-6pt}
\begin{tabular}{ccc}
%\\[-20pt]
%\hs\includegraphics[width=.32\textwidth=-90,angle=0]{fde-arxiv/figs/I12} &
%\hs\includegraphics[width=.32\textwidth=-90,angle=0]{fde-arxiv/figs/I13} &
%\hs\includegraphics[width=.32\textwidth=-90,angle=0]{fde-arxiv/figs/I14} \\[-10pt]
%\scriptsize \bf (1,5) & \scriptsize \bf (2,3) &\scriptsize \bf (2,4) \\[-25pt]
\hs\includegraphics[width=.32\textwidth=-90,angle=0]{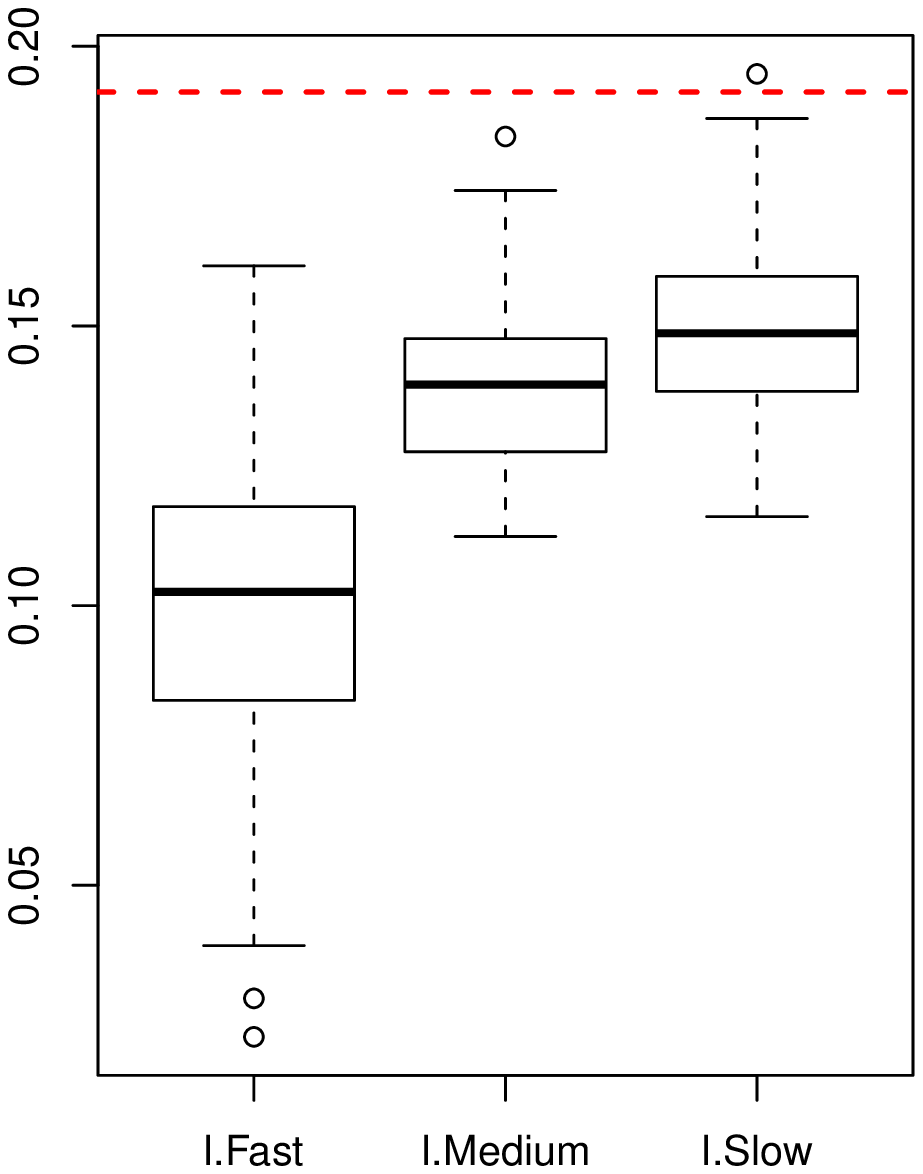} &
\hs\includegraphics[width=.32\textwidth=-90,angle=0]{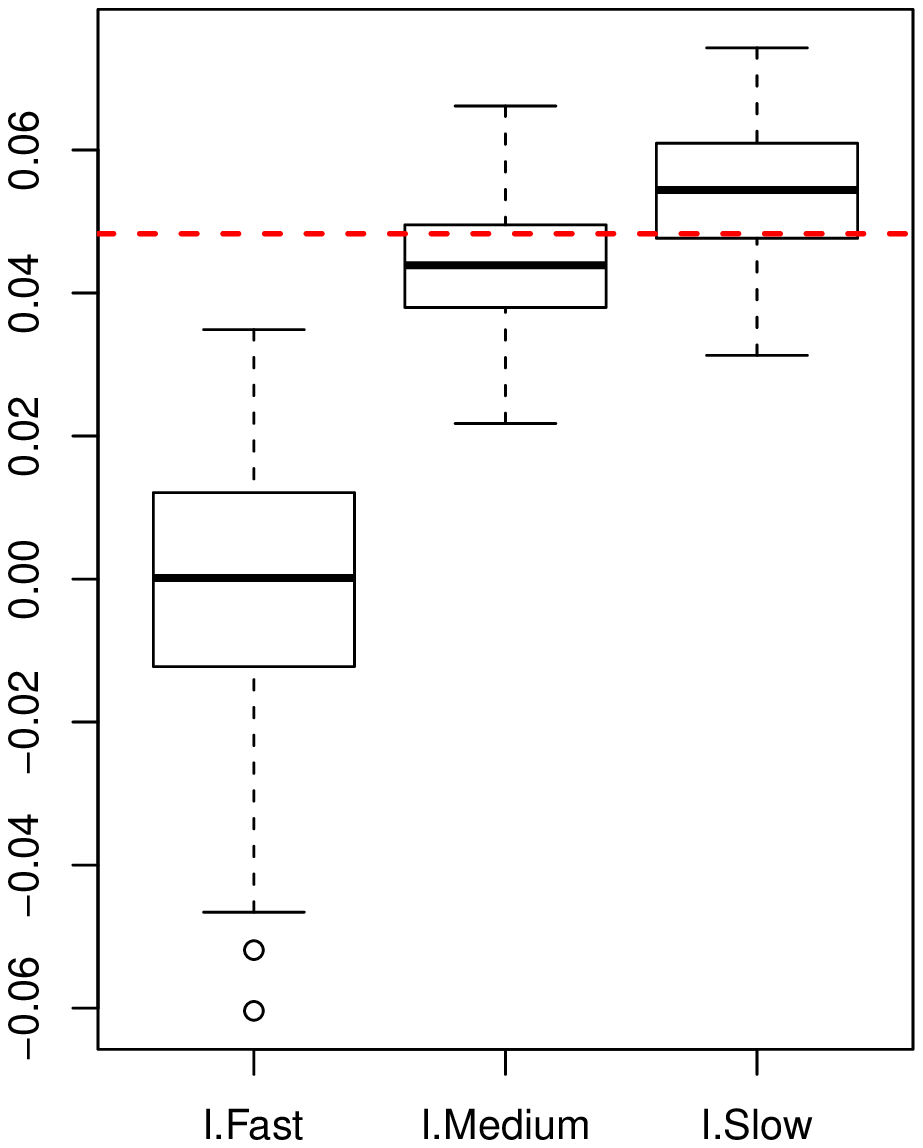} &
\hs\includegraphics[width=.32\textwidth=-90,angle=0]{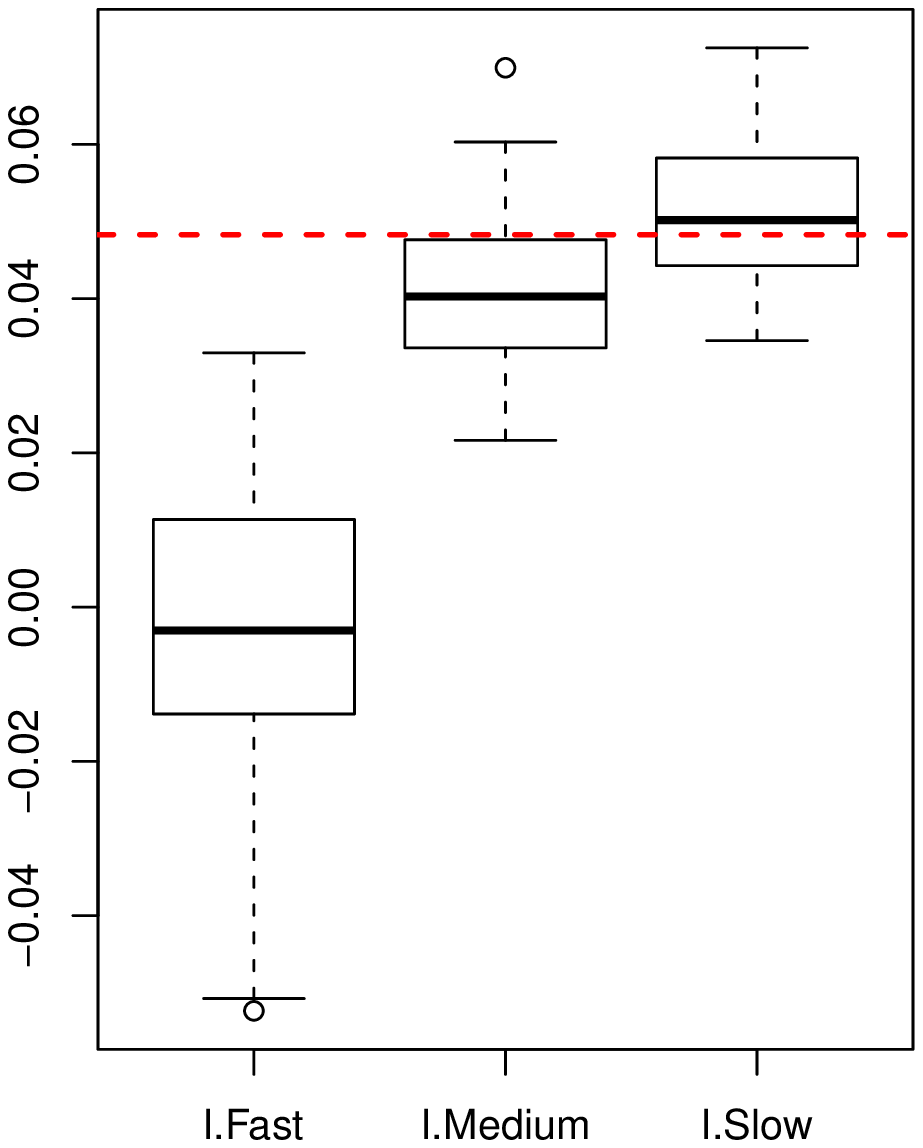} \\[-20pt]
\end{tabular}
\end{center}
\caption{\small (Gaussian example) Boxplots of 
$\hat{I}_{\rm fast}$, $\hat{I}_{\rm medium}$, and $\hat{I}_{\rm slow}$ on 
three different pairs of variables. The red-dashed horizontal lines represent the population values.}
\label{fig.MIcompare}
\end{figure}

\begin{figure}[htp!]
\begin{center}
\begin{tabular}{cccc}
\multicolumn{2}{c}{$(X_1,X_5)$} & \multicolumn{2}{c}{\hskip-1cm$(X_2,X_4)$}\\[-30pt]
\vspace{-0.5cm}
\hskip-1.2cm\includegraphics[width=.35\textwidth]{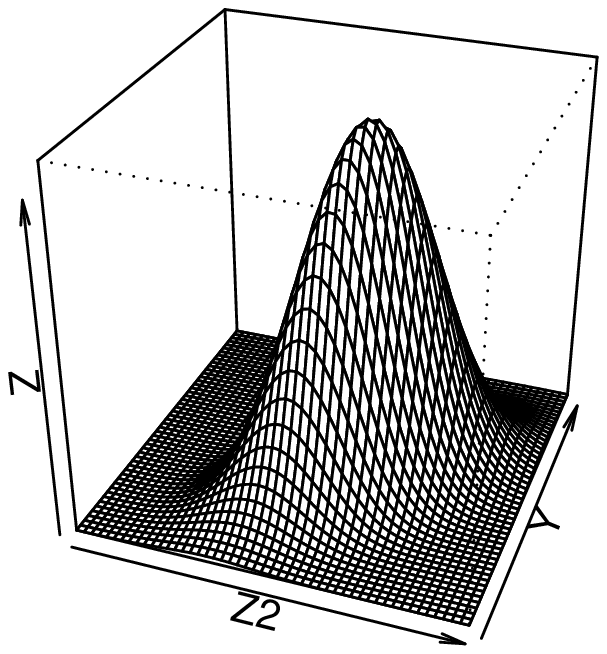} &
\hskip-2cm\includegraphics[width=.35\textwidth]{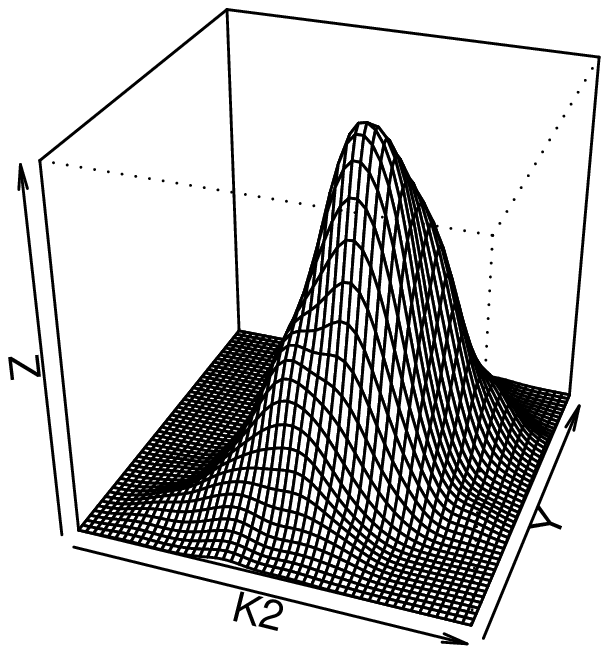} &
\hskip-2cm\includegraphics[width=.35\textwidth]{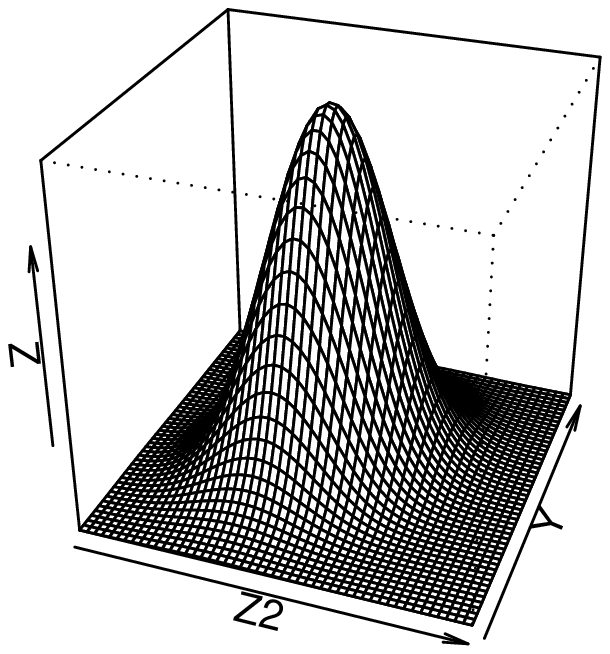} &
\hskip-2cm\includegraphics[width=.35\textwidth]{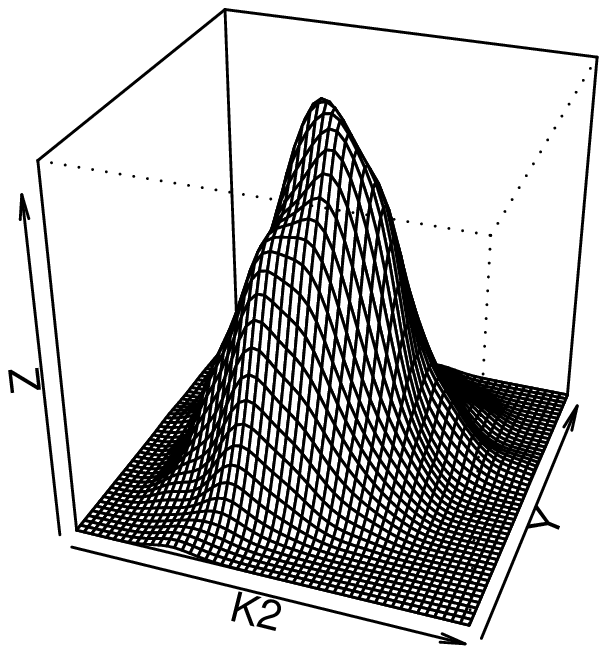} \\[-30pt]
\vspace{-0cm}
\hskip-1.2cm\includegraphics[width=.28\textwidth]{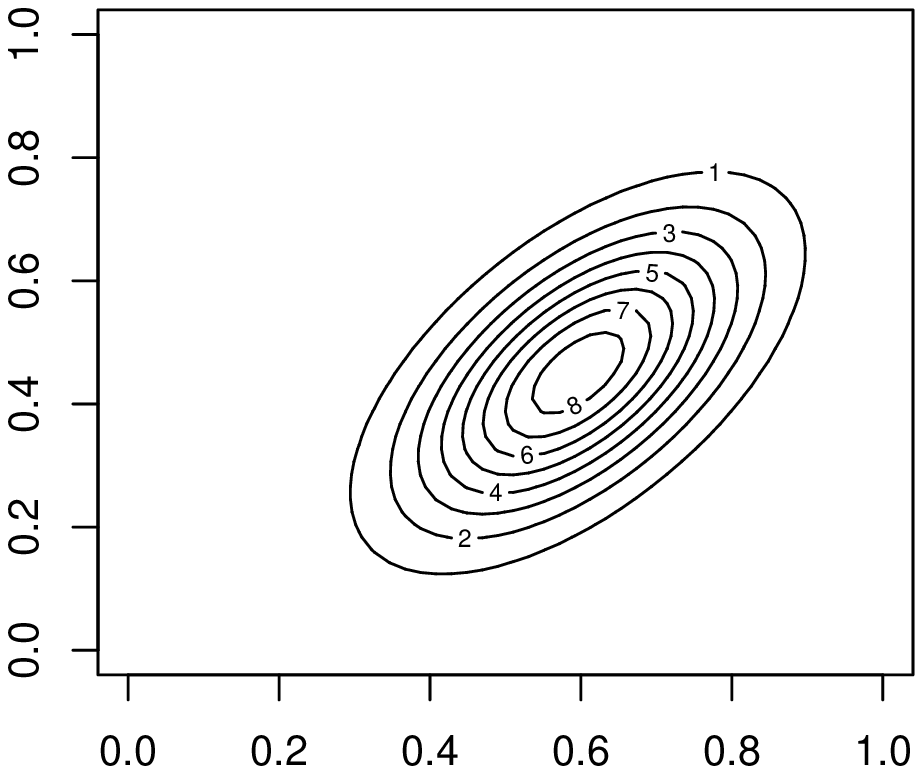} &
\hskip-2cm\includegraphics[width=.28\textwidth]{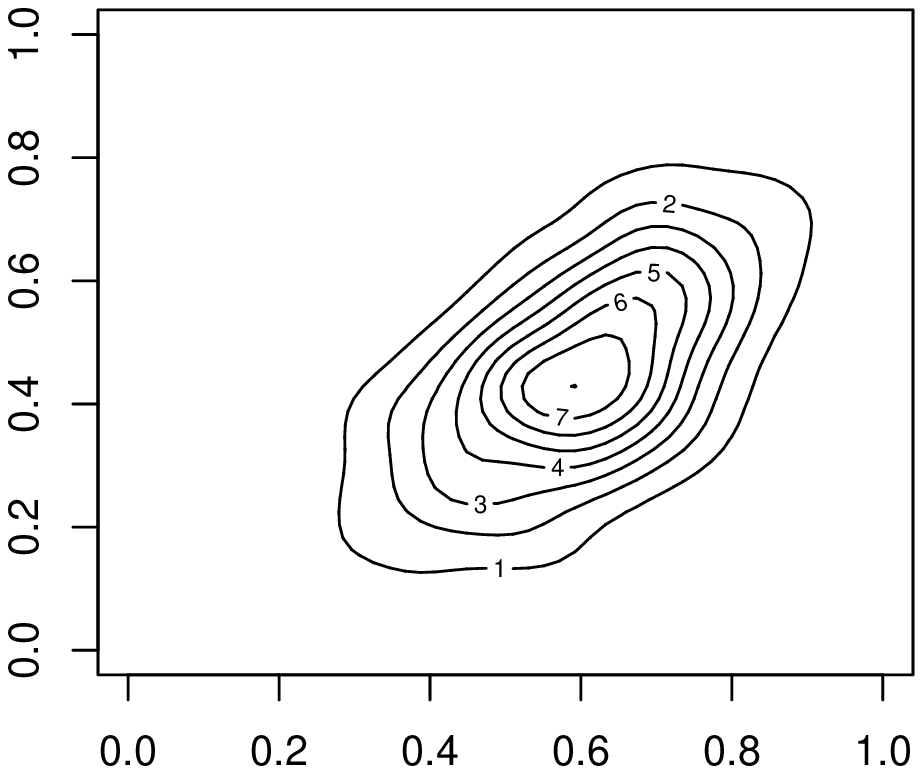} &
\hskip-2cm\includegraphics[width=.28\textwidth]{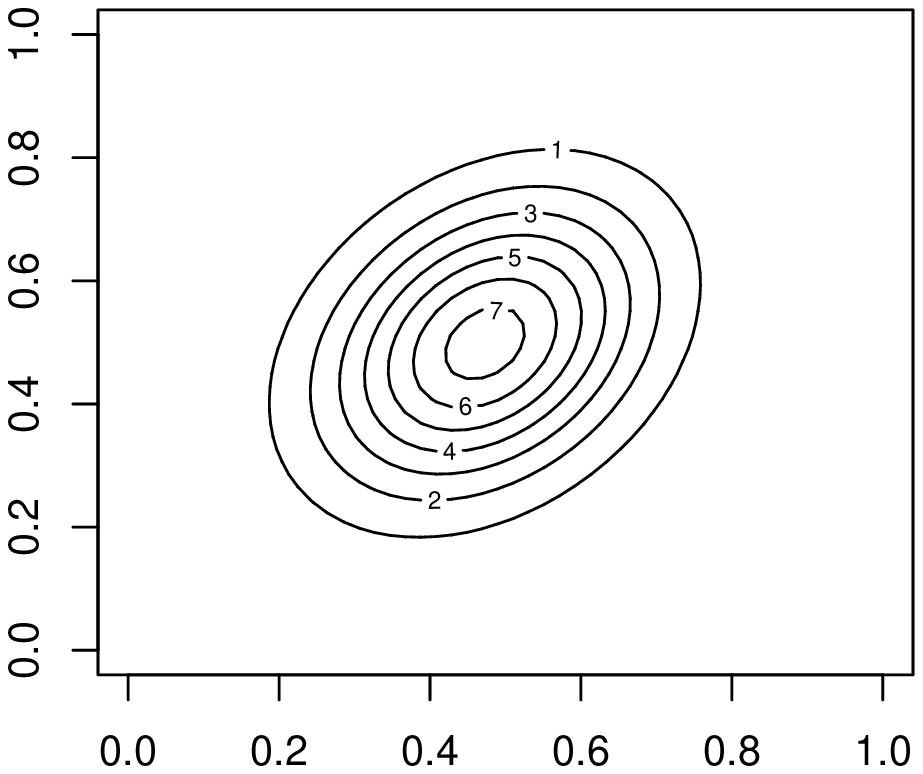} &
\hskip-2cm\includegraphics[width=.28\textwidth]{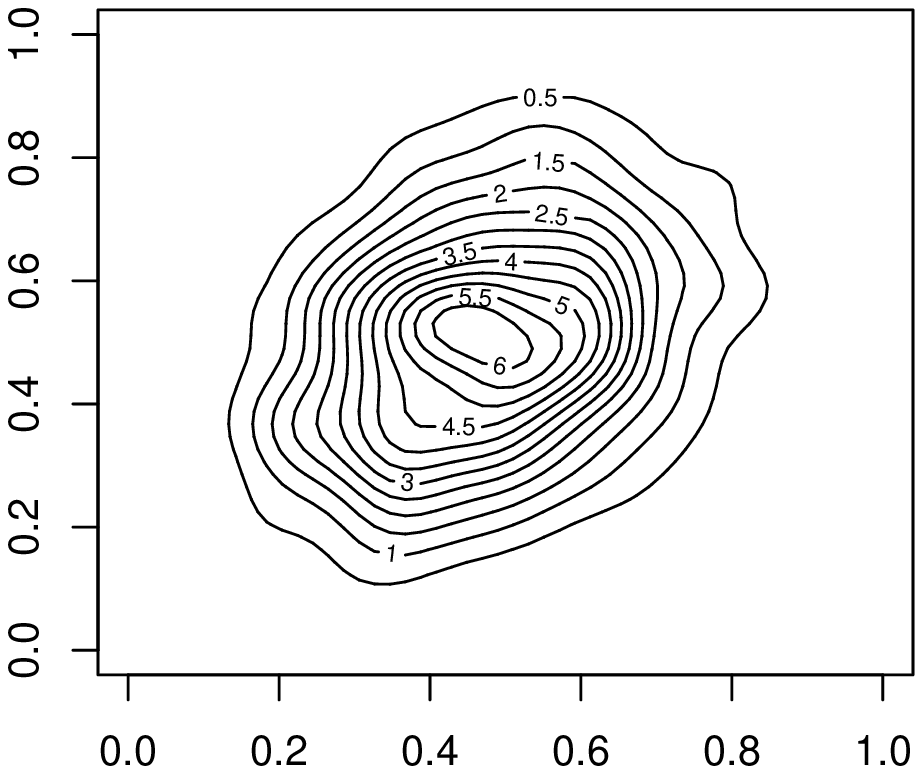} \\[-15pt]
\hskip-1.2cm\small Gaussian & 
\hskip-2cm\small kernel & 
\hskip-2cm\small Gaussian & 
\hskip-2cm\small kernel \\
\end{tabular}
\end{center}
\caption{ Perspective and contour plots of the
  bivariate Gaussian fits vs.~the kernel density estimates
  for two edges of a Gaussian graphical model.}
\label{fig:density2d}
\vskip10pt
\end{figure}

Figure \ref{fig.MIcompare} compares $\hat{I}_{\rm fast}$,
$\hat{I}_{\rm medium}$, and $\hat{I}_{\rm slow}$ on different pairs of
variables.  The boxplots are based on 100 trials. Compared to the
ground truth, which can be computed exactly in the Gaussian case, we
see that the performance of $\hat{I}_{\rm medium}$ and $\hat{I}_{\rm
  slow}$ is better than that of $\hat{I}_{\rm fast}$.  This is due to
the fact that simply replacing the population density with a
``plug-in'' version can lead to biased estimates; in fact,
$\hat{I}_{\rm fast}$ is not even guaranteed to be non-negative.  In
what follows, we employ $\hat{I}_{\rm medium}$ for all the
calculations, due to its ease of computation and good finite sample
performance.
Figure \ref{fig:density2d} compares the bivariate fits of the kernel
density estimates and the Gaussian models over four edges.  For
the Gaussian fits of each edge, we directly calculate the bivariate
sample covariance and sample mean and plug them into the bivariate
Gaussian density function.  From the perspective and contour plots, we
see that the bivariate kernel density estimates provide reasonable
fits for these bivariate components.

\begin{figure}[htp!]
\begin{center}
%\vskip-10pt
\begin{tabular}{cc}
\small Gaussian & \small non-Gaussian\\[-20pt]
\hskip-.1in
\includegraphics[width=0.5\textwidth]{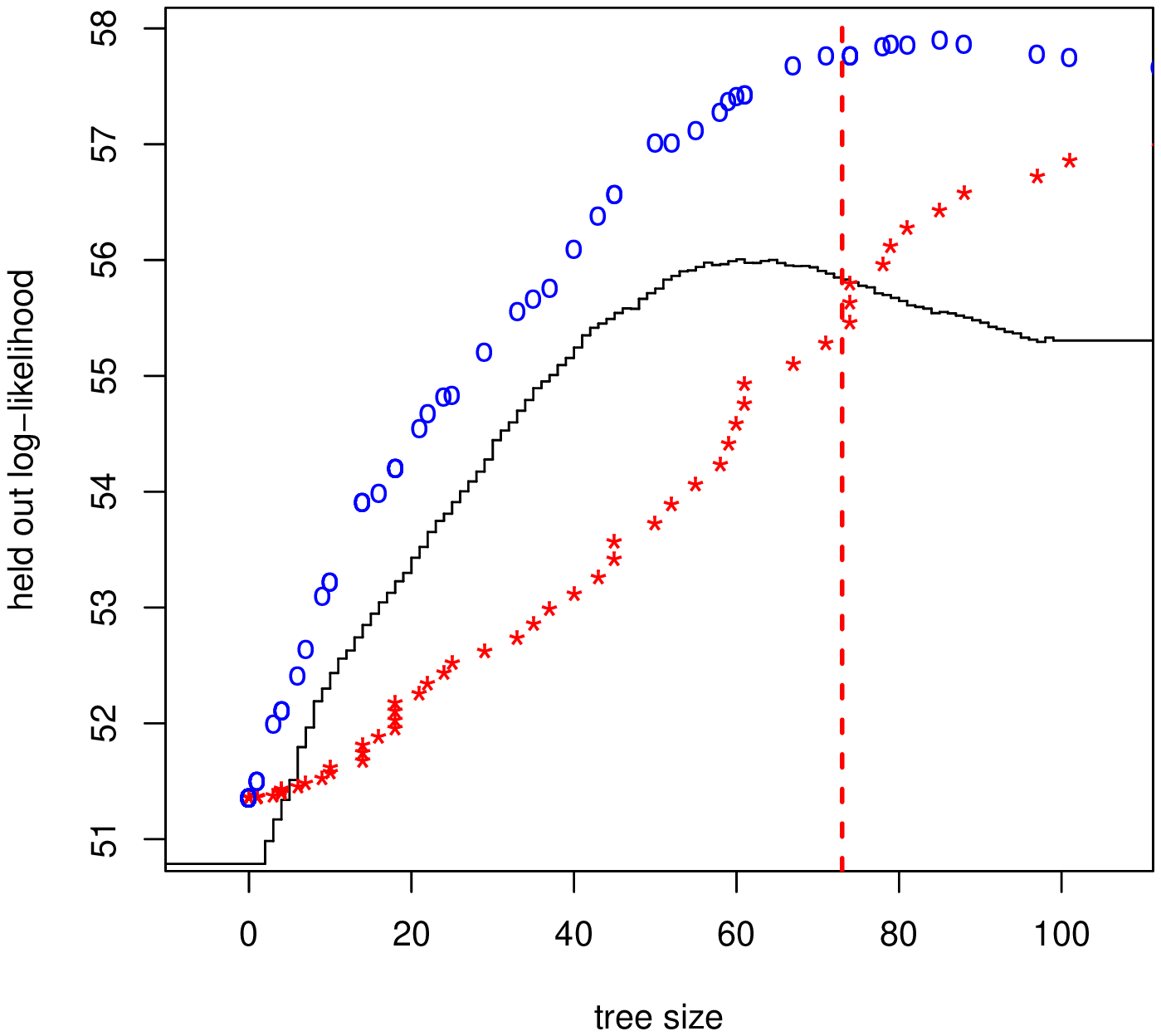} & 
\hskip-.1in
\includegraphics[width=0.5\textwidth]{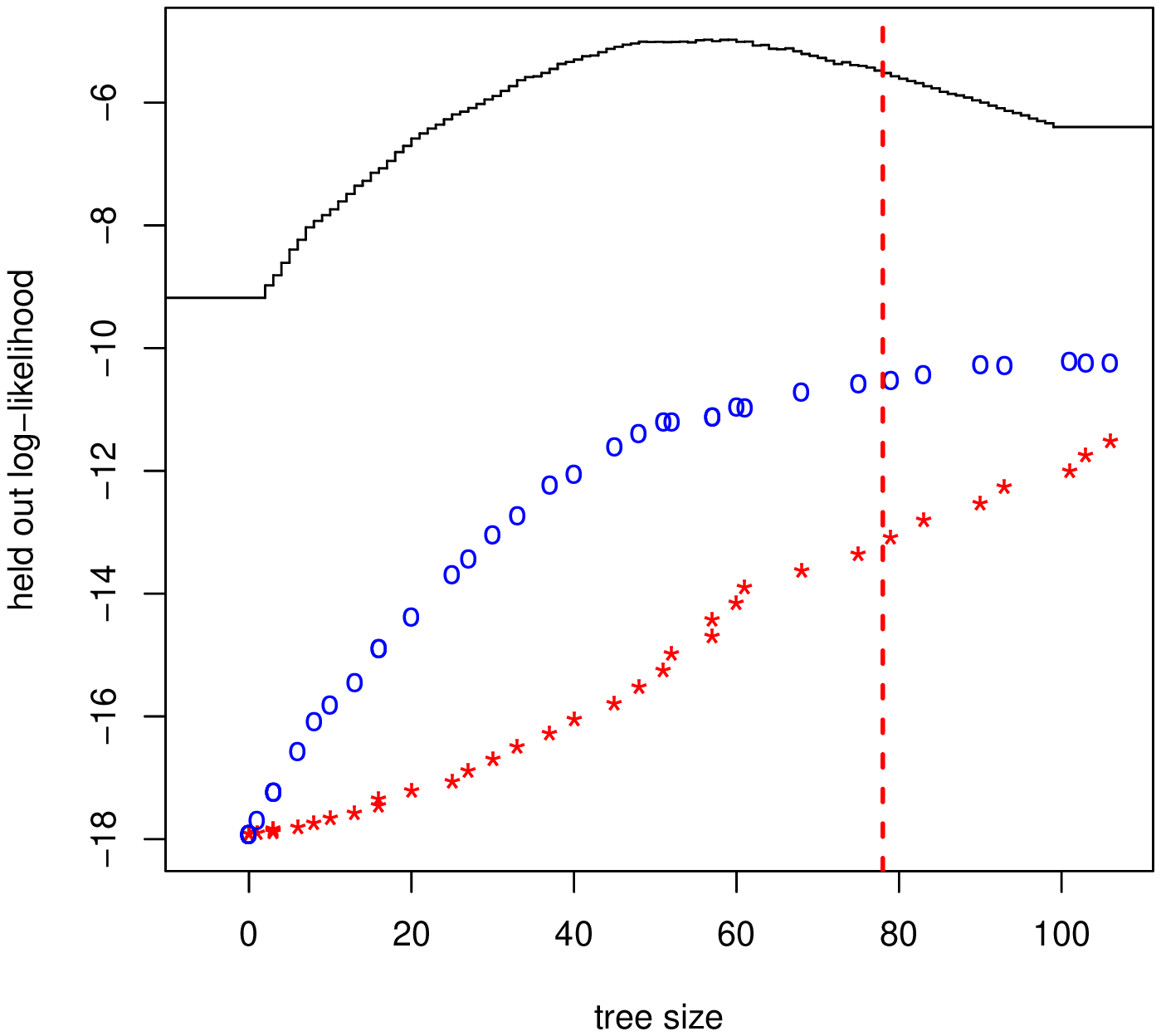}
\end{tabular}
\vskip-30pt
\begin{tabular}{ccc}
\hskip-.6in\includegraphics[width=0.45\textwidth]{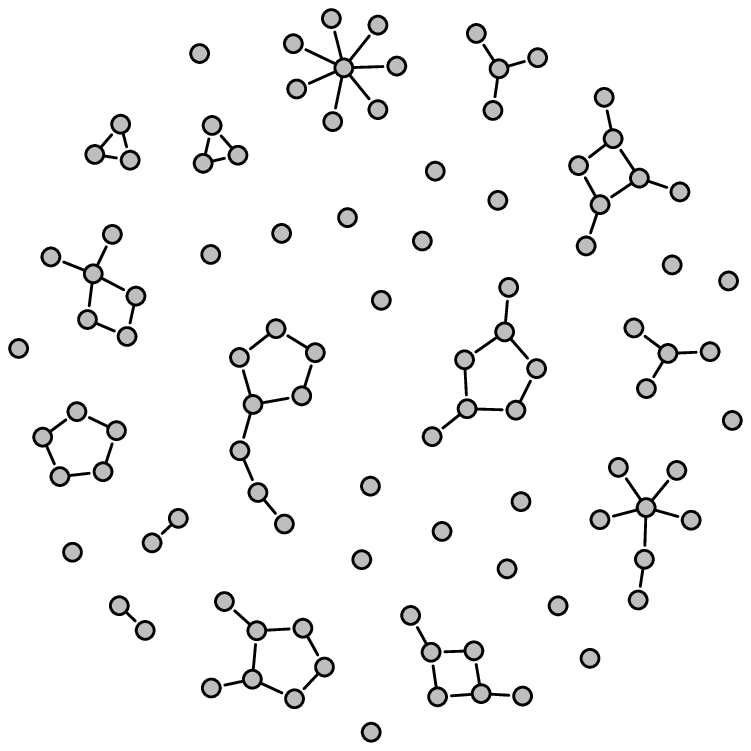}&
\hskip-.9in\includegraphics[width=0.45\textwidth]{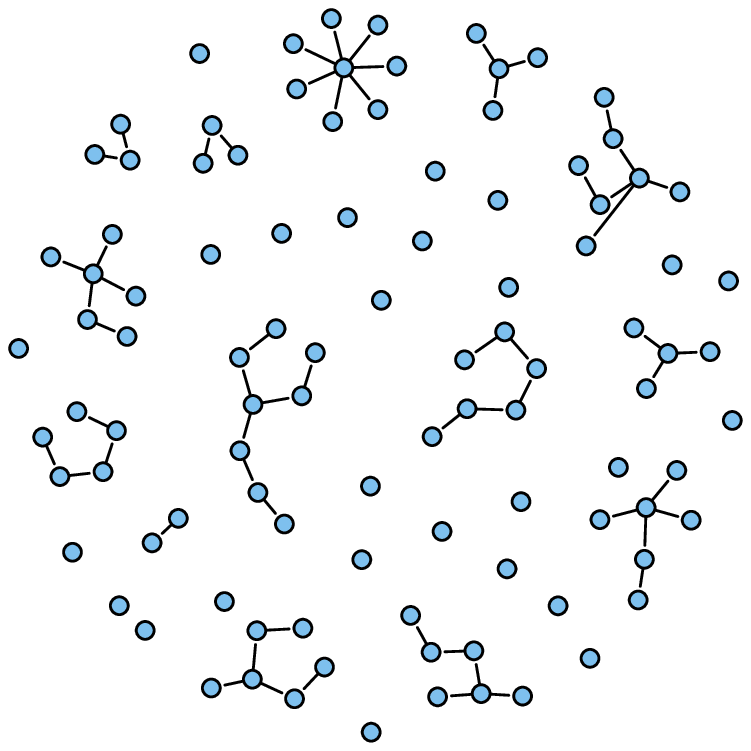}&
\hskip-.9in\includegraphics[width=0.45\textwidth]{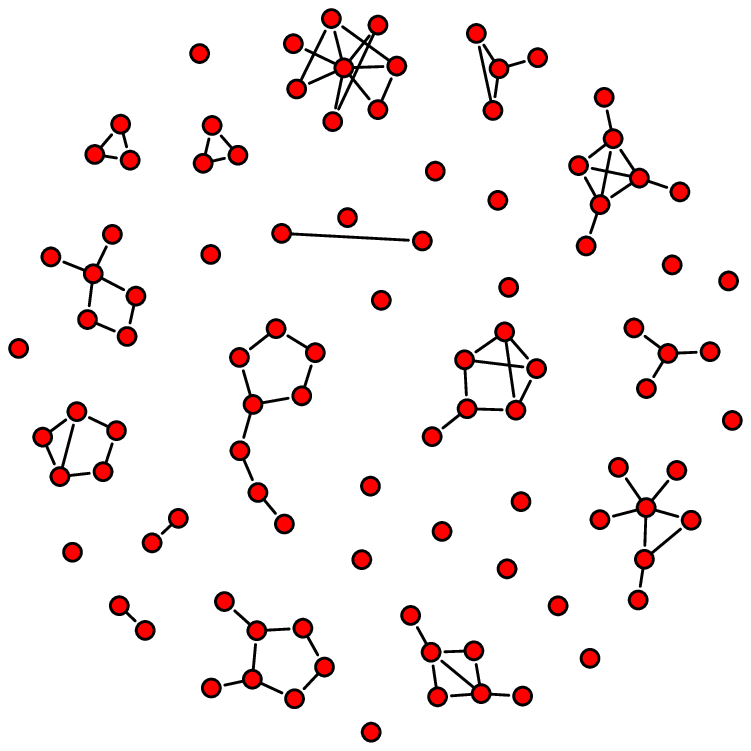}
\\[-70pt]
\hskip-.6in\includegraphics[width=0.45\textwidth]{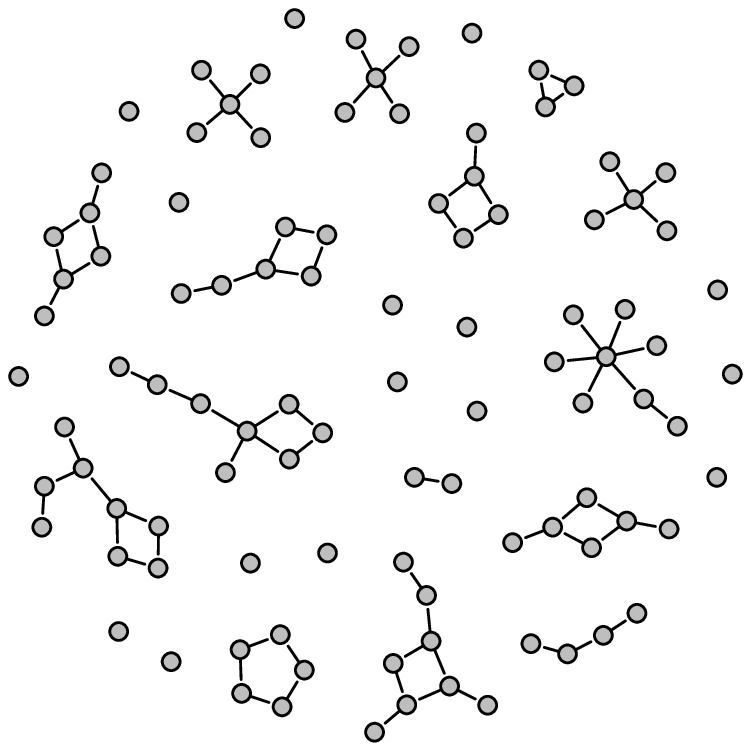}&
\hskip-.9in\includegraphics[width=0.45\textwidth]{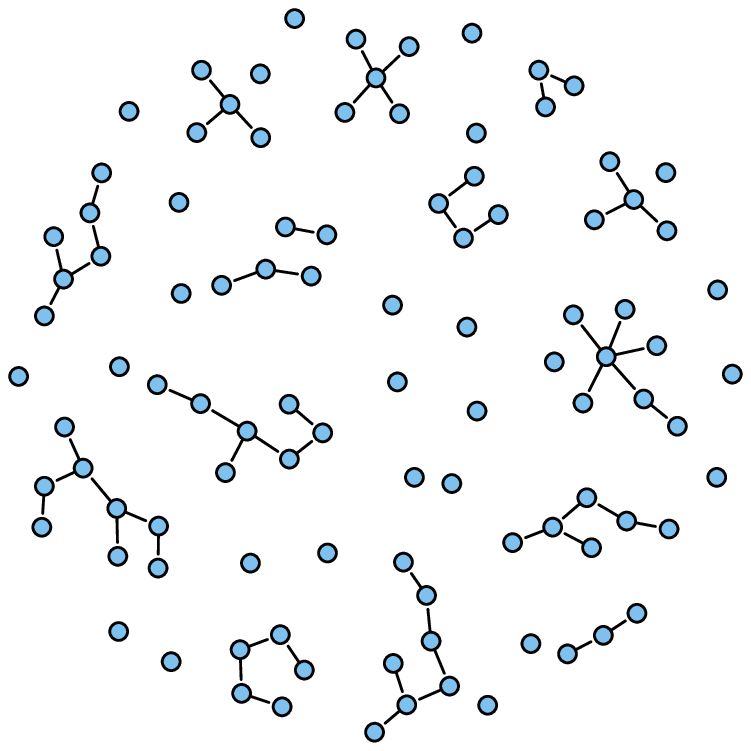}&
\hskip-.9in\includegraphics[width=0.45\textwidth]{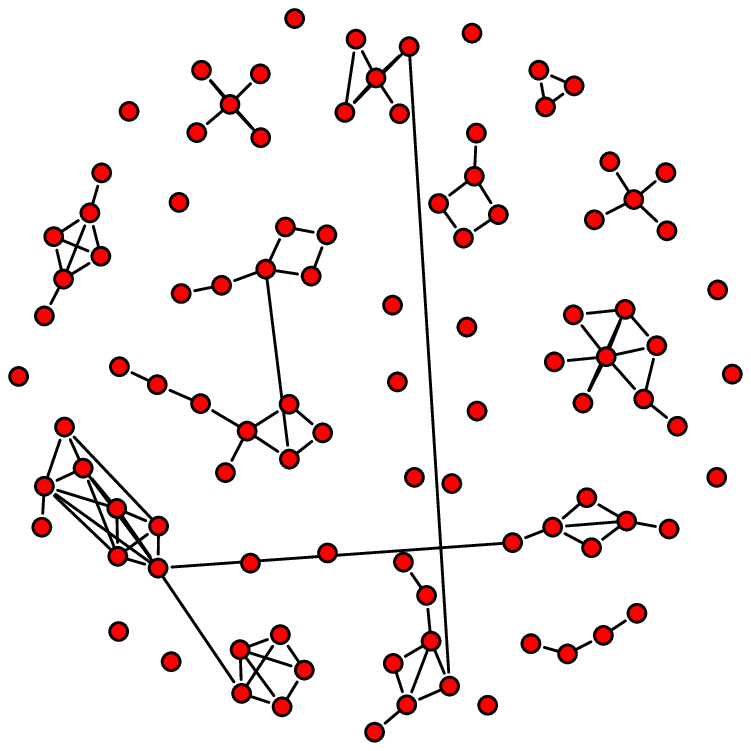}
\\[-40pt]
\hskip-.4in\small true & \hskip-.7in\small forest density estimator & \hskip-.7in\small glasso
\\[5pt]
\end{tabular}
\end{center}
\caption{ Synthetic data.  Top-left Gaussian, and top-right
  non-Gaussian: Held-out log-likelihood plots of the forest
  density estimator (black step function), glasso (red stars), and
  refit glasso (blue circles), the vertical dashed red line indicates
  the size of the true graph.  Bottom plots show the true and estimated
  graphs for the Gaussian (second row) and non-Gaussian data
  (third row).}\label{fig.loglikesim}
\end{figure}

A typical run showing the held-out log-likelihood and estimated graphs
is provided in Figure~\ref{fig.loglikesim}.  We see that for the
Gaussian data, the refit glasso has a higher held-out log-likelihood
than the forest  density estimator and the
glasso.  This is expected, since the Gaussian model is correct.
For very sparse models, however, the performance of the glasso is worse than
that of the forest  density estimator, due to the large
parameter bias resulting from the $\ell_1$ regularization. We also
observe an efficiency loss in the nonparametric forest  density
estimator, compared to the refit glasso.
The graphs are automatically selected using
the held-out log-likelihood, and we see that the nonparametric forest-based
kernel density estimator tends to select a sparser model, while the
parametric Gaussian models tend to overselect.  This observation is
new and is quite typical in our simulations.  Another observation is
that the held-out log-likelihood curve of the glasso becomes flat for
less sparse models but never goes down. This suggests that the
held-out log-likelihood is not a good model selection criterion for
the glasso.  For the non-Gaussian data, even though the refit glasso
results in a reasonable graph, the forest  density
estimator performs much better in terms of held-out log-likelihood
risk and graph estimation accuracy.

To compare with $\k$-restricted forests, we generated additional Gaussian and non-Gaussian synthetic data as
before except on a different graph structure. In Figure
\ref{fig.varyheldout}, we use 400 training examples while varying the size
of heldout data to compare the log-likelihoods of four different
methods; the log-likelihood is evaluated on a third large dataset. In
Figure \ref{fig.nongaussgraphs}, we consider only non-Gaussian data,
use 400 training data and 400 heldout data, and generate graphs with
best heldout log-likelihood across the four methods. We compute
bandwidth, heldout log-likelihood, and mutual information same as
before.

\begin{figure}[ht]
\centering
\subfloat[][Non-Gaussian Data]{\includegraphics[trim=0cm .5cm .5cm .5cm, clip, width=8cm]{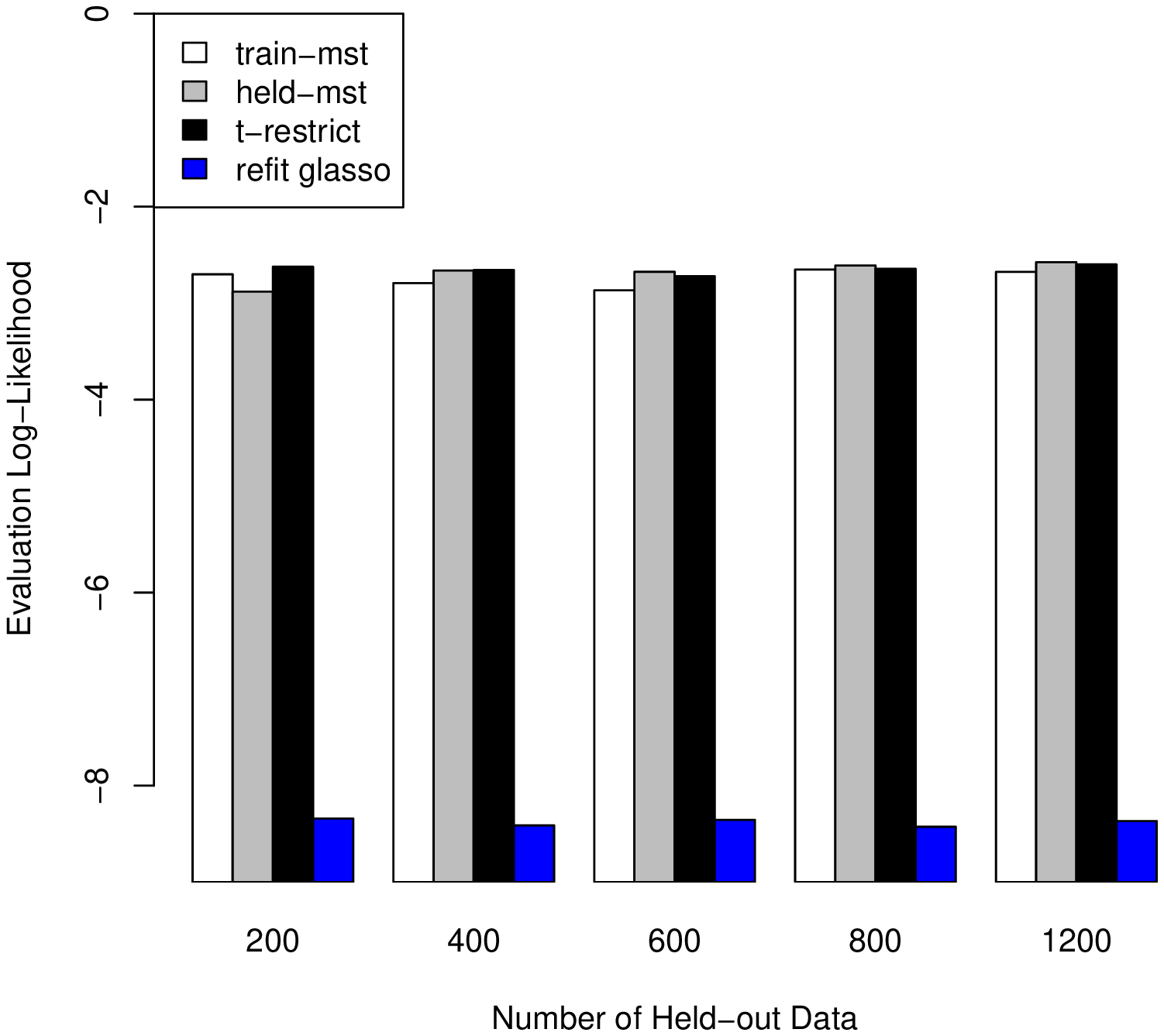}}
\hskip 0.5cm
\subfloat[][Gaussian Data]{\includegraphics[trim=0cm .5cm .5cm .5cm,clip,width=8cm]{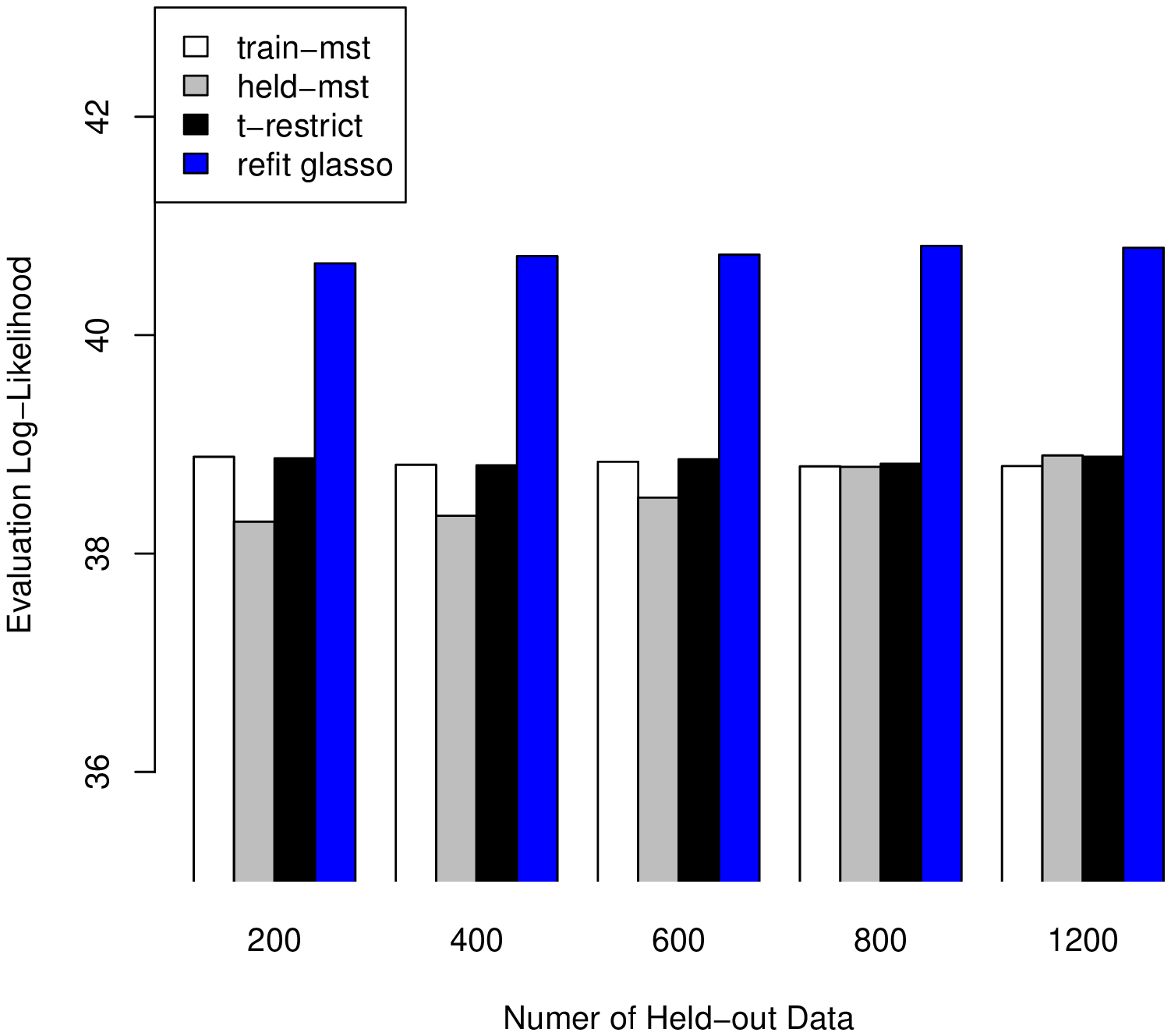}}
\caption{Log-likelihood comparison of various methods: (left white) MST on Training Data with Pruning (gray) MST on Heldout Data (black) t-Restricted Graph (blue) Refit Glasso}
\label{fig.varyheldout}
\end{figure}

\begin{figure}[t]
\centering
\subfloat[][]{\includegraphics[trim=1cm 1cm 1cm 1cm,clip,scale=0.45]{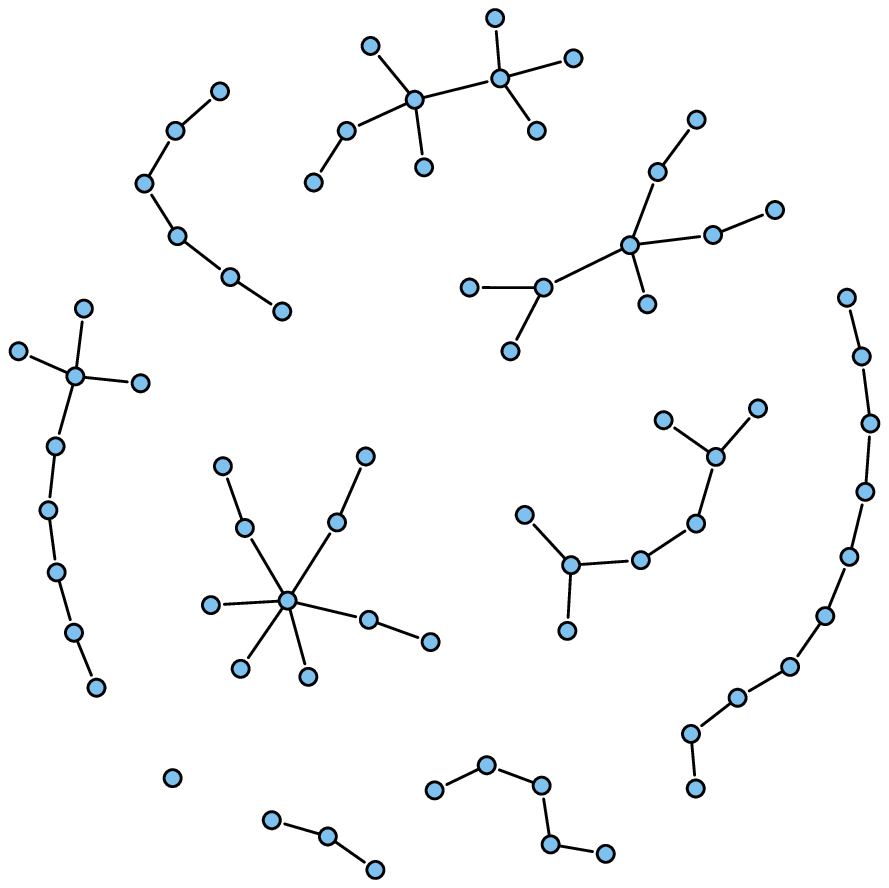}}
\hskip-20pt
\subfloat[][]{\includegraphics[trim=1cm 1cm 1cm 1cm,clip,scale=0.40]{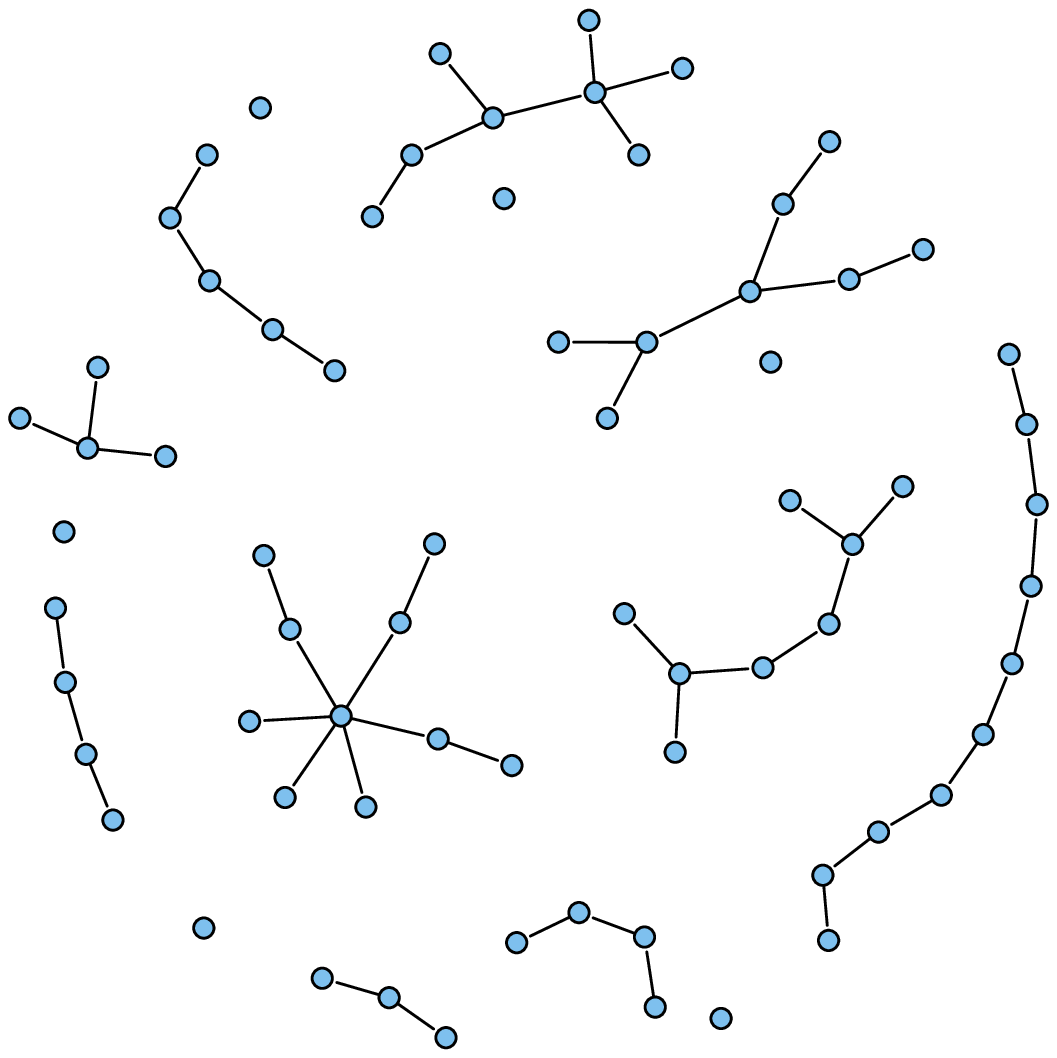}}
\hskip-20pt
\subfloat[][]{\includegraphics[trim=1cm 1cm 1cm 1cm,clip,scale=0.40]{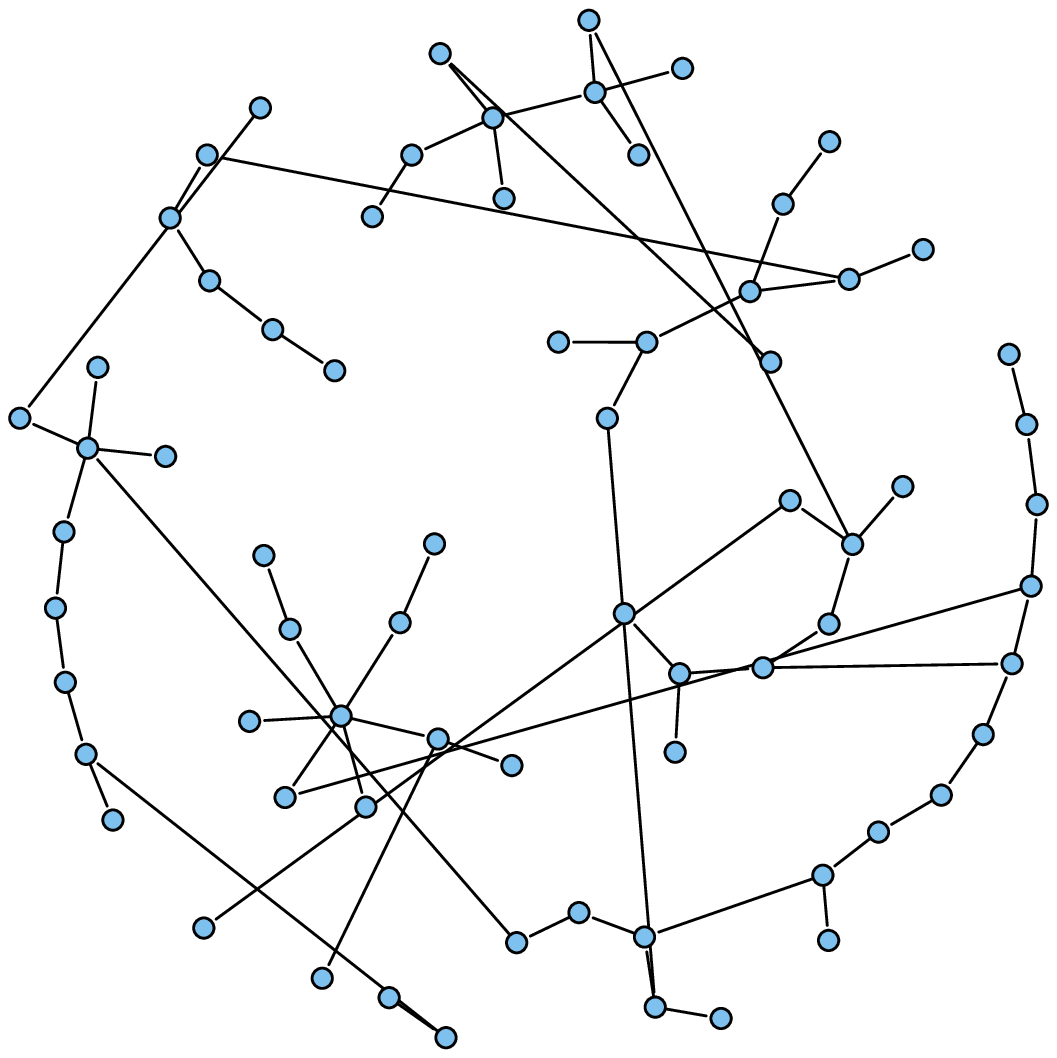}}
\hskip-10pt\subfloat[][]{\includegraphics[trim=1cm 1cm 1cm 1cm,clip,scale=0.40]{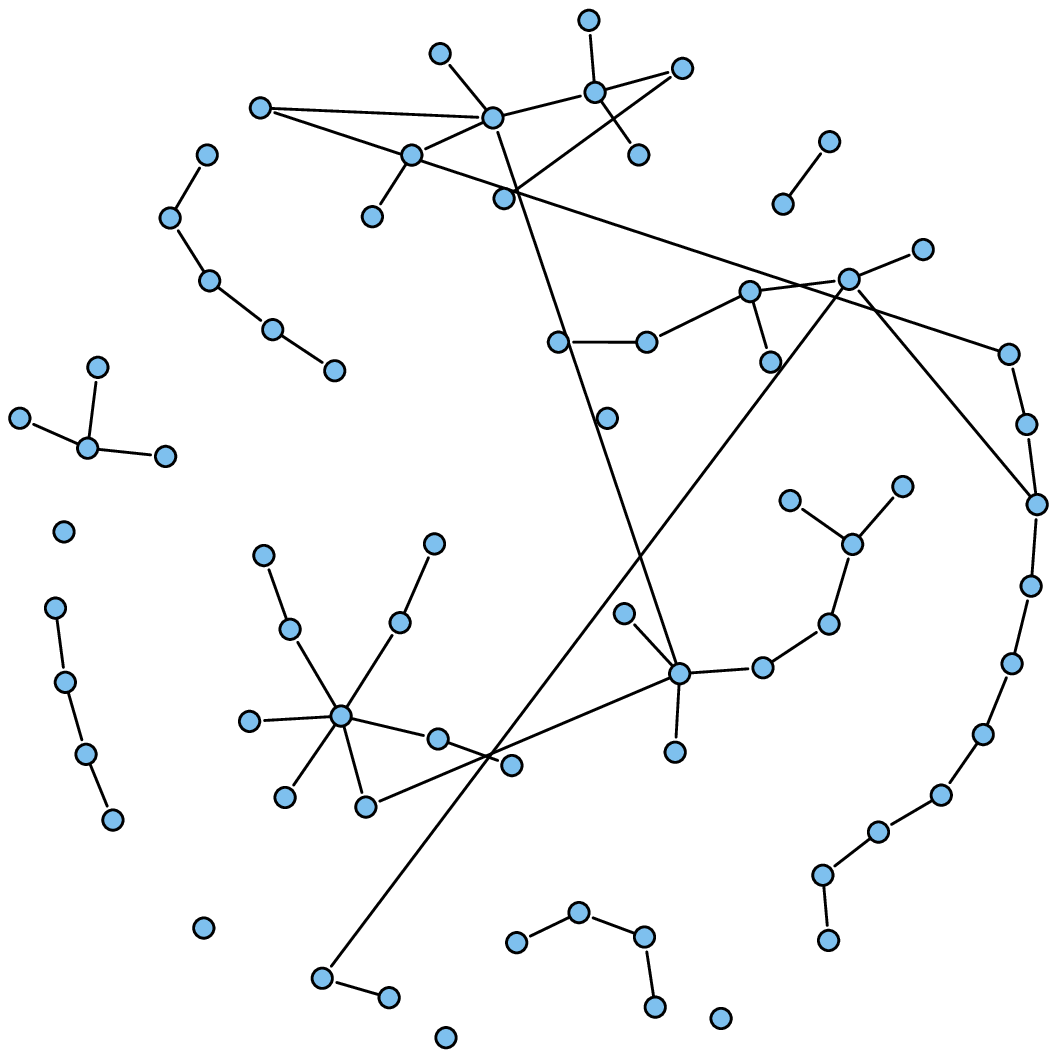}}
\hskip-10pt\subfloat[][]{\includegraphics[trim=1cm 1cm 1cm 1cm,clip,scale=0.40]{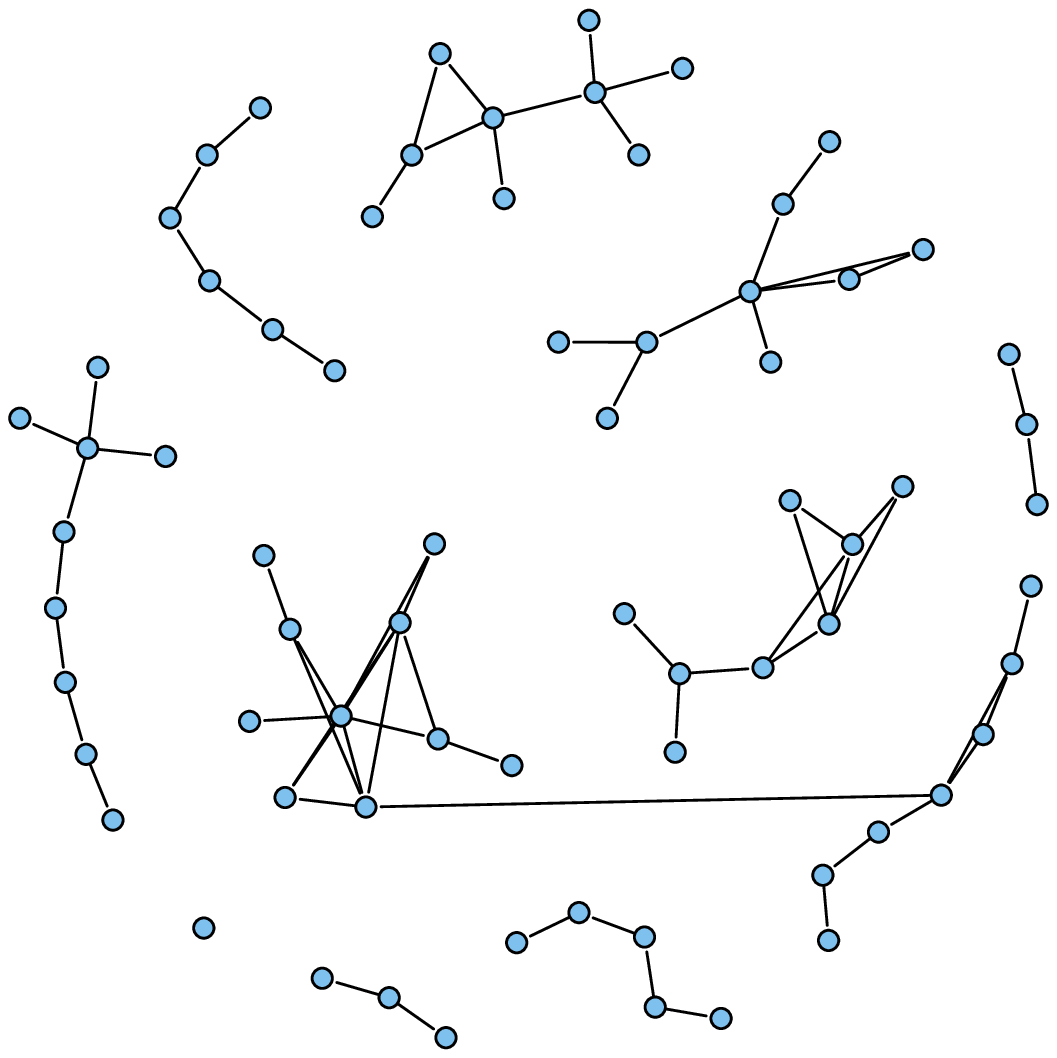}}
\caption{Graphs generated on non-Gaussian Data: (a) True Graph, (b) t-Restricted Forest (c) MST on Heldout Data (d) MST on Training Data with Pruning (e) Refit Glasso}
\label{fig.nongaussgraphs}
\end{figure}

We observe that although creating a maximum spanning tree (MST) on the held-out data is
asymptotically optimal; it can perform quite poorly. Unless there are
copious amount of heldout data, held-out MST overfits on the heldout
data and tend to give large graphs; in contrast, $t$-restricted forest
has the weakest theoretical guarantee but it gives the best
log-likelihood and produces sparser graphs. It is not surprising to
note that MST on heldout data improves as heldout data size
increases. Somewhat surprisingly though, Training-MST-with-pruning and
$t$-restricted forest appear to be insensitive to the heldout data size.

\subsection{Microarray data}

\subsubsection{Arabidopsis thaliana Data}

In this example, we consider a dataset based on Affymetrix GeneChip
microarrays for the plant \textit{Arabidopsis thaliana},
\citep{wille:04}.  The sample size is $n=118$.  The expression levels
for each chip are pre-processed by a log-transformation and
standardization.  A subset of 40 genes from the isoprenoid pathway are
chosen, and we study the associations among them using the glasso,
the refit glasso, and the tree-based kernel density estimator.

\begin{figure}[htp!]
\begin{center}
\vskip-.40in
\begin{tabular}{cc}
\hskip-.25in
\includegraphics[width=0.5\textwidth]{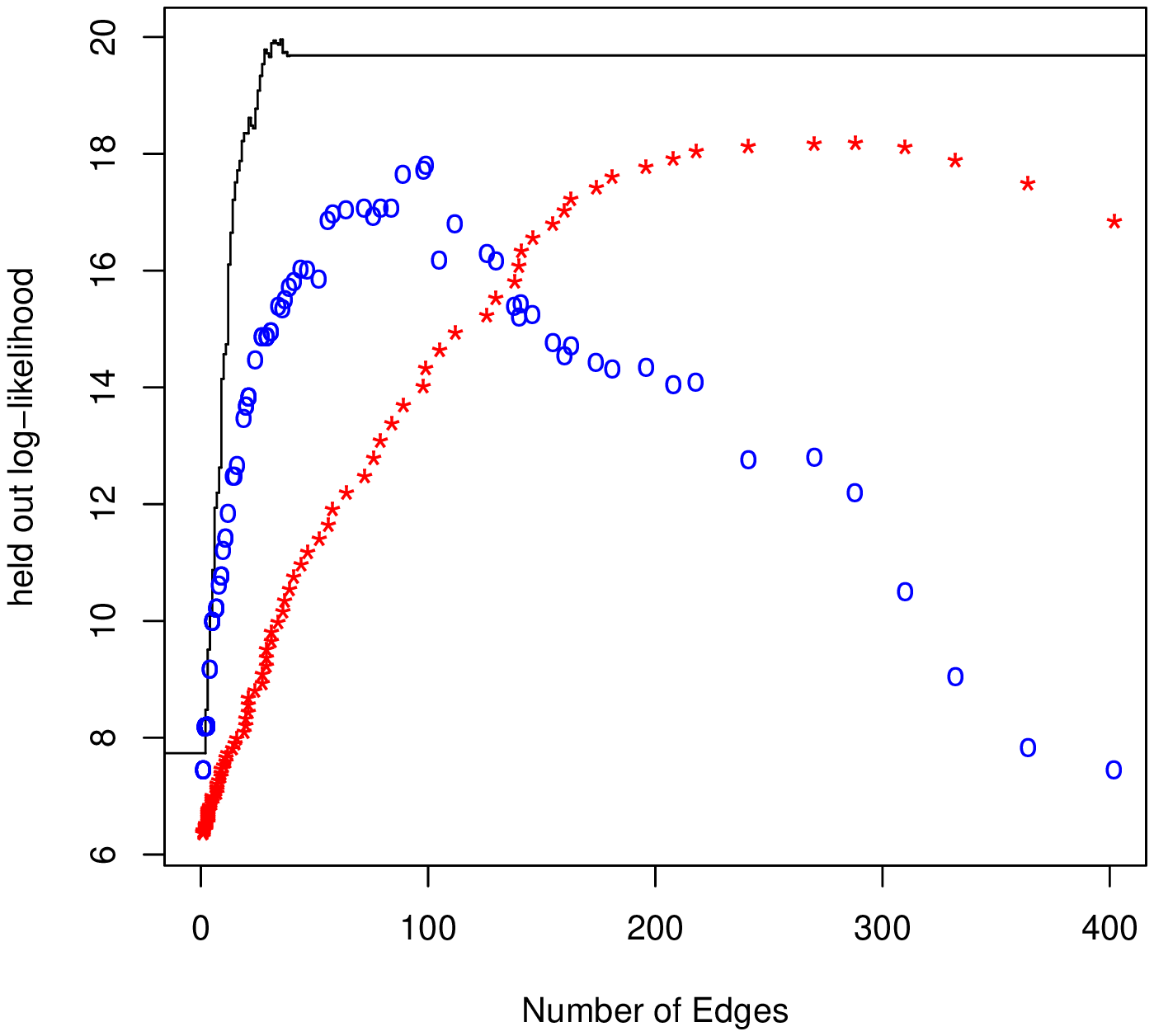} & 
\hskip-.50in
\includegraphics[width=0.5\textwidth]{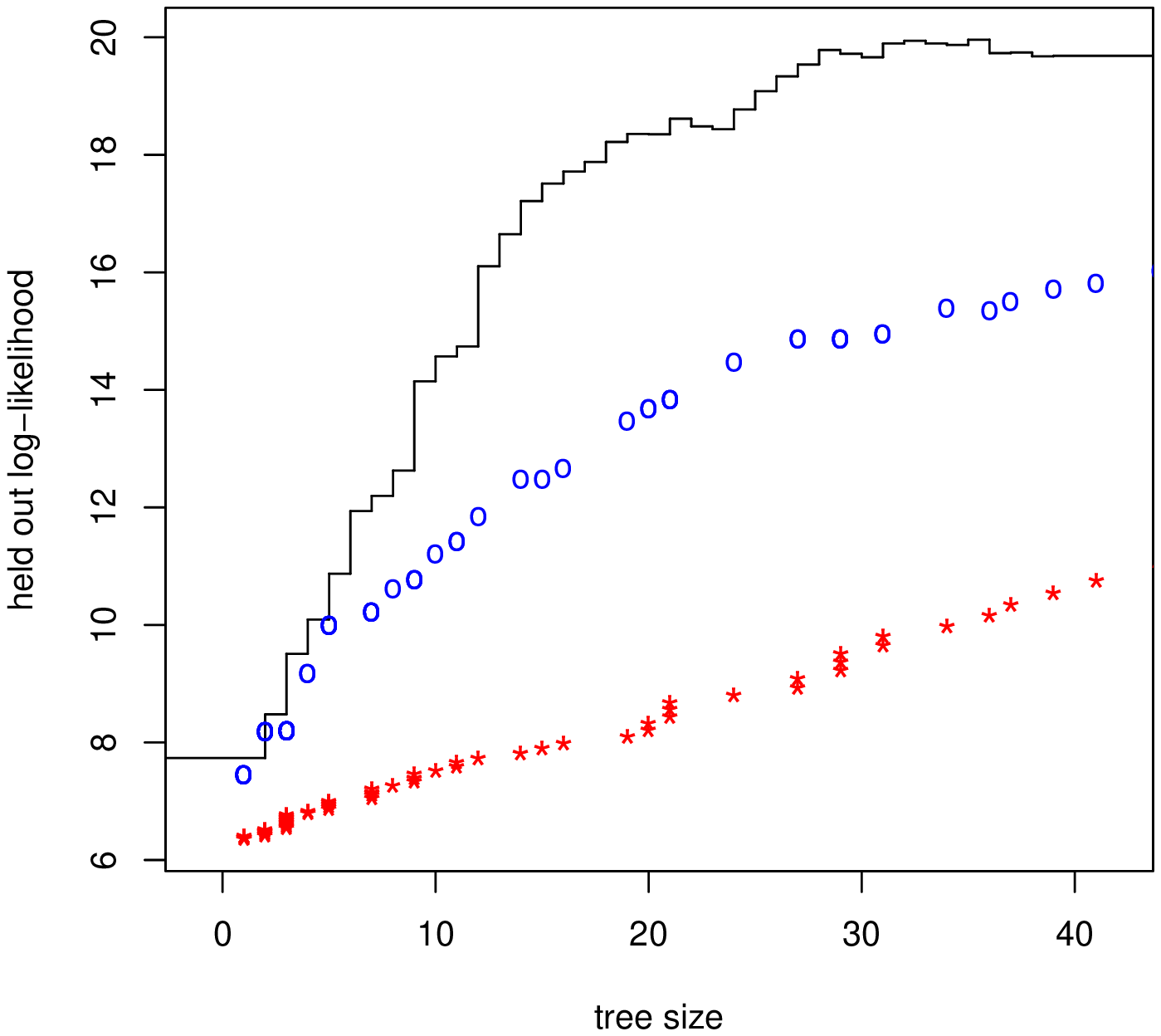} \\[-40pt]
\hskip-.50in
\includegraphics[width=0.6\textwidth]{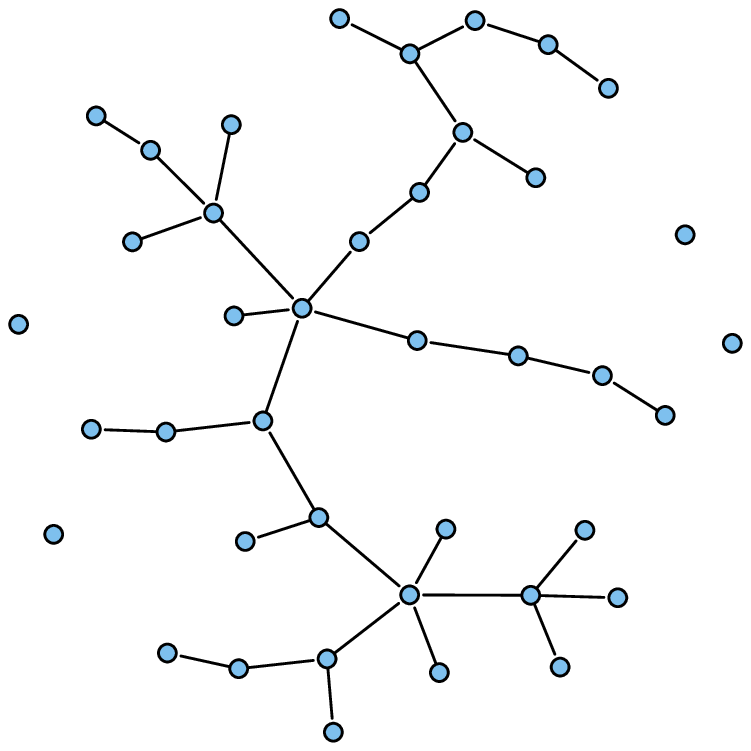} &
\hskip-.80in
\includegraphics[width=0.6\textwidth]{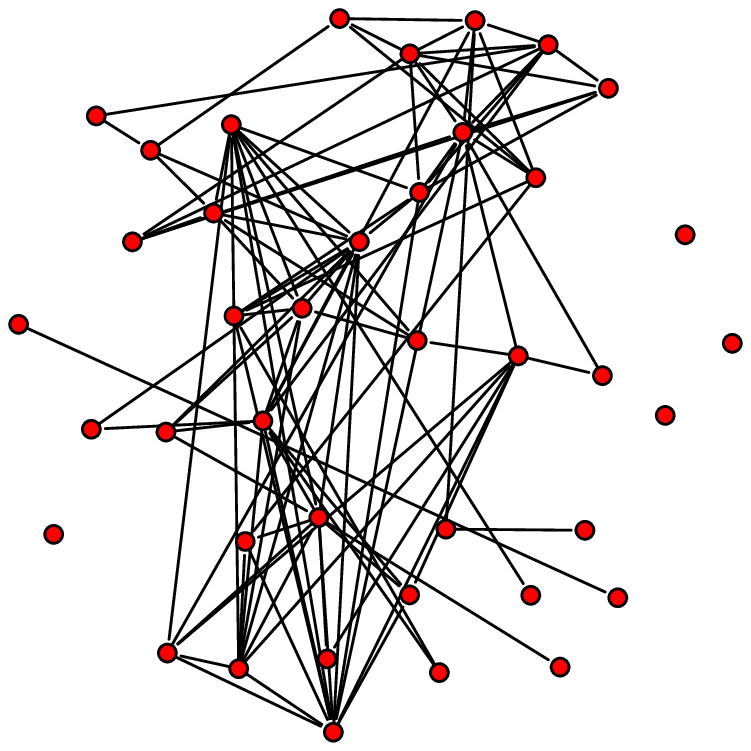} \\[-60pt]
\end{tabular}
\end{center}
\caption{Results on microarray data.  Top: held-out
  log-likelihood (left) and its zoom-in (right) of the tree-based
  kernel density estimator (black step function), glasso (red stars),
  and refit glasso (blue circles).  Bottom: estimated graphs
  using the tree-based estimator (left) and 
  glasso (right).}
\label{fig.loglikegene}
\end{figure}

From the held-out log-likelihood curves in Figure
\ref{fig.loglikegene}, we see that the tree-based kernel density
estimator has a better generalization performance than the glasso and
the refit glasso.  This is not surprising, given that the true distribution of the
data is not Gaussian. Another observation is that for the tree-based
kernel density estimator, the held-out log-likelihood curve achieves
a maximum when there are only 35 edges in the model. In contrast, the
held-out log-likelihood curves of the glasso and refit glasso achieve
maxima when there are around 280 edges and 100 edges respectively,
while their predictive estimates are still inferior to those of the tree-based
kernel density estimator.

Figure \ref{fig.loglikegene} also shows the estimated graphs for the
tree-based kernel density estimator and the glasso.  The graphs are
automatically selected based on held-out log-likelihood.  The two
graphs are clearly  different; it appears that the
nonparametric tree-based kernel density estimator has the potential to
provide different biological insights than the parametric Gaussian
graphical model.

\subsubsection{HapMap Data}

This dataset comes from \cite{nayak:09}. The dataset contains Affymetrics chip
measured expression levels of 4238 genes for 295 normal subjects in the 
\emph{Centre d'Etude du Polymorphisme Humain} (CEPH) and the International
HapMap collections. The 295 subjects come from four different groups:
148 unrelated grandparents in the CEPH-Utah pedigrees, 43 Han Chinese in Beijing, 
44 Japanese in Tokyo, and 60 Yoruba in Ibadan, Nigeria. Since we want to
find common network patterns across different groups of subjects, we pooled the 
data together into a $n=295$ by $d = 4238$ numerical matrix.  

We estimate the full 4238 node graph using both the forest
density estimator (described in Section \ref{subsec.tde.step1} and
\ref{subsec.tde.step2}) and the Meinshausen-B\"uhlmann neighborhood
search method as proposed in \citep{Meinshausen:2006} with regularization parameter
chosen to give it about same number as edges as the forest graph. 

To construct the kernel density estimates $\hat p(x_i,x_j)$
we use an array of Nvidia graphical processing units (GPU) to parallelize the computation over
the pairs of variables $X_i$ and $X_j$.  We discretize the domain of $(X_i,X_j)$
into a $128 \times 128$ grid,
and correspondingly employ $128 \times 128$
parallel cells in the GPU array, taking advantage of shared memory in CUDA.  Parallelizing in this way
increases the total performance by approximately a factor of 40,
allowing the experiment to complete in a day.

The forest density estimated graph reveals one strongly connected component
of more than 3000 genes and various isolated genes; this is consistent
with the analysis in \cite{nayak:09} and is realistic for the
regulatory system of humans. The Gaussian graph contains similar
component structure, but the set of edges differs significantly. We
also ran the $t$-restricted forest algorithm for $t=2000$ and it
successfully separates the giant component into three smaller
components.  For visualization purposes, in Figure \ref{fig.hapmap}, we show only a 934 gene subgraph of the
strongly connected component among the full 4238 node graphs we
estimated. More detailed analysis of the biological implications of this work will left as a future study.

\begin{figure}[htp!]
\begin{center}
\vskip-20pt
\begin{tabular}{cc}
\hskip-70pt\includegraphics[scale=0.8, trim=10mm 10mm 10mm 10mm]{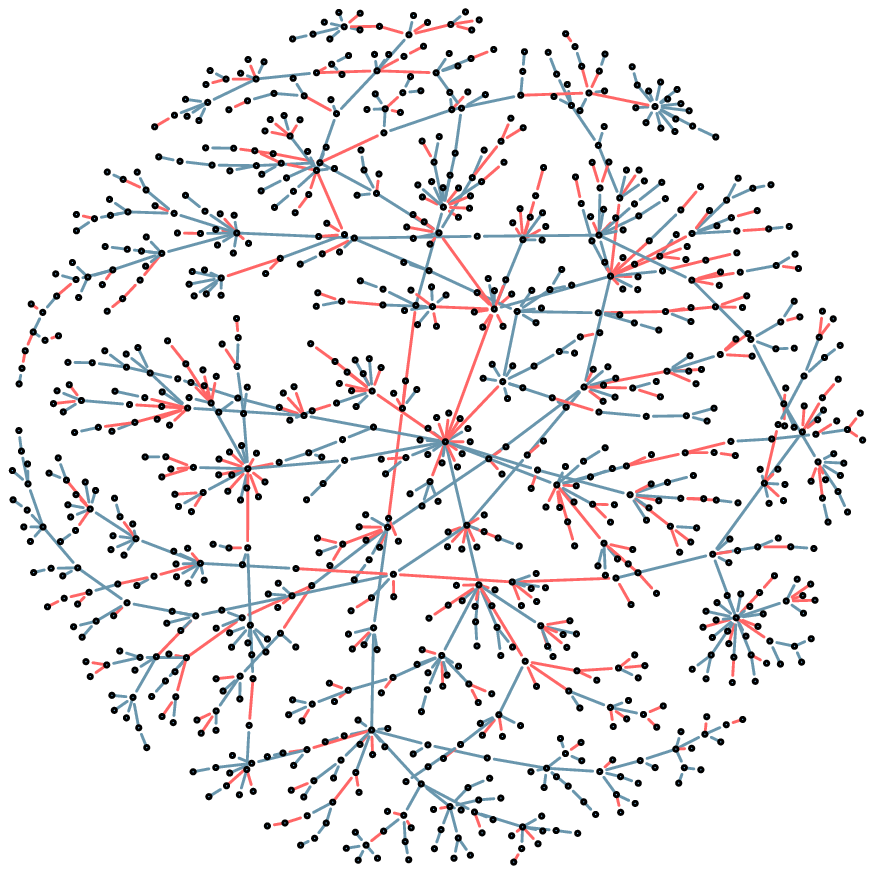} &
\hskip-90pt\includegraphics[scale=0.8, trim=10mm 10mm 10mm 10mm]{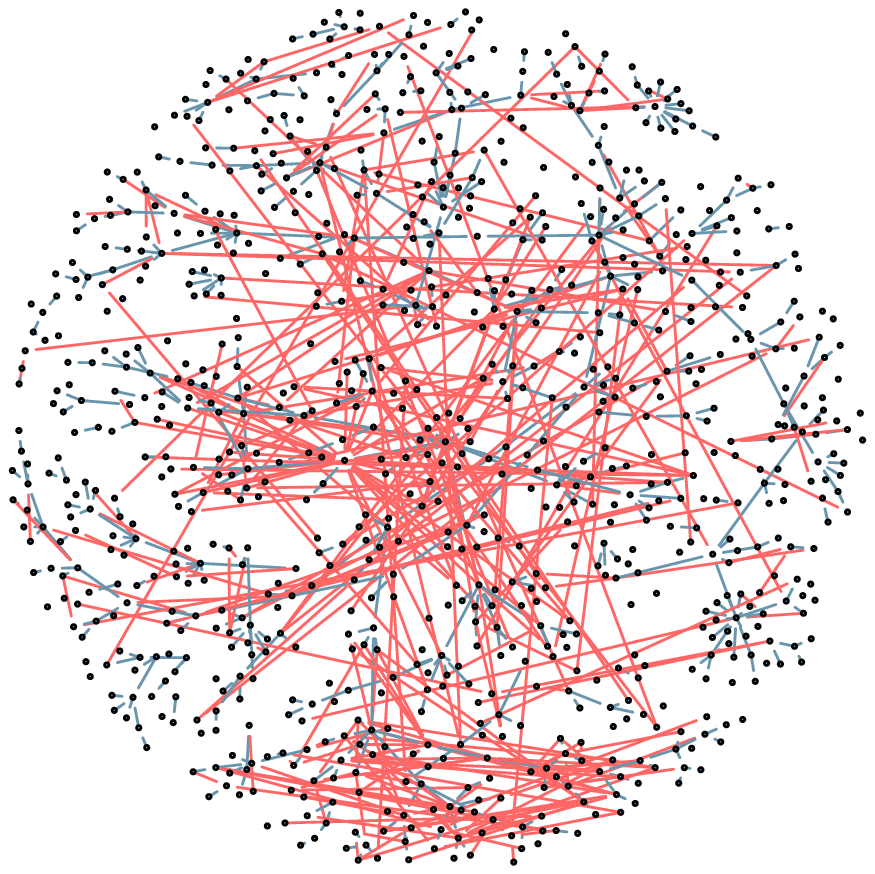}
\end{tabular}
\end{center}
\vskip-35pt
\caption{A 934 gene subgraph of the full estimated 4238 gene
  network. Left: estimated forest graph. Right: estimated Gaussian
  graph. Red edges in the forest graph are missing from the Gaussian
  graph and vice versa; the blue edges are shared by both graphs. Note
  that the layout of the genes is the same for both graphs.}\label{fig.hapmap}
\end{figure}

\section{Conclusion}

We have studied forest density estimation for high dimensional data.
Forest density estimation skirts the curse of dimensionality by
restricting to undirected graphs without cycles, while allowing fully
nonparametric marginal densities.  The method is computationally
simple, and the optimal size of the forest can be robustly selected by
a data-splitting scheme.  We have established oracle properties and
rates of convergence for function estimation in this setting. Our
experimental results compared the forest density estimator to the
sparse Gaussian graphical model in terms of both predictive risk and
the qualitative properties of the estimated graphs for human gene
expression array data.  Together, these results indicate that forest
density estimation can be a useful tool for relaxing the normality
assumption in graphical modeling.

\section*{Acknowledgements}

The research reported here was supported in part by NSF
grant CCF-0625879, AFOSR contract FA9550-09-1-0373, and a grant from
Google.  

\appendix

\section{Proofs}
\label{sec.proofs}

\subsection{Proof of Lemma \ref{lemma.key}}
\label{sec.keylemm}

We only need to consider the more complicated bivariate
case \eqref{eq.bivariatesup}; the result in \eqref{eq.univariatesup}
follows from the same line of proof.
First, given the assumptions, the following lemma can be obtained by
an application of Corollary 2.2 of \emcite{Gine:2002}.
For a detailed proof, see \emcite{rinaldo:2009}.

\begin{lemma}\label{lemma.rw} \citep{Gine:2002}
  Let $\hat{p}$ be a bivariate kernel density estimate using a
  kernel $K(\cdot)$ for which Assumption
  \ref{assump.kernel} holds and suppose that 
\begin{equation}
\sup_{t\in\X^{2}}\sup_{h_{2}>0}\int_{\X^{2}} K^{2}_{2}(u)\truep(t - uh_{2})du \leq D < \infty. \label{eq.Gienecondition}
\end{equation}
\begin{enumerate}
\item Let the bandwidth $h_{2}$ be fixed. Then there exit constants $L>0$ and $C>0$, which depend only on the VC characteristics of $\mathcal{F}_{2}$ in \eqref{eq.F2}, such that for any $c_{1} \geq C$ and $0 < \epsilon \leq c_{1}D/\|K_{2} \|_{\infty}$, there exists $n_{0}>0$ which depends on $\epsilon$, $D$, $\|K _{2}\|_{\infty}$ and the VC characteristics of $K_{2}$, such that for all $n \geq n_{0}$,
\begin{equation}
\mathbb{P}\left(\sup_{u\in\X^{2}}|\hat{p}(u) -\mathbb{E}\hat{p}(u)   | > 2\epsilon \right)\leq L\exp\left\{-\frac{1}{L}\frac{\log(1+c_{1}/(4L))}{c_{1}} \frac{nh^{2}_{2}\epsilon^{2}}{D}\right\}. \label{eq.expkern1}
\end{equation}
\item Let $h_{2}\rightarrow 0$ in such a way that ${n
    h^{2}_{2}}/{\log h_{2}} \rightarrow \infty$, and let $\epsilon
  \rightarrow 0$ so that
\begin{eqnarray}
\epsilon = \Omega\left(\sqrt{\frac{\log r_{n}}{nh^{2}_{2}} } \right), \label{eq.epsiloncond}
\end{eqnarray}
where $r_{n} = \Omega(h^{-1}_{2})$. Then \eqref{eq.expkern1} holds for sufficiently large $n$.
\end{enumerate}
\end{lemma}

From $(D2)$ in Assumption \ref{assump.density} and $(K1)$ in Assumption \ref{assump.kernel}, it's easy to see that \eqref{eq.Gienecondition} is satisfied.  Also, since 
\begin{equation}
h_{2} \asymp \left(\frac{\log n}{n}\right)^{\frac{1}{2+2\beta}},  
\end{equation}
it's clear that ${n h^{2}_{2}}/{\log h_{2}} \rightarrow \infty$.
Part 2 of Lemma \ref{lemma.rw} shows that there exist $c_{2}$ and $c_{3}$ such that 
\begin{equation}
\mathbb{P} \left( \sup_{(x_{i}, x_{j})\in\X_{i} \times \X_{j} } | \hat{p}(x_{i}, x_{j})  - \mathbb{E}\hat{p}(x_{i}, x_{j})| \geq \frac{\epsilon}{2} \right) \leq c_{2}\exp\left(-c_{3}n^{\frac{\beta}{1+\beta}} (\log n)^{\frac{1}{1+\beta}}\epsilon^{2}\right) \label{eq.firsthalf}
\end{equation}
for all $\epsilon$ satisfying \eqref{eq.epsiloncond}. 

This shows that for any $i,j \in \{1, \ldots, d \}$ with $i\neq j$,
the bivariate kernel density estimate $\hat{p}(x_{i}, x_{j})$ is
uniformly close to $\mathbb{E}\hat{p}(x_{i}, x_{j})$. Note that
$\mathbb{E}\hat{p}(x_{i}, x_{j})$ can be written as
\begin{equation}
\mathbb{E}\hat{p}(x_{i}, x_{j}) = \int \frac{1}{h^{2}_{2}} K\left(\frac{u_{i} - x_{i}}{h_{2}}\right)K\left(\frac{v_{j} - x_{j}}{h_{2}}\right)\truep(u_{i}, v_{j}) \,du_{i}dv_{j}.  
\end{equation}

The next lemma, from \emcite{Rig09}, provides a uniform deviation bound
on the bias term $\mathbb{E}\hat{p}(x_{i}, x_{j}) - \truep(x_{i},
x_{j}) $.

\begin{lemma}\label{lemma.bias}
\citep{Rig09} Under  $(D1)$ in Assumption \ref{assump.density} and $(K3)$ in Assumption \ref{assump.kernel}, we have
\begin{equation}
\sup_{(x_{i}, x_{j})\in\X_{i}\times\X_{j}}\left|
  \mathbb{E}\hat{p}(x_{i}, x_{j}) - \truep(x_{i}, x_{j})\right| 
\leq L_{1}  h^{\beta}_{2} \int_{\X^{2}} (u^{2} +
v^{2})^{\beta/2} K(u)K(v) \, du dv . 
\end{equation} 
where $L$ is defined in $(D1)$ of Assumption \ref{assump.density}.
\end{lemma}

Let $c_{4} = \ds L_{1}\int_{\X^{2}} (u^{2} +
  v^{2})^{\beta/2} K(u)K(v) \,du dv $. From the discussion of
Example 6.1 in \emcite{Rig09} and $(K1)$ in Assumption
\ref{assump.kernel}, we know that $c_{4} < \infty$ and only depends on
$K$ and $\beta$. Therefore
\begin{equation}
\mathbb{P} \left( \sup_{(x_{i}, x_{j})\in\X_{i} \times \X_{j} } |\truep(x_{i}, x_{j})  - \mathbb{E}\hat{p}(x_{i}, x_{j})| \geq \frac{\epsilon}{2} \right) = 0  \label{eq.secondhalf}
\end{equation}
for $\epsilon \geq 4c_{4}h^{\beta}_{2}$.

The desired result in Lemma \ref{lemma.key} is an exponential
probability inequality showing that  $\hat{p}(x_{i}, x_{j})$ is close
to $\truep(x_{i}, x_{j})$. To obtain this, we use a union bound:
\begin{eqnarray}
\lefteqn{\mathbb{P}\left(\max_{(i,j)\in \{1,\ldots, d \} \times \{1, \ldots, d \}}\sup_{(x_{i}, x_{j})\in\X_{i} \times \X_{j} } |\hat{p}(x_{i}, x_{j})  - \truep(x_{i}, x_{j})|  \geq \epsilon \right) } \nonumber~~~~~~~~~~~~~ \\
& \leq & d^{2}\mathbb{P} \left(\sup_{(x_{i}, x_{j})\in\X_{i} \times \X_{j} } |\hat{p}(x_{i}, x_{j})  - \mathbb{E}\hat{p}(x_{i}, x_{j})|  \geq \frac{\epsilon}{2} \right) \nonumber \\
& & + \; d^{2}\mathbb{P} \left( \sup_{(x_{i}, x_{j})\in\X_{i} \times \X_{j} } |\truep(x_{i}, x_{j})  - \mathbb{E}\hat{p}(x_{i}, x_{j})| \geq \frac{\epsilon}{2} \right).
\end{eqnarray}
Choosing
\begin{eqnarray}
\epsilon = \Omega\left(4c_{4}\sqrt{\frac{\log n + \log d}{n^{\beta/(1+\beta)}}}\right) , 
\end{eqnarray}
the result directly follows by combining \eqref{eq.firsthalf} and \eqref{eq.secondhalf}

\subsection{Proof of Theorem \ref{thm.persistency}}

First, from $(D2)$ in Assumption \ref{assump.density} and Lemma~\ref{lemma.key}, we have for any $i\neq j$,
\begin{eqnarray}\label{eq.supnormbound}
  \max_{(i,j)\in \{1,\ldots, d \} \times \{1, \ldots, d \}}\sup_{(x_{i}, x_{j})\in\X_{i} \times \X_{j} } \left( \frac{\hat{p}(x_{i}, x_{j}) }{\truep(x_{i}, x_{j})} -1 \right)= O_{P}\left(\sqrt{\frac{\log n + \log d}{n^{\beta/(\beta +1)} } }\right).
\end{eqnarray}

The next lemma bounds the deviation of $\hat{R}(\hat{p}_F)$ from
$R(\truep_{F})$ over different choices of $F \in \mathcal{F}_{d}$ with $|E({F})| \leq k$.  In the following,  we let
\begin{eqnarray}
\mathcal{F}^{(k)}_{d} = \left\{ F \in \mathcal{F}_{d}: |E({F})| \leq k \right\} 
\end{eqnarray}
denote the family of $d$-node forests with no more than $k$ edges.

\begin{lemma}\label{lemma.riskdeviate1} Under the assumptions  of Theorem \ref{thm.persistency}, we have
\begin{eqnarray}
\sup_{F \in \mathcal{F}^{(k)}_{d}}|\hat{R}(\hat{p}_F) - R(\truep_{F})| =  O_{P}\left(k\sqrt{\frac{\log n + \log d}{n^{\beta/(\beta +1)} } } + d\sqrt{\frac{\log n + \log d}{n^{2\beta/(1+2\beta)}}} \right). 
\end{eqnarray}
\end{lemma}
\begin{proof} For any $F\in\mathcal{F}^{(k)}_{d}$,  we have
\begin{eqnarray}
\lefteqn{|\hat{R}(\hat{p}_F) - R(\truep_{F})|   } \nonumber\\
&\leq  & \underbrace{\biggl| \sum_{(i,j)\in E({F})} \left( \int_{\X_{i}\times \X_{j}} \truep(x_{i}, x_{j}) \log \frac{\truep(x_{i}, x_{j}) }{\truep(x_{i}) \truep(x_{j})} dx_{i}dx_{j} -\int_{\X_{i}\times \X_{j}} \hat{p}(x_{i}, x_{j}) \log \frac{\hat{p}(x_{i}, x_{j}) }{\hat{p}(x_{i}) \hat{p}(x_{j})} dx_{i}dx_{j} \right) \biggr| }_{A_{1}}\nonumber\\
& & + \underbrace{ \biggl| \sum_{k\in V_{F}} \left( \int_{\X_{k}} \truep(x_{k}) \log\truep(x_{k})dx_{k} - \int_{\X_{k}} \hat{p}(x_{k}) \log\hat{p}(x_{k})dx_{k}  \right) \biggr|}_{A_{2}}. \nonumber\\
%& \leq &  \sum_{(i,j) \in E({F})} \biggl\{ \biggl|  \int_{\X_{i}\times \X_{j}}  \left(\truep(x_{i}, x_{j}) \log \truep( x_{i}, x_{j}) - \hat{p}(x_{i}, x_{j}) \log \hat{p}(x_{i}, x_{j}) \right) dx_{i} dx_{j}\biggr|  \nonumber\\
%& +& \biggl| \biggl( \int_{\X_{i}}  \left(\truep(x_{i})\log \truep( x_{i}) - \hat{p}(x_{i}) \log \hat{p}(x_{i}) \right) dx_{i}   +  \int_{\X_{j}}  \left(\truep(x_{j})\log \truep( x_{j}) - \hat{p}(x_{j}) \log \hat{p}(x_{j}) \right) dx_{j} \biggr) \biggr|\biggr\} \nonumber \\
%& & + \biggl| \sum_{k\in V_{F}} \left( \int_{\X_{k}} \truep(x_{k}) \log\truep(x_{k})dx_{k} - \int_{\X_{k}} \hat{p}(x_{k}) \log\hat{p}(x_{k})dx_{k}  \right) \biggr| %\nonumber\\
%& \equiv & A_{1} + A_{2}.
\end{eqnarray}
Defining
$ \truep_{ij} = \truep(x_{i}, x_{j})$ and $\hat{p}_{ij} = \hat{p}(x_{i}, x_{j})$,
we further have
\begin{eqnarray}
\lefteqn{A_{1} = } \\ & & O_{P}\!\!\left(\sum_{(i,j) \in E({F})} \left( \biggl| \sup_{(x_{i}, x_{j})\in\X_{i} \times \X_{j} }|\truep(x_{i}, x_{j}) -\hat{p}(x_{i}, x_{j}) |\int_{\X_{i}\times \X_{j}} \log \hat{p}( x_{i}, x_{j}) dx_{i} dx_{j}\biggr|+ D(\truep_{ij} \| \hat{p}_{ij} ) \right)\right) \nonumber
 \end{eqnarray}
where we use the  fact the univariate and bivariate densities are assumed to have the same smoothness parameter $\beta$,  therefore the univariate terms are of higher order,
and so can be safely ignored. 

It now suffices to show that 
\begin{eqnarray}
A_{1} = O_{P}\left(k\sqrt{\frac{\log n + \log d}{n^{\beta/(\beta +1)} } }\right)\label{eq.A1}
\end{eqnarray}
and
\begin{eqnarray}
A_{2} = O_{P}\left( d\sqrt{\frac{\log n + \log d}{n^{2\beta/(1+2\beta)}}}\right).\label{eq.A2}
\end{eqnarray}
In the sequel, we only prove \eqref{eq.A1}; \eqref{eq.A2} follows in the same way.

To show \eqref{eq.A1}, using the fact that $\max_{(i,j)\in \{1,\ldots,
  d \} \times \{1, \ldots, d \}}\truep(x_{i}, x_{j}) \leq c_{2}$, it's
sufficient to prove that
\begin{equation}
\max_{(i,j)\in \{1,\ldots, d \} \times \{1, \ldots, d \}} \sup_{(x_{i}, x_{j})\in\X_{i} \times \X_{j} }|\truep(x_{i}, x_{j}) -\hat{p}(x_{i}, x_{j}) |=  O_{P}\left(\sqrt{\frac{\log n + \log d}{n^{\beta/(\beta +1)} } }\right) \label{eq.todo1}
\end{equation}
and
\begin{equation}
\max_{(i,j)\in \{1,\ldots, d \} \times \{1, \ldots, d \}} D(\truep_{ij} \| \hat{p}_{ij})=  O_{P}\left(\sqrt{\frac{\log n + \log d}{n^{\beta/(\beta +1)} } }\right). \label{eq.todo2}
\end{equation}

Equation \eqref{eq.todo1} directly follows from \eqref{eq.bivariatesup} in Lemma \ref{lemma.key}, while \eqref{eq.todo2} follows from the fact that, for any  densities $p$ and $q$, where $q$ is strictly positive,
\begin{eqnarray}
D(p\| q) = \int \frac{p(x)}{q(x)}\log \frac{ p(x)}{q(x)} q(x)dx.
\end{eqnarray}
By a Taylor expansion, for $x\approx 1$, 
\begin{equation}
x\log x = (x - 1) + o\left(x-1 \right)
\end{equation}
and we then have 
\begin{equation}
D(\truep_{ij} \| \hat{p}_{ij} ) = 
O_{P}\left( \sup_{(x_{i}, x_{j})\in\X_{i} \times \X_{j}
  }|\truep(x_{i}, x_{j}) -\hat{p}(x_{i}, x_{j}) | \right).
\end{equation}
The desired result follows by combining \eqref{eq.todo1} and \eqref{eq.todo2}.
\end{proof}

The next auxiliary lemma is also needed to obtain the main result. It
shows that $\hat{R}(\hat{p}_F) $ does not deviate much from
$R(\hat{p}_{F})$ uniformly over different choices of $F \in
\mathcal{F}^{(k)}_{d}$.

\begin{lemma}\label{lemma.riskdeviate2}
Under the assumptions of Theorem \ref{thm.persistency}, we have
\begin{eqnarray}
\sup_{F\in\mathcal{F}^{(k)}_{d}}|R(\hat{p}_F) - \hat{R}(\hat{p}_{F})| =   O_{P}\left(k\sqrt{\frac{\log n + \log d}{n^{\beta/(\beta +1)} } } + d\sqrt{\frac{\log n + \log d}{n^{2\beta/(1+2\beta)}}} \right).
\end{eqnarray}
\end{lemma}

\begin{proof} Following the same line of argument as in Lemma \ref{lemma.riskdeviate1}, we have
for all $F \in \mathcal{F}^{(k)}_{d}$,
\begin{eqnarray}
\lefteqn{|R(\hat{p}_F) - \hat{R}(\hat{p}_{F})|   } \\
&\leq & \biggl| \sum_{(i,j)\in E({F})} \left( \int_{\X_{i}\times \X_{j}} \truep(x_{i}, x_{j}) \log \frac{\hat{p}(x_{i}, x_{j}) }{\hat{p}(x_{i}) \hat{p}(x_{j})} dx_{i}dx_{j} -\int_{\X_{i}\times \X_{j}} \hat{p}(x_{i}, x_{j}) \log \frac{\hat{p}(x_{i}, x_{j}) }{\hat{p}(x_{i}) \hat{p}(x_{j})} dx_{i}dx_{j} \right) \biggr| \nonumber\\
& & + \biggl| \sum_{k\in V_{F}} \left( \int_{\X_{k}} \truep(x_{k}) \log\hat{p}(x_{k})dx_{k} - \int_{\X_{k}} \hat{p}(x_{k}) \log\hat{p}(x_{k})dx_{k}  \right) \biggr| \nonumber\\
& = & \lefteqn{O_{P}\left(\sum_{(i,j) \in E({F})} \left( \biggl| \sup_{(x_{i}, x_{j})\in\X_{i} \times \X_{j} }|\truep(x_{i}, x_{j}) -\hat{p}(x_{i}, x_{j}) |\int \log \hat{p}( x_{i}, x_{j}) dx_{i} dx_{j}\biggr| \right)\right) } \nonumber \\
& & + \biggl| \sum_{k\in V_{F}} \left( \int_{\X_{k}} \truep(x_{k}) \log\hat{p}(x_{k})dx_{k} - \int_{\X_{k}} \hat{p}(x_{k}) \log\hat{p}(x_{k})dx_{k}  \right) \biggr|. \nonumber
\end{eqnarray}
From \eqref{eq.supnormbound}, we get that 
\begin{eqnarray}
\max_{(i,j)\in \{1,\ldots, d \} \times \{1, \ldots, d \}}\log| \hat{p}( x_{i}, x_{j})| < \max\{ |\log c_{2}|,  |\log c_{1}| \} + 1
\end{eqnarray}
for large enough $n$.
The result then directly follows from \eqref{eq.bivariatesup}  and \eqref{eq.univariatesup} in Lemma \ref{lemma.key}.
\end{proof}

The proof of the  main theorem follows by repeatedly applying the previous two lemmas.
As in Proposition \ref{prop.oracle}, with
\begin{equation}
\truep_{F^{(k)}_{d}} = \argmin_{q_{F} \in \mathcal{P}^{(k)}_{d}} R(q_{F}),
\end{equation}
we have
\begin{eqnarray}
\lefteqn{R(\hat{p}_{\hat{F}^{(k)}_{d}})  -  R(\truep_{F^{(k)}_{d}}) } \nonumber \\
& =  &  R(\hat{p}_{\hat{F}^{(k)}_{d}}) - \hat{R}(\hat{p}_{\hat{F}^{(k)}_{d}}) +  \hat{R}(\hat{p}_{\hat{F}^{(k)}_{d}}) -  R(\truep_{F^{(k)}_{d}})  \\
&=&  \hat{R}(\hat{p}_{\hat{F}^{(k)}_{d}}) -  R(\truep_{F^{(k)}_{d}})  +  O_{P}\left(k\sqrt{\frac{\log n + \log d}{n^{\beta/(\beta +1)} } } + d\sqrt{\frac{\log n + \log d}{n^{2\beta/(1+2\beta)}}} \right)\label{eq.toexplain1}\\
& \leq &  \hat{R}(\hat{p}_{F^{(k)}_{d}}) - R(\truep_{F^{(k)}_{d}}) +  O_{P}\left(k\sqrt{\frac{\log n + \log d}{n^{\beta/(\beta +1)} } } + d\sqrt{\frac{\log n + \log d}{n^{2\beta/(1+2\beta)}}} \right)\label{eq.toexplain2}\\
& = &  R(\truep_{F^{(k)}_{d}}) - R(\truep_{F^{(k)}_{d}})+  O_{P}\left(k\sqrt{\frac{\log n + \log d}{n^{\beta/(\beta +1)} } } + d\sqrt{\frac{\log n + \log d}{n^{2\beta/(1+2\beta)}}} \right)  \label{eq.toexplain3}\\
& = & O_{P}\left(k\sqrt{\frac{\log n + \log d}{n^{\beta/(\beta +1)} } } + d\sqrt{\frac{\log n + \log d}{n^{2\beta/(1+2\beta)}}} \right). 
\end{eqnarray}
where \eqref{eq.toexplain1} follows from Lemma
\ref{lemma.riskdeviate2}, \eqref{eq.toexplain2} follows from the fact
that $\hat{p}_{\hat{F}^{(k)}_{d}}$ is the minimizer of
$\hat{R}(\cdot)$, and \eqref{eq.toexplain3} follows from Lemma
\ref{lemma.riskdeviate1}.

\subsection{Proof of Theorem \ref{thm.randompersistency} }\label{subsec.randompersistency}

To simplify notation, we  denote 
\begin{eqnarray}
\phi_{n}(k) &=& k\sqrt{\frac{\log n + \log d}{n^{\beta/(\beta +1)}}}\\
\psi_{n}(d) &=& d\sqrt{\frac{\log n + \log d}{n^{2\beta/(1+2\beta)}}}.
\end{eqnarray}
Following the same proof as Lemma \ref{lemma.riskdeviate2}, we obtain
the following.
\begin{lemma}\label{lemma.heldoutriskdeviate2}
Under the assumptions of Theorem \ref{thm.persistency}, we have
\begin{eqnarray}
\sup_{F\in\mathcal{F}^{(k)}_{d}}|R(\hat{p}_{F}) - \hat{R}_{n_2}(\hat{p}_{F}) | =   O_{P}\biggl(\phi_{n}(k) + \psi_{n}(d)\biggr). \label{eq.useful1}
\end{eqnarray}
where $\hat{R}_{n_2}$ is the held out risk.
\end{lemma}

\comment{By a simple analysis, we can also show that 
\begin{eqnarray}
\sup_{F\in\mathcal{F}^{(k)}_{d}}|\hat{R}(\hat{p}_{F}) - \hat{R}_{n_2}(\hat{p}_{F}) | =   O_{P}\biggl(\phi_{n}(k) + \psi_{n}(d)\biggr). \label{eq.useful2}
\end{eqnarray}}

To prove Theorem~\ref{thm.randompersistency}, we now have
\begin{eqnarray}
R(\hat{p}_{\hat{F}^{(\hat{k})}_{d}})  - R(\hat{p}_{\hat{F}^{(k^{*})}_{d}}) 
& =  &  R(\hat{p}_{\hat{F}^{(\hat{k})}_{d}}) - \hat{R}_{n_2}(\hat{p}_{\hat{F}^{(\hat{k})}_{d}}) + \hat{R}_{n_2}(\hat{p}_{\hat{F}^{(\hat{k})}_{d}}) -  R(\hat{p}_{\hat{F}^{(k^{*})}_{d}})  \\
& =  &  O_{P}(\phi_{n}(\hat{k}) + \psi_{n}(d))+  \hat{R}_{n_2}(\hat{p}_{\hat{F}^{(\hat{k})}_{d}}) -  R(\hat{p}_{\hat{F}^{(k^{*})}_{d}})    \\
& \leq  &  O_{P}(\phi_{n}(\hat{k}) + \psi_{n}(d))+   \hat{R}_{n_2}(\hat{p}_{\hat{F}^{(k^{*})}_{d}}) -  R(\hat{p}_{\hat{F}^{(k^{*})}_{d}}) \label{eq.randomkey1}   \\
& = & O_{P}\left(\phi_{n}(\hat{k}) + \phi_{n}(k^{*}) + \psi_{n}(d) \right). \label{eq.randomkey2} 
\end{eqnarray}
where \eqref{eq.randomkey1} follows from the fact that $\hat{k}$ is
the minimizer of $\hat{R}_{n_2}(\cdot)$.

%deneb

\subsection{Proof of Theorem \ref{thm.heldoutforest}}

Using the shorthand
\begin{eqnarray}
& \phi_n(k) = \ds k \sqrt{\frac{\log n + \log d}{n^{\beta/(1+\beta)}}} \\
& \psi_n(d) = \ds d \sqrt{\frac{\log n + \log d}{n^{2\beta/(1+2\beta)}}} 
\end{eqnarray}

We have that 
\begin{eqnarray}
 R(\hat{p}_{\hat{F}_{n_2}}) - R(\hat{p}_{F^*})&=& R(\hat{p}_{\hat{F}_{n_2}}) - \hat{R}_{n_2}(\hat{p}_{\hat{F}_{n_2}}) + \hat{R}_{n_2}(\hat{p}_{\hat{F}_{n_2}}) - R(\hat{p}_{F^*})\\
 &=& O_{P}(\phi_n(\hat{k}) + \psi_n(d)) + \hat{R}_{n_2}(\hat{p}_{\hat{F}_{n_2}}) - R(\hat{p}_{F^*}) \\
 &\leq & O_{P}(\phi_n(\hat{k}) + \psi_n(d)) + \hat{R}_{n_2}(\hat{p}_{F^*}) - R(\hat{p}_{F^*}) \label{line.min} \\
 &=& O_{P}(\phi_n(\hat{k}) + \phi_n(k^*) + \psi_n(d))\\
\end{eqnarray}
where line \ref{line.min} follows because $\hat{F}_{n_2}$ is the minimizer of $\hat{R}_{n_2}(\cdot)$.

\subsection{Proof of Theorem \ref{thm.sparsistency}}

We begin by showing an exponential probability inequality 
on the difference between the empirical and population mutual informations.

\begin{lemma}\label{lemma.MIbound}  Under Assumptions \ref{assump.density}, \ref{assump.kernel}, there exist generic constants $c_{5}$ and $c_{6}$ satisfying
\begin{eqnarray}
\mathbb{P} \left(|I(X_{i}; X_{j})  - \hat{I}(X_{i}; X_{j})| >  \epsilon\right)  \leq  c_{5}\exp\left(-c_{6}n^{\frac{\beta}{1+\beta}} (\log n)^{\frac{1}{1+\beta}}\epsilon^{2}\right). 
\end{eqnarray}
for arbitrary $i, j \in \{1, \ldots, d \}$ with $i\neq j$,  and
$\epsilon \rightarrow 0$ so that
\begin{eqnarray}
\epsilon = \Omega\left(\sqrt{\frac{\log r_{n}}{nh^{2}_{2}} } \right), 
\end{eqnarray}
where $r_{n} = \Omega(h^{-1}_{2})$. 
\end{lemma}

\begin{proof} For any $\ds \epsilon = \Omega\left(\sqrt{\frac{\log r_{n}}{nh^{2}_{2}} } \right)$, we have
\begin{eqnarray} 
\lefteqn{\mathbb{P} \left(|I(X_{i}; X_{j})  - \hat{I}(X_{i}; X_{j})| > \epsilon\right)} \nonumber \\
& = & \mathbb{P}\left( |\int_{\X_{i}\times \X_{j}}  \truep(x_{i}, x_{j}) \log \frac{\truep(x_{i}, x_{j}) }{\truep(x_{i}) \truep(x_{j})} dx_{i}dx_{j} -\int_{\X_{i}\times \X_{j}}  \hat{p}(x_{i}, x_{j}) \log \frac{\hat{p}(x_{i}, x_{j}) }{\hat{p}(x_{i}) \hat{p}(x_{j})} dx_{i}dx_{j}  | > \epsilon \right) \nonumber\\
& \leq &  \mathbb{P}\left(  | \int_{\X_{i}\times \X_{j}}  \left(\truep(x_{i}, x_{j}) \log \truep( x_{i}, x_{j}) - \hat{p}(x_{i}, x_{j}) \log \hat{p}(x_{i}, x_{j}) \right) dx_{i} dx_{j} |> \frac{\epsilon}{2} \right)  \nonumber \\
& & + ~\mathbb{P}\left(  | \int_{\X_{i}\times \X_{j}}  \left(\truep(x_{i}, x_{j}) \log \truep( x_{i})\truep( x_{j}) - \hat{p}(x_{i}, x_{j}) \log \hat{p}(x_{i})\hat{p}(x_{j}) \right) dx_{i} dx_{j} |> \frac{\epsilon}{2} \right) \label{eq.decompose1}
\end{eqnarray}
Since the second term of \eqref{eq.decompose1} only involves
univariate kernel density estimates, this 
term is dominated by the first term, and we only need to analyze 
\begin{eqnarray}
\mathbb{P}\left(  | \int_{\X_{i}\times \X_{j}}  \left(\truep(x_{i}, x_{j}) \log \truep( x_{i}, x_{j}) - \hat{p}(x_{i}, x_{j}) \log \hat{p}(x_{i}, x_{j}) \right) dx_{i} dx_{j} |> \frac{\epsilon}{2} \right). 
\end{eqnarray}
The desired result then follows from the same analysis as in Lemma \ref{lemma.riskdeviate1}.
\end{proof}

Let 
\begin{equation}
L_{n} = \Omega\left(\sqrt{\frac{\log n + \log d}{n^{\beta/(1+\beta)}}} \right)
\end{equation}
be defined as in Assumption \ref{assump.sparsistency}. To prove the main theorem, we see the event $\hat{F}^{(k)}_{d} \neq
F^{(k)}_{d}$ implies that there must be at least exist two pairs of edges
$(i,j)$ and $(k, \ell)$, such that
\begin{eqnarray}
\sgn\Bigl(I(X_{i}, X_{j}) -  I(X_{k},    X_{\ell} ) \Bigr) \neq \sgn\left( \hat{I}(X_{i}, X_{j}) - \hat{I}(X_{k}, X_{\ell} )\right) . \label{eq::key11}
\end{eqnarray}
Therefore, we have
\begin{eqnarray}
\lefteqn{ \mathbb{P}\left( \hat{F}^{(k)}_{d} \neq F^{(k)}_{d} \right) }  \\
%&  \leq   & \mathbb{P} \left(\exists \Bigl( (i,j), (k,\ell) \right) \in \mathcal{J}, \left( I(X_{i}, X_{j}) -  I(X_{k}, X_{\ell} ) \Bigr)\cdot \left( \hat{I}(X_{i}, X_{j}) - \hat{I}(X_{k}, X_{\ell} )\right) \leq 0 \right) \\
&  \leq   & \mathbb{P} \left(\Bigl( I(X_{i}, X_{j}) -  I(X_{k},
    X_{\ell} ) \Bigr)\cdot \left( \hat{I}(X_{i}, X_{j}) -
    \hat{I}(X_{k}, X_{\ell} )\right) \leq 0,\;\text{for some $(i,j)$,
    $(k,\ell)$} \right). \nonumber
\end{eqnarray}
With $d$ nodes,  there can be  no more than $d^{4}/2$ pairs of edges;
thus, applying a union bound yields
\begin{eqnarray}
\lefteqn{\mathbb{P} \left(\Bigl( I(X_{i}, X_{j}) -  I(X_{k},
   X_{\ell} ) \Bigr)\cdot \left( \hat{I}(X_{i}, X_{j}) -
    \hat{I}(X_{k}, X_{\ell} )\right) \leq 0,\;\text{for some $(i,j)$,
    $(k,\ell)$} \right)}  \\
&  \leq   &\frac{ d^{4}}{2} \max_{\left( (i,j), (k,\ell) \right) \in \mathcal{J}}\mathbb{P} \left(\Bigl( I(X_{i}, X_{j}) -  I(X_{k}, X_{\ell} ) \Bigr)\cdot \left( \hat{I}(X_{i}, X_{j}) - \hat{I}(X_{k}, X_{\ell} )\right) \leq 0\right).~~~~~~~~~~~
\end{eqnarray}
Assumption \ref{assump.sparsistency} specifies that 
\begin{eqnarray}
\min_{\left( (i, j), (k, \ell)\right)\in \mathcal{J}}|I(X_{i}, X_{j}) - I(X_{k}, X_{\ell}) | > 2L_{n}.
\end{eqnarray}
Therefore, in order for \eqref{eq::key11} hold, there must exist an edge $(i,j) \in \mathcal{J}$ such that 
\begin{eqnarray}
|I(X_{i}, X_{j})  - \hat{I}(X_{i}, X_{j})| >  L_{n}.
\end{eqnarray}
Thus, we have
\begin{eqnarray}
\lefteqn{\max_{\left( (i,j), (k,\ell) \right) \in \mathcal{J}}\mathbb{P} \left(\Bigl( I(X_{i}, X_{j}) -  I(X_{k}, X_{\ell} ) \Bigr)\cdot \left( \hat{I}(X_{i}, X_{j}) - \hat{I}(X_{k}, X_{\ell} )\right) \leq 0\right) } \\
& &\leq  \; \max_{i, j \in \{1, \ldots, d \}, i\neq j}\mathbb{P} \left(|I(X_{i}, X_{j})  - \hat{I}(X_{i}, X_{j})| >  L_{n}\right)~~~~~~~~~~~~~~~~~~~~~~~~~\\
&  & \leq   \;
c_{5}\exp\left(-c_{6}n^{\frac{\beta}{1+\beta}} (\log n)^{\frac{1}{1+\beta}}L_{n}^{2}\right).\label{eq.toexplain4} 
\end{eqnarray}
where \eqref{eq.toexplain4} follows from Lemma \ref{lemma.MIbound}.

Chaining together the above arguments, we obtain
\begin{eqnarray}
\lefteqn{ \mathbb{P}\left( \hat{F}^{(k)}_{d} \neq F^{(k)}_{d} \right) }  \\
%&  \leq   & \mathbb{P} \left(\exists \left( (i,j), (k,\ell) \right) \in \mathcal{J}, \Bigl( I(X_{i}, X_{j}) -  I(X_{k}, X_{\ell} ) \Bigr)\cdot \left( \hat{I}(X_{i}, X_{j}) - \hat{I}(X_{k}, X_{\ell} )\right) \leq 0 \right) \\
&  \leq   & \mathbb{P} \left(\Bigl( I(X_{i}, X_{j}) -  I(X_{k},
    X_{\ell} ) \Bigr)\cdot \left( \hat{I}(X_{i}, X_{j}) -
    \hat{I}(X_{k}, X_{\ell} )\right) \leq 0,\;\text{for some $(i,j)$,
    $(k,\ell)$} \right) \nonumber \\
&  \leq   &\frac{ d^{4}}{2} \max_{\left( (i,j), (k,\ell) \right) \in \mathcal{J}}\mathbb{P} \left(\Bigl( I(X_{i}, X_{j}) -  I(X_{k}, X_{\ell} ) \Bigr)\cdot \left( \hat{I}(X_{i}, X_{j}) - \hat{I}(X_{k}, X_{\ell} )\right) \leq 0\right)  \\
&  \leq   & d^{4} \max_{i, j \in \{1, \ldots, d \}, i\neq j}\mathbb{P} \left(|I(X_{i}, X_{j})  - \hat{I}(X_{i}, X_{j})| >  L_{n}\right)\\
& \leq & d^{4} c_{5}\exp\left(-c_{6}n^{\frac{\beta}{1+\beta}} (\log n)^{\frac{1}{1+\beta}}L_{n}^{2}\right)\\
& =  & o\left( c_{5}\exp \left( 4\log d-{c_{6}} (\log n)^{\frac{1}{1+\beta}}\log d \right)\right)  \\ 
& =  &  o(1). 
\end{eqnarray}
The conclusion of the theorem now directly follows.

\subsection{Proof of NP-hardness of $\k$-Restricted Forest}
\label{sec.npproof}

We will reduce an instance of exact 3-cover (X3C) to an instance of
finding a maximum weight $\k$-restricted forest ($\k$-RF).

Recall that in X3C, we are given a finite set $X$ with $|X|$ = $3q$
and a family of 3-element subsets of $X$, $\cS = \{ S \subset X : |S|
= 3\}$.  The objective is to find a subfamily $\cS' \subset \cS$ such
that every element of $X$ occurs in exactly one member of $\cS'$,
or to determine that no such subfamily exists.

Suppose then we are given $X = \{x_1, \ldots, x_n\}$ and $\cS = \{S
\subset X : |S| = 3\}$, with $m = |\cS|$.
We construct the graph $G$ in an instance of $\k$-RF as follows,
and as illustrated in Figure~\ref{fig:gadget}.

\begin{figure}[t]
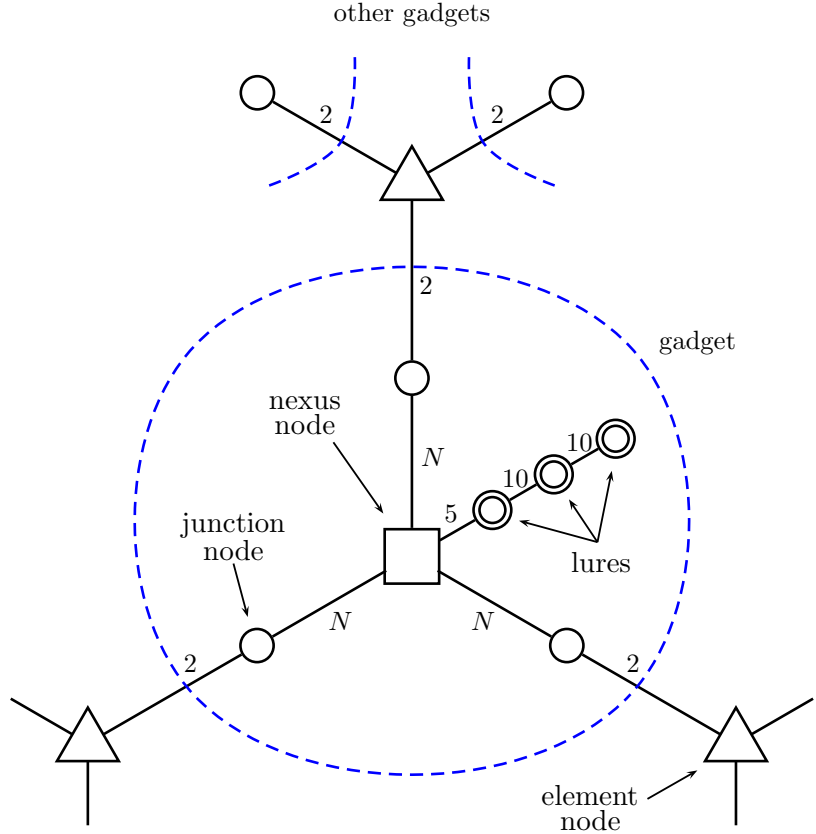

\begin{center}
\begin{tabular}{c}
\psset{nodesep=0pt}
\psset{linewidth=1pt}
\def\xscale{27pt}
\def\yscale{\xscale}
\psset{xunit=\xscale}
\psset{yunit=\yscale}
\pspicture(-6,-5)(6,8)
\def\tri{\pspolygon[](-.433,-.25)(.433,-.25)(0,.5)}
\def\tril{
   \pspolygon[](-.433,-.25)(.433,-.25)(0,.5)
   \psline(0,-.25)(0,-1.15)
   \psline(-.216,.125)(-1.08,.625)
}
\def\trir{
   \pspolygon[](-.433,-.25)(.433,-.25)(0,.5)
   \psline(0,-.25)(0,-1.15)
   \psline(.216,.125)(1.08,.625)
}
\def\trit{
   \pspolygon[](-.433,-.25)(.433,-.25)(0,.5)
   \cnodeput(2.165,1.25){TCR}{$\;\;$}
   \cnodeput(-2.165,1.25){TCL}{$\;\;$}
   \pnode(.216,.125){TTR}
   \pnode(-.216,.125){TTL}
   \ncline{-}{TTR}{TCR}
   \ncline{-}{TTL}{TCL}
}
%\psgrid
\cnodeput(0,2.5){A}{$\;\;$}
\pnode(0,5){AA}
\rput(0,5.25){\trit}
\cnodeput(-2.17,-1.25){B}{$\;\;$}
\pnode(-4.33,-2.50){BB}
\rput(-4.54,-2.62){\tril}
\cnodeput(2.17,-1.25){C}{$\;\;$}
\pnode(4.33,-2.50){CC}
\rput(4.54,-2.62){\trir}
\rput(0,0){\psframe(-.4,-.4)(.4,.4)}
\pnode(0,.4){TC}
\pnode(-.36,-.2){BL}
\pnode(.36,-.2){BR}
\ncline{-}{TC}{A}
\ncline{-}{BL}{B}
\ncline{-}{BR}{C}
\ncline{-}{A}{AA}
\ncline{-}{B}{BB}
\ncline{-}{C}{CC}
\cnodeput[doubleline=true](1.13,.65){L1}{$\;$}
\cnodeput[doubleline=true](1.99,1.15){L2}{$\;$}
\cnodeput[doubleline=true](2.857,1.65){L3}{$\;$}
\pnode(0.3897114,0.2250000){L0}
\ncline{-}{L0}{L1}
\ncline{-}{L1}{L2}
\ncline{-}{L2}{L3}
\psccurve[linestyle=dashed,linecolor=blue,showpoints=false](-3,-2)(3,-2)(3,3)(-3,3)
\pscurve[linestyle=dashed,linecolor=blue,showpoints=false](2.0,5.2)(1,5.8)(.8,7)
\pscurve[linestyle=dashed,linecolor=blue,showpoints=false](-2.0,5.2)(-1,5.8)(-.8,7)
\rput(0,7.6){other gadgets}
\rput(4,3){gadget}
\rput(-1.5,2){$\stackrel{\mbox{\small nexus}}{\mbox{\small node}}$}
\rput(-2.5,.3){$\stackrel{\mbox{\small junction}}{\mbox{\small node}}$}
\rput(2.5,-3.5){$\stackrel{\mbox{\small element}}{\mbox{\small node}}$}
\rput(2.65,-.08){$\mbox{\small lures}$}
\psline[linewidth=.7pt]{->}(3.3,-3.4)(4,-3)
\psline[linewidth=.7pt]{->}(2.6,.2)(1.5,.5)
\psline[linewidth=.7pt]{->}(2.6,.2)(2.2,.8)
\psline[linewidth=.7pt]{->}(2.6,.2)(2.8,1.2)
\psline[linewidth=.7pt]{->}(-1.1,1.6)(-.4,.6)
\psline[linewidth=.7pt]{->}(-2.5,-.1)(-2.3,-.85)
\rput(-1.2,6.2){\fs 2}
\rput(1.2,6.2){\fs 2}
\rput(.2,3.8){\fs 2}
\rput(3.1,-1.5){\fs 2}
\rput(-3.1,-1.5){\fs 2}
\rput(0.3,1.4){\fs $N$}
\rput(-1.0,-0.9){\fs $N$}
\rput(1.0,-0.9){\fs $N$}
\rput(0.55,0.6){\fs 5}
\rput(1.45,1.1){\fs 10}
\rput(2.35,1.6){\fs 10}
\endpspicture
\end{tabular}
\end{center}
\vskip-30pt
\caption{Gadget constructed in the reduction from X3C}
\label{fig:gadget}
\end{figure}

For each $x \in X$, add an \textit{element node} to $G$.
For each $S \in \cS$, construct a \textit{gadget}, which is a
subgraph comprised of a \textit{nexus node}, three \textit{junction nodes},
and three \textit{lure nodes}; see Figure~\ref{fig:gadget}.
We assign weights to the edges in a gadget in the following manner:
\begin{eqnarray*}
w(\text{element},\text{junction}) &=& 2\\
w(\text{nexus},\text{lure}_1) &=& 5\\
w(\text{lure}_1,\text{lure}_2) &=& 10\\
w(\text{lure}_2,\text{lure}_3) &=& 10\\
w(\text{nexus},\text{junction}) &=& N > 31m.
\end{eqnarray*}
Note that the weight $N$ is chosen to be strictly greater than the weight
all of the non-nexus-junction edges in the graph combined.
To complete the instance of $\k$-RF, let $\k = 7$.
%% reduction proof

\begin{lemma}
\label{lem.njedge}
Suppose $G$ is a graph constructed in the transformation from X3C
described above. Then $F^*_t$ must contain all the nexus-junction edges.
\end{lemma}

\begin{proof}
The set of all nexus-junction edges together form a well-defined
$\k$-restricted forest, since each subtree has a nexus node and 3
junction nodes. Call this forest $F$.
If some forest $F'$ is missing a nexus-junction edge, then $F'$ must
have weight strictly less than $F$, since $N$ is larger than
the sum of all of the non-nexus-junction edges.
\end{proof}

\begin{lemma}
\label{lem.nexus}
Each subtree in $F^*_\k$ can contain at most one nexus node.
\end{lemma}

\begin{proof}
Suppose a subtree $T$ in $F^*_\k$ contains two nexus nodes. Then it
must contain 6 junction nodes by Lemma
\ref{lem.njedge}. Thus, $T$ contains at least 8 nodes, and therefore violates
the $\k$-restriction constraint.
\end{proof}

\begin{lemma}
For each nexus node contained in $F^*_\k$, 
the corresponding three junction nodes are either connected to all
or none of the three neighboring element nodes.
\end{lemma}

\begin{proof}
  By the previous two Lemmas \ref{lem.njedge} and \ref{lem.nexus}, each
  subtree is associated with at most one gadget, and hence at most one
  $S \in \cS$, and moreover each gadget has as least one associated subtree.

Without loss of generality, we consider a region of the graph local to some arbitrary subtree. 
By the size constraint, a subtree cannot contain all the adjacent
element nodes and all the lure nodes. 

We now perform a case analysis:
\begin{packed_enum}

\item If a subtree contains no element nodes and all the lure nodes,
  then it has weight $3N+25$.  Call this an $\off$ configuration.

\item If a subtree contains two element nodes, and a second subtree of
  three nodes contains all the lure nodes, then the total weight of
  both subtrees is $3N+24$.  This is suboptimal because we can
  convert to an $\off$ configuration and gain additional weight without affecting
  any other subtrees. Hence, such a configuration cannot exist in
  $F^*_\k$.

\item If a subtree contains two element nodes and $\text{lure}_1$, and
  a second subtree contains just $\text{lure}_2$ and $\text{lure}_3$,
  then the total weight of the two subtrees is $3N+19$. This is again
  suboptimal.

\item If a subtree contains an element node and both $\text{lure}_1$
  and $\text{lure}_2$, then there cannot be a second subtree in region
  local to the gadget. The weight of this one subtree is $(3N+2+5+10)
  = 3N+17$, which is suboptimal.

\item If a subtree contains all three element nodes and no lure nodes,
  and a second subtree contains all the lure nodes, then the total
  weight is $(3N+6)+20 = 3N+26$. Call this an $\on$ configuration.

\end{packed_enum}

Thus, we see that each gadget in $F^*_\k$ must be either an $\on$
or an $\off$ configuration.
\end{proof}

Recall that each gadget corresponds to a 3-element subset $S$
in the family $\cS$.  Since a gadget in an $\on$ configuration has
greater weight than a gadget in an $\off$ configuration, an optimal
$\k$-RF will have as many gadgets in the $\on$ configuration as
possible.  Thus, to solve X3C we can find the optimal $\k$-RF and, to
obtain a subcover $\cS'$, we place all $S$ into $\cS'$ that
correspond to $\on$ gadgets in the forest.  By
Lemma~\ref{lem.nexus} each subtree can contain at most one nexus
node, which implies that each $\on$ gadget is connected to
element nodes that are not connected to any other $\on$ gadgets.
Thus, this results in a subcover for which each element 
of $X$ appears in at most one $S \in \cS'$.

%sirius

\subsection{Proof of Theorem \ref{thm.approx}}
\label{sec.approxproof}

Recall that we want to show that Algorithm~\ref{alg.approx} returns a
forest with weight that is at least a quarter of the weight of the
optimal $\k$-restricted forest. Let us distinguish two types of
constraints:
\begin{packed_enum}
\item[(a)] the degree of any node is at most $\k$;
\item[(b)] the graph is acyclic.
\end{packed_enum}
Note that the optimal $\k$-restricted forest $F^*_\k$ satisfies both the
constraints above, and hence the maximum weight set of edges that
satisfy both the constraints above has weight at least $w(F^*_\k)$.
Recall that the first stage of Algorithm~\ref{alg.approx} greedily adds
edges subject to these two constraints---the next two lemmas show that
the resulting forest has weight at least $\frac12 w(F^*_\k)$.

\begin{lemma}
\label{lem.2indep}
The family of subgraphs satisfying the constraints (a) and (b) form a
2-independence family. That is, for any subgraph $T$ satisfying (a) and
(b), and for any edge $e \in G$, there exist at most two edges $\{e_1,
e_2\}$ in $T$ such that $T \cup \{e\} - \{e_1, e_2\}$ also satisfies
constraints~(a) and~(b).
\end{lemma}

\begin{proof}
  Let $T$ be a subgraph satisfying (a) and (b) and suppose we add $e =
  (u, v)$ in $T$.  Then the degrees of both $u$ and $v$ are at most
  $\k+1$. If no cycles were created, then we can simply remove an edge
  in $T$ containing $u$ (if any) and an edge in $T$ containing $v$ (if
  any) to satisfy the degree constraint~(a) as well. % If either $u$ or
%   $v$ have no incident edges except the edge $e$ that we added, then we
%   can remove any edge in the subgraph.  
  If adding $e$ created a cycle of the form $\{\ldots, (u',u), (u,v),
  (v,v')\}$, then the edges $(u',u)$ and $(v,v')$ can be removed to
  satisfy both constraints (a) and (b).
\end{proof}

\begin{lemma}
  \label{lem.approx}
  Let $F_1$ be the forest output after Step 1 of algorithm
  \ref{alg.approx}. Then $w(F_1) \geq \frac{1}{2} w(F^*_\k)$.
\end{lemma}

\begin{proof}
  Let $F^{**}$ be a maximum weight forest that obeys both
  constraints~(a) and~(b). Since the optimal $\k$-restructed forest
  $F^*_\k$ obeys both these constraints, we have $w(F^*_\k) \leq
  w(F^{**})$.  By a theorem of \emcite{Hausmann:80}, in a
  $p$-independence family the greedy algorithm is a
  $\frac{1}{p}$-approximation to the maximum weight $p$-independent set.
  By Lemma \ref{lem.2indep}, we know that the set of all subgraphs
  satisfying constraints (a) and (b) is a 2-independent family. Hence,
  $w(F_1) \geq \frac{1}{2} w(F^{**}) \geq \frac{1}{2} w(F^*_\k)$.
\end{proof}

We can now turn to the proof of Theorem \ref{thm.approx}.

\begin{proof}
  Given a graph $G$, let $F_1$ be the forest output by first step of
  Algorithm \ref{alg.approx}, and let $F_A$ be the forest outputted by
  the second step.  We claim that $w(F_A) \geq \frac{1}{2} w(F_1)$;
  combined with Lemma~\ref{lem.approx}, this will complete the proof of
  the theorem.

  To prove the claim, we first show that given any tree $T$ with edge
  weights and maximum degree $\k \geq 2$, we can obtain a sub-forest $F$
  with total weight $w(F) \geq \frac12 w(T)$, and where the number of
  edges in each tree in the forest $F$ is at most $\k-1$. Indeed, root
  the tree $T$ at an arbitrary node of degree-$1$, and call an edge $e$
  \emph{odd} or \emph{even} depending on the parity of the number of
  edges in the unique path between $e$ and the root. Note that the set
  of odd edges and the set of even edges partition $T$ into sub-forests
  composed entirely of stars of maximum degree $\k-1$, and one of these
  sub-forests contains half the weight of $T$, which is what we wanted
  to show.

  Applying this procedure to each tree $T$ in the forest $F_1$, we get
  the existence of a $\k-1$-restricted subforest $F_1' \subseteq F_1$
  that has weight at least $\frac12 w(F_1)$. Observe that a
  $\k-1$-restricted subforest is \emph{a fortiori} a $k$-restricted
  subforest, and since $w(F_A)$ is the best $\k$-restricted subforest of
  $F_1$, we have
  \begin{equation}
   w(F_A) \geq w(F'_1) \geq \frac{1}{2}w(F_1) \geq \frac{1}{4}w(F^*_\k),
  \end{equation}
  completing the proof.
\end{proof}

\subsubsection{An Improved Approximation Algorithm}

We can get an improved approximation algorithm based on a linear
programming approach.  Recall that $F^{**}$ is a maximum weight forest
satisfying both~(a) and~(b). A result of \emcite{SL07} implies that
given any graph $G$ with non-negative edge weights, one can find in
polynomial time a forest $F_{SL}$ such that
\begin{equation}
  \label{eq:sl1}
  w(F_{SL}) \geq w(F^{**}) \geq w(F_\k^*),
\end{equation}
but where the maximum degree in $F_{SL}$ is $\k+1$. Now applying the
procedure from the proof of Theorem~\ref{thm.approx}, we get a
$\k$-restricted forest $F_{SL}'$ whose weight is at least half of
$w(F_{SL})$. Combining this with~(\ref{eq:sl1}) implies that $w(F_{SL}')
\geq w(F_\k^*)$, and completes the proof of the claimed improved
approximation algorithm. We remark that the procedure of \emcite{SL07}
to find the forest $F_{SL}$ is somewhat computationally intensive, since
it requires solving vertex solutions to large linear programs.

\subsection{Proof of Theorem \ref{thm.approxpersist}}

Proceeding as in the proof of Theorem \ref{thm.randompersistency}, we have
that
\begin{eqnarray}
\left|R(\hat p_{\hat F_t}) - R(\hat p_{F_t^*})\right| &\leq&
R(\hat p_{\hat F_t}) - \hat R_{n_1}(\hat p_{\hat F_t}) + 
\left|\hat R_{n_1}(\hat p_{\hat F_t}) - R(\hat p_{F_t^*})\right|\\
&=& O_P\left( k\phi_n(d) + d\psi_n(d)\right) +\left|\hat R_{n_1}(\hat p_{\hat F_t}) - R(\hat p_{F_t^*})\right|.
\end{eqnarray}
Now, let $\hat H_{n1}$ denote the estimated entropy $H(X) = \sum_k
H(X_k)$, constructed using the kernel density estimates $\hat p_{n_1}(x_k)$.
Since the risk is the negative expected log-likelihood, we have
using the approximation guarantee that
\begin{eqnarray}
\hat R_{n_1}(\hat p_{\hat F_t}) - R(\hat p_{F_t^*}) &=&
- \hat w_{n_1}(\hat F_t) + \hat H_{n_1} - R(\hat p_{F_t^*}) \\
&\leq&
- \frac{1}{c} \hat w_{n_1}(F_t^*) + \hat H_{n_1} - R(\hat p_{F_t^*}) \\
&=&
\hat R_{n_1}(\hat p_{F_t}^*) + \frac{c-1}{c} \hat w_{n_1}(F_t^*)  - R(\hat p_{F_t^*}) \\
&=&
O_P\left(k^*\phi_n(d) + d\psi_n(d) + \frac{c-1}{c} w(F_t^*)\right)
\end{eqnarray}
and the result follows.

\subsection{The TreePartition Subroutine}
\label{sec.treepartition}

To produce the best $\k$-restricted subforest of the forest $F_1$, we
use a divide-and-conquer forest partition algorithm described by 
\citep{Lukes:74}, which we now describe in more detail.

To begin, note that finding an optimal subforest is equivalent to
finding a partition of the nodes in the forest, where each disjoint
tree in the subforest is a cluster in the partition.  Since a
forest contains a disjoint set of trees, it suffices to find the
optimal $\k$-restricted partition of each of the trees. 

For every subtree $T$, with root $v$, we will find a \emph{list of
  partitions} $v.P = \{v.P_0, v.P_1, ..., v.P_k\}$ such that
\begin{packed_enum}
\item for $i \neq 0$, $v.P_i$ is a partition whose cluster containing root $v$ has size $i$;
\item $v.P_i$ has the maximum weight among all partitions satisfying the above condition.
\end{packed_enum}
We define $v.P_0$ to be $\argmax \{w(v.P_1), \ldots, w(v.P_\k)\}$.
The \texttt{Merge} subroutine used in \texttt{TreePartition} takes two
lists of parititions $\{v.P, u_i.P\}$, where $v$ is the parent of
$u_i$, $v.P$ is a partition of node $v$ unioned with subtrees of
children $\{u_1, \ldots, u_{i-1}\}$, and $u_i.P$ is a partition of
the subtree of child $u_i$; refer to Figure \ref{fig.partition}.

Since a partition is a list of clusters of nodes, we denote by
$\texttt{Concat}(v.P_2, u.P_{k-2})$ the concatenation of clusters of
partitions $v.P_2, u.P_{k-2}$. Note that the concatenation forms a
partition if $v.P_2$ and $u.P_{k-2}$ are respectively partitions of
two disjoint sets of vertices.  The weight of a partition is denoted
$w(v.P_2)$, that is, the weight of all edges between nodes of the same
cluster in the partition $v.P_2$.

\begin{figure}[H]
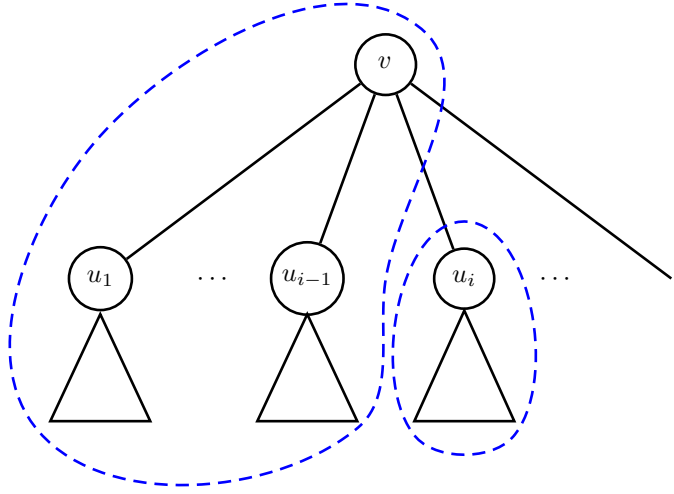

\vskip10pt
\begin{center}
\begin{tabular}{c}
\psset{nodesep=0pt}
\psset{linewidth=1pt}
\def\xscale{27pt}
\def\yscale{\xscale}
\psset{xunit=\xscale}
\psset{yunit=\yscale}
\pspicture(0,0)(8,8)
%\psgrid
\cnodeput(0,4){u1}{\fs$\;u_1\;$}
\cnodeput(2.9,4){u2}{\fs$u_{i-1}$}
\cnodeput(5.1,4){u3}{\fs$\;u_{i}\;$}
\pnode(8,4){u4}
\cnodeput(4,7){u5}{\fs$\;\,v_{\;}\;$}
\ncline{-}{u1}{u5}
\ncline{-}{u2}{u5}
\ncline{-}{u3}{u5}
\ncline{-}{u4}{u5}
\rput(1.6,4){$\cdots$}
\rput(6.4,4){$\cdots$}
\pspolygon[](0,3.5)(.7,2)(-.7,2)
\pspolygon[](2.9,3.5)(3.6,2)(2.2,2)
\pspolygon[](5.1,3.55)(5.8,2)(4.4,2)
\psccurve[linestyle=dashed,linecolor=blue,showpoints=false](-.9,1.9)(3.6,1.9)(4.0,4)(4.5,7.5)(0,6)
\psccurve[linestyle=dashed,linecolor=blue,showpoints=false](4.4,1.8)(5.8,1.8)(5.1,4.8)
\endpspicture
\end{tabular}
\end{center}
\vskip-30pt
\caption{The \texttt{TreePartition} procedure to merge two subproblems.}\label{fig.partition}
\end{figure}

\begin{algorithm}[H]
\label{alg.partition}
\caption{\ \ {\tt TreePartition}$(T,t)$}
\begin{algorithmic}[1]
\ss\STATE \textbf{Input} a tree $T$, a positive integer $\k$
\ss\STATE \textbf{Returns} an optimal partition into trees of size $\leq t$.
\ss\STATE Initialize $v.P_1 = [\{v\}]$ where $v$ is root of $T$, if $v$ has no children, return $v.P_1$
\ss\STATE For all children $\{u_1, ... u_s\}$ of $v$, recursively call \texttt{TreePartition}$(u_i,t)$ to get a collection of lists of partitions $\{u_1.P, u_2.P, ... u_s.P\}$
\ss\STATE For each child $u_i \in \{u_1, ... u_s\}$ of $v$ \\
      Update $v.P \leftarrow \texttt{Merge}(u_i.P, v.P)$
\ss\STATE \textbf{Output} $v.P_0$
\end{algorithmic}
\end{algorithm}

\begin{algorithm}[H]
\label{alg:merge}
\caption{\ \ \texttt{Merge}$(v.P,\, u.P)$}
\begin{algorithmic}[1]
\ss\STATE \textbf{Input} a list of partitions $v.P$ and $u.P$, where $v$ is a parent of $u$.
\ss\STATE \textbf{Returns} a single list of partitions $v.P'$.
\ss\STATE For $i = 1,\ldots, \k$:
\begin{enumerate}
   \item Let $(s^*,t^*) = \argmax_{(s,t):s+t = i} w(\texttt{Concat}(v.P_s, u.P_t))$
   \item Let $v.P'_i = \texttt{Concat}(v.P_{s^*}, u.P_{t^*})$
\end{enumerate}
\ss\STATE Select $v.P'_0 = \argmax_{v.P'_i} w(v.P'_i)$
\ss\STATE \textbf{Output} $\{v.P'_0, v.P'_1, ... v.P'_n\}$
\end{algorithmic}
\end{algorithm}

%antares

\section{Computation of the Mutual Information Matrix}
\label{subsection.efficientMI}

In this appendix we explain different methods for computing the
mutual information matrix, and making the tree estimation more
efficient.
One way to evaluate the empirical mutual information is to use
\begin{equation}
\hat{I}(X_{i}; X_{j}) = \frac{1}{n_{1}}\sum_{s\in\mathcal{D}_{1}}\log \frac{\hat{p}_{n_1}(X^{(s)}_{i}, X^{(s)}_{j}) }{\hat{p}_{n_1}(X^{(s)}_{i}) \,\hat{p}_{n_1}(X^{(s)}_{j})}.  \label{eq.anotherMI}
\end{equation}
Compared with our proposed method
\begin{equation}
\hat{I}_{n_1}(X_{i}, X_{j}) = 
\frac{1}{m^{2}}\sum_{k=1}^{m}\sum_{\ell=1}^{m}\hat{p}_{n_1}(x_{ki},x_{\ell j}) 
\log \frac{\hat{p}_{n_1}(x_{ki}, x_{\ell j}) }{\hat{p}_{n_1}(x_{ki})\,
  \hat{p}_{n_1}(x_{\ell j})},  \label{eq.empiricalMIagain}
\end{equation}
\eqref{eq.anotherMI} is somewhat
easier to calculate. However, if the sample size in
$\mathcal{D}_{1}$ is small, the approximation error can be large.
A different analysis is needed to provide
justification of the method based on \eqref{eq.anotherMI}, which would
be more difficult since $\hat{p}_{n_1}(\cdot)$ is dependent on
$\mathcal{D}_{1}$.  For these reasons we use
the method in \eqref{eq.empiricalMIagain}.

Also, note that instead of using the grid based method to evaluate the numerical
integral, one could use sampling. If we can obtain $m_{1}$
i.i.d. samples from the bivariate density $\hat{p}(X_{i}, X_{j})$, 
\begin{eqnarray}
\left\{(X^{(s)}_{i}, X^{(s)}_{j})\right\}_{s=1}^{m_{1}} \;\stackrel{\rm i.i.d.}{\sim}\; \hat{p}_{n_1}(x_{i}, x_{j}),
\end{eqnarray}
then the empirical mutual information can be evaluated as
\begin{equation}
\hat{I}(X_{i}; X_{j}) = \frac{1}{m_{1}}\sum_{s=1}^{m_{1}}\log \frac{\hat{p}(X^{(s)}_{i}, X^{(s)}_{j}) }{\hat{p}(X^{(s)}_{i}) \hat{p}(X^{(s)}_{j})}. 
\end{equation}

Compared with \eqref{eq.anotherMI}, the main advantage of this approach
is that the estimate can be arbitrarily close to 
\eqref{eq.empiricalMI} for large enough $m_{1}$ and $m$. Also, the
computation can be easier compared to Algorithm
\ref{algorithm:naiveMI}.
Let $\hat{p}_{n_1}(X_{i}, X_{j})$ be the bivariate kernel density estimator
on $\mathcal{D}_{1}$.
To sample a point from $\hat{p}_{n_1}(X_{i}, X_{j})$, we
first random draw a sample $(X^{(k')}_{i}, X^{(\ell')}_{j})$ from
$\mathcal{D}_{1}$, and then sample a point $(X,Y)$ from the bivariate
distribution
\begin{eqnarray}
(X,Y) \sim \frac{1}{h^{2}_{2}}  K\left( \frac{X^{(k')}_{i} - \cdot}{h_{2}} \right)K\left( \frac{X^{(\ell')}_{j} - \cdot}{h_{2}} \right). 
\end{eqnarray}
Though this sampling strategy is superior to Algorithm
\ref{algorithm:naiveMI}, it requires evaluation of the bivariate
kernel density estimates on many random points, which is time
consuming; the grid-based method is preferred.

In our two-stage procedure, the stage
requires calculation of the empirical mutual information
$\hat{I}(X_{i}; X_{j})$  for $\binom{d}{2}$ entries.
Each requires $O(m^{2} n_{1})$ work to evaluate the bivariate and
univariate kernel density estimates on the $m\times m$ grid, in a 
naive implementation. Therefore, the total time to calculate the
empirical mutual information matrix $M$ is $O(m^{2}n_{1} d^{2})$.  In
the second stage, the time complexity of the Chow-Liu algorithm is dominated by the first step. Therefore the total
time complexity is $O\left(m^{2}n_{1} d^{2}\right)$. The first stage requires
$O(d^{2})$ space to store the matrix $M$ and $O(m^{2}n_{1})$
space to evaluate the kernel density estimates on $\mathcal{D}_{1}$.
The space complexity for the Chow-Liu algorithm is $O(d^{2})$,
and thus the total space complexity is $O(d^{2} +
m^{2}n_{1})$.

\begin{algorithm}[ht!]
   \caption{More efficient calculation of the mutual information matrix $M$.}\label{algorithm:fastMI}
%   {\bfseries Input:} Data set $\mathcal{D}_{1}$ and the  bandwidths $h_{1}$, $h_{2}$
   \begin{algorithmic}[1] 

\STATE Initialize $M = \mathbf{0}_{d\times d}$ and $H^{(i)} = \mathbf{0}_{n_{1}\times m}$ for $i=1,\ldots, d$.
 
\vspace{0.05in}

 \STATE{\% \sf calculate and pre-store the univariate KDE} 
 \FOR{$k = 1, \ldots, d$}
    \FOR{$k' = 1, \ldots, m$}
            \STATE $\ds \hat{p}(x^{(k')}_{k}) \leftarrow \frac{1}{n_{1}}\sum_{s \in \mathcal{D}_{1}} \frac{1}{h_{1}}  K\left( \frac{X^{(s)}_{k} - x^{(k')}_{k}}{h_{1}} \right)$
   \ENDFOR
\ENDFOR

\vspace{0.05in}
% {$\%$ \sf comment: the main loop}

 \FOR{$k' = 1, \ldots, m$} 
   \STATE{\% \sf calculate the components used for the bivariate KDE} 
   \FOR{$i' = 1, \ldots, n_{1}$}
       \FOR{$i = 1, \ldots, d$} 
        \STATE $\ds H^{(i)}(i', k') \leftarrow \frac{1}{h_{2}} K\left( \frac{X^{i'}_{i} - x^{(k')}_{i}}{h_{2}}\right)$
       \ENDFOR
  \ENDFOR

  \vspace{0.05in}
  \STATE{\% \sf calculate the mutual information matrix}
  \FOR{$\ell' = 1, \ldots, m$}
       \FOR{$i = 1, \ldots, d-1$}
           \FOR{$j = i+1, \ldots, d$} 
             \STATE $\hat{p}(x^{(k')}_{i}, x^{(\ell')}_{j}) \leftarrow 0$
           \ENDFOR
           \FOR{$i' = 1, \ldots, n_{1}$} 
             \STATE $\hat{p}(x^{(k')}_{i}, x^{(\ell')}_{j}) \leftarrow \hat{p}(x^{(k')}_{i}, x^{(\ell')}_{j})  + H^{(i)}(i', k')\cdot H^{(j)}(i', \ell')$
           \ENDFOR
           \STATE $\hat{p}(x^{(k')}_{i}, x^{(\ell')}_{j}) \leftarrow \hat{p}(x^{(k')}_{i}, x^{(\ell')}_{j})/n_{1}$
           \STATE $M(i,j) \leftarrow M(i,j) + \frac{1}{m^2}\hat{p}(x^{(k')}_{i}, x^{(\ell')}_{j})\cdot\log\left(\hat{p}(x^{(k')}_{i}, x^{(\ell')}_{j})/(\hat{p}(x^{(k')}_{i})\cdot\hat{p}(x^{(\ell')}_{j})) \right)$
       \ENDFOR
  \ENDFOR 
 \ENDFOR
% \STATE $M(i,j) \leftarrow M(i,j)/m^{2}$ for each $i,j =1,\ldots, d$
\end{algorithmic}
%   {\bfseries Output:} The calculated empirical pairwise mutual information matrix $M$.

\end{algorithm}

The quadratic time and space complexity in the number of variables $d$
is acceptable for many practical applications but can be prohibitive when the dimension $d$
is large.  The main
bottleneck is to calculate the empirical mutual information matrix
$M$. Due to the utilization of the kernel density estimate, the time
complexity is $O(d^{2}m^{2}n_{1})$. The straightforward implementation 
in Algorithm \ref{algorithm:naiveMI} is conceptually easy
but computationally inefficient, due to many 
redundant operations. For example, in the nested for loop, many
components of the bivariate and univariate kernel density estimates
are repeatedly evaluated. In Algorithm
\ref{algorithm:fastMI}, we suggest an alternative method which can
significantly reduce such redundancy at the price of increased but still
affordable space complexity.

The main technique used in Algorithm \ref{algorithm:fastMI} is to change
the order of the multiple nested for loops, combined with some
pre-calculation.  This algorithm can significantly boost
the empirical performance, although the worst case time complexity remains the same. 
An alternative suggested by \cite{Bach03}
is to approximate the mutual information, although this would require
further analysis and justification.

\bibliography{tde}

\begin{thebibliography}{25}

\bibitem[\protect\citeauthoryear{Aigner and Ziegler}{1998}]{THEBOOK:98}
\begin{bbook}[author]
\bauthor{\bsnm{Aigner},~\bfnm{Martin}\binits{M.}} \AND
  \bauthor{\bsnm{Ziegler},~\bfnm{G\"nter}\binits{G.}}
(\byear{1998}).
\btitle{Proofs from {THE BOOK}}.
\bpublisher{Springer-Verlag}.
\end{bbook}
\endbibitem

\bibitem[\protect\citeauthoryear{Bach and Jordan}{2003}]{Bach03}
\begin{barticle}[author]
\bauthor{\bsnm{Bach},~\bfnm{Francis~R.}\binits{F.~R.}} \AND
  \bauthor{\bsnm{Jordan},~\bfnm{Michael~I.}\binits{M.~I.}}
(\byear{2003}).
\btitle{Beyond Independent Components: Trees and Clusters}.
\bjournal{Journal of Machine Learning Research}
\bvolume{4}
\bpages{1205--1233}.
\end{barticle}
\endbibitem

\bibitem[\protect\citeauthoryear{Banerjee, {El Ghaoui} and
  d'Aspremont}{2008}]{Banerjee:08}
\begin{barticle}[author]
\bauthor{\bsnm{Banerjee},~\bfnm{Onureena}\binits{O.}}, \bauthor{\bsnm{{El
  Ghaoui}},~\bfnm{Laurent}\binits{L.}} \AND
  \bauthor{\bsnm{d'Aspremont},~\bfnm{Alexandre}\binits{A.}}
(\byear{2008}).
\btitle{Model selection through sparse maximum likelihood estimation}.
\bjournal{Journal of Machine Learning Research}
\bvolume{9}
\bpages{485--516}.
\end{barticle}
\endbibitem

\bibitem[\protect\citeauthoryear{Cayley}{1889}]{Cayley:1889}
\begin{barticle}[author]
\bauthor{\bsnm{Cayley},~\bfnm{Arthur}\binits{A.}}
(\byear{1889}).
\btitle{A theorem on trees}.
\bjournal{Quart. J. Math.}
\bvolume{23}
\bpages{376--378}.
\end{barticle}
\endbibitem

\bibitem[\protect\citeauthoryear{Chechetka and
  Guestrin}{2007}]{Chechetka+Guestrin:nips07tjtpac}
\begin{binproceedings}[author]
\bauthor{\bsnm{Chechetka},~\bfnm{Anton}\binits{A.}} \AND
  \bauthor{\bsnm{Guestrin},~\bfnm{Carlos}\binits{C.}}
(\byear{2007}).
\btitle{Efficient Principled Learning of Thin Junction Trees}.
In \bbooktitle{In Advances in Neural Information Processing Systems (NIPS)}.
\end{binproceedings}
\endbibitem

\bibitem[\protect\citeauthoryear{Chow and Liu}{1968}]{chow68}
\begin{barticle}[author]
\bauthor{\bsnm{Chow},~\bfnm{C.}\binits{C.}} \AND
  \bauthor{\bsnm{Liu},~\bfnm{C.}\binits{C.}}
(\byear{1968}).
\btitle{Approximating discrete probability distributions with dependence
  trees}.
\bjournal{IEEE Transactions on Information Theory}
\bvolume{14}
\bpages{462--467}.
\end{barticle}
\endbibitem

\bibitem[\protect\citeauthoryear{Fan and Gijbels}{1996}]{Fan:Gijb:1996}
\begin{bbook}[author]
\bauthor{\bsnm{Fan},~\bfnm{Jianqing}\binits{J.}} \AND
  \bauthor{\bsnm{Gijbels},~\bfnm{Ir\`{e}ne}\binits{I.}}
(\byear{1996}).
\btitle{Local polynomial modelling and its applications}.
\bpublisher{Chapman and Hall}.
\end{bbook}
\endbibitem

\bibitem[\protect\citeauthoryear{Friedman, Hastie and
  Tibshirani}{2007}]{FHT:07}
\begin{barticle}[author]
\bauthor{\bsnm{Friedman},~\bfnm{Jerome~H.}\binits{J.~H.}},
  \bauthor{\bsnm{Hastie},~\bfnm{Trevor}\binits{T.}} \AND
  \bauthor{\bsnm{Tibshirani},~\bfnm{Robert}\binits{R.}}
(\byear{2007}).
\btitle{Sparse inverse covariance estimation with the graphical lasso}.
\bjournal{Biostatistics}
\bvolume{9}
\bpages{432--441}.
\end{barticle}
\endbibitem

\bibitem[\protect\citeauthoryear{Garey and Johnson}{1979}]{Garey:79}
\begin{bbook}[author]
\bauthor{\bsnm{Garey},~\bfnm{M.~R.}\binits{M.~R.}} \AND
  \bauthor{\bsnm{Johnson},~\bfnm{David~S.}\binits{D.~S.}}
(\byear{1979}).
\btitle{Computers and Intractability: A Guide to the Theory of
  NP-Completeness}.
\bpublisher{W. H. Freeman}.
\end{bbook}
\endbibitem

\bibitem[\protect\citeauthoryear{Gin\'{e} and Guillou}{2002}]{Gine:2002}
\begin{barticle}[author]
\bauthor{\bsnm{Gin\'{e}},~\bfnm{E.}\binits{E.}} \AND
  \bauthor{\bsnm{Guillou},~\bfnm{A.}\binits{A.}}
(\byear{2002}).
\btitle{Rates of strong uniform consistency for multivariate kernel density
  estimators}.
\bjournal{Annales de l'institut Henri Poincar\'{e} (B), Probabilit\'{e}s et
  Statistiques}
\bvolume{38}
\bpages{907-921}.
\end{barticle}
\endbibitem

\bibitem[\protect\citeauthoryear{Hausmann, Korte and
  Jenkyns}{1980}]{Hausmann:80}
\begin{barticle}[author]
\bauthor{\bsnm{Hausmann},~\bfnm{D.}\binits{D.}},
  \bauthor{\bsnm{Korte},~\bfnm{B.}\binits{B.}} \AND
  \bauthor{\bsnm{Jenkyns},~\bfnm{T.~A.}\binits{T.~A.}}
(\byear{1980}).
\btitle{Worst Case Analysis of Greedy Type Algorithms for Independence
  Systems}.
\bjournal{Math. Programming Studies}
\bvolume{12}
\bpages{120--131}.
\end{barticle}
\endbibitem

\bibitem[\protect\citeauthoryear{Kruskal}{1956}]{Kruskal:1956}
\begin{barticle}[author]
\bauthor{\bsnm{Kruskal},~\bfnm{Joseph~B.}\binits{J.~B.}}
(\byear{1956}).
\btitle{On the Shortest Spanning Subtree of a Graph and the Traveling Salesman
  Problem}.
\bjournal{Proceedings of the American Mathematical Society}
\bvolume{7}
\bpages{48--50}.
\end{barticle}
\endbibitem

\bibitem[\protect\citeauthoryear{Lauritzen}{1996}]{Lauritzen:1996}
\begin{bbook}[author]
\bauthor{\bsnm{Lauritzen},~\bfnm{Steffen~L.}\binits{S.~L.}}
(\byear{1996}).
\btitle{Graphical Models}.
\bpublisher{Clarendon Press}.
\end{bbook}
\endbibitem

\bibitem[\protect\citeauthoryear{Liu, Lafferty and Wasserman}{2009}]{npn:09}
\begin{barticle}[author]
\bauthor{\bsnm{Liu},~\bfnm{Han}\binits{H.}},
  \bauthor{\bsnm{Lafferty},~\bfnm{John}\binits{J.}} \AND
  \bauthor{\bsnm{Wasserman},~\bfnm{Larry}\binits{L.}}
(\byear{2009}).
\btitle{The Nonparanormal: Semiparametric Estimation of High Dimensional
  Undirected Graphs}.
\bjournal{Journal of Machine Learning Research}
\bvolume{10}
\bpages{2295--2328}.
\end{barticle}
\endbibitem

\bibitem[\protect\citeauthoryear{Lukes}{1974}]{Lukes:74}
\begin{barticle}[author]
\bauthor{\bsnm{Lukes},~\bfnm{J.~A.}\binits{J.~A.}}
(\byear{1974}).
\btitle{Efficient Algorithm for the Partitioning of Trees}.
\bjournal{IBM Jour. of Res. and Dev.}
\bvolume{18}
\bpages{274}.
\end{barticle}
\endbibitem

\bibitem[\protect\citeauthoryear{Meinshausen and
  B\"{u}hlmann}{2006}]{Meinshausen:2006}
\begin{barticle}[author]
\bauthor{\bsnm{Meinshausen},~\bfnm{Nicolai}\binits{N.}} \AND
  \bauthor{\bsnm{B\"{u}hlmann},~\bfnm{Peter}\binits{P.}}
(\byear{2006}).
\btitle{High dimensional graphs and variable selection with the {L}asso}.
\bjournal{The Annals of Statistics}
\bvolume{34}
\bpages{1436-1462}.
\end{barticle}
\endbibitem

\bibitem[\protect\citeauthoryear{Nayak {\it et~al.}}{2009}]{nayak:09}
\begin{barticle}[author]
\bauthor{\bsnm{Nayak},~\bfnm{R.}\binits{R.}},
  \bauthor{\bsnm{Kearns},~\bfnm{M.}\binits{M.}},
  \bauthor{\bsnm{Spielman},~\bfnm{Richard}\binits{R.}} \AND
  \bauthor{\bsnm{Cheung},~\bfnm{Vivian}\binits{V.}}
(\byear{2009}).
\btitle{Coexpression Network based on Natural Variation in Human Gene
  Expression Reveals Gene Interactions and Functions}.
\bjournal{Genome Research}
\bvolume{19}
\bpages{1953--1962}.
\end{barticle}
\endbibitem

\bibitem[\protect\citeauthoryear{Nolan and Pollard}{1987}]{nolan:1987}
\begin{barticle}[author]
\bauthor{\bsnm{Nolan},~\bfnm{Deborah}\binits{D.}} \AND
  \bauthor{\bsnm{Pollard},~\bfnm{David}\binits{D.}}
(\byear{1987}).
\btitle{U-Processes: Rates of Convergence}.
\bjournal{The Annals of Statistics}
\bvolume{15}
\bpages{780 - 799}.
\end{barticle}
\endbibitem

\bibitem[\protect\citeauthoryear{Rigollet and Vert}{2009}]{Rig09}
\begin{barticle}[author]
\bauthor{\bsnm{Rigollet},~\bfnm{Philippe}\binits{P.}} \AND
  \bauthor{\bsnm{Vert},~\bfnm{R\'{e}gis}\binits{R.}}
(\byear{2009}).
\btitle{Fast rates for plug-in estimators of density level sets}.
\bjournal{Bernoulli (to appear)}.
\end{barticle}
\endbibitem

\bibitem[\protect\citeauthoryear{Rinaldo and Wasserman}{2009}]{rinaldo:2009}
\begin{barticle}[author]
\bauthor{\bsnm{Rinaldo},~\bfnm{Alessandro}\binits{A.}} \AND
  \bauthor{\bsnm{Wasserman},~\bfnm{Larry}\binits{L.}}
(\byear{2009}).
\btitle{Low-Noise Density Clustering}.
\bjournal{Technical report, Carnegie Mellon University}.
\end{barticle}
\endbibitem

\bibitem[\protect\citeauthoryear{Singh and Lau}{2007}]{SL07}
\begin{bincollection}[author]
\bauthor{\bsnm{Singh},~\bfnm{Mohit}\binits{M.}} \AND
  \bauthor{\bsnm{Lau},~\bfnm{Lap~Chi}\binits{L.~C.}}
(\byear{2007}).
\btitle{Approximating minimum bounded degree spanning trees to within one of
  optimal}.
In \bbooktitle{S{TOC}'07---{P}roceedings of the 39th {A}nnual {ACM} {S}ymposium
  on {T}heory of {C}omputing}
\bpages{661--670}.
\bpublisher{ACM}, \baddress{New York}.
\end{bincollection}
\endbibitem

\bibitem[\protect\citeauthoryear{Tan, Anandkumar and Willsky}{2009}]{Tan:09b}
\begin{barticle}[author]
\bauthor{\bsnm{Tan},~\bfnm{V.}\binits{V.}},
  \bauthor{\bsnm{Anandkumar},~\bfnm{A.}\binits{A.}} \AND
  \bauthor{\bsnm{Willsky},~\bfnm{A.}\binits{A.}}
(\byear{2009}).
\btitle{Learning {G}aussian Tree Models: {A}nalysis of Error Exponents and
  Extremal Structures}.
\bjournal{\rm arXiv:0909.5216}.
\end{barticle}
\endbibitem

\bibitem[\protect\citeauthoryear{Tan {\it et~al.}}{2009}]{Tan:09a}
\begin{barticle}[author]
\bauthor{\bsnm{Tan},~\bfnm{V.}\binits{V.}},
  \bauthor{\bsnm{Anandkumar},~\bfnm{A.}\binits{A.}},
  \bauthor{\bsnm{Tong},~\bfnm{L.}\binits{L.}} \AND
  \bauthor{\bsnm{Willsky},~\bfnm{A.}\binits{A.}}
(\byear{2009}).
\btitle{A Large-Deviation Analysis for the Maximum Likelihood Learning of Tree
  Structures}.
\bjournal{\rm arXiv:0905.0940}.
\end{barticle}
\endbibitem

\bibitem[\protect\citeauthoryear{Tsybakov}{2008}]{Tsybakov09}
\begin{bbook}[author]
\bauthor{\bsnm{Tsybakov},~\bfnm{Alexandre~B.}\binits{A.~B.}}
(\byear{2008}).
\btitle{Introduction to Nonparametric Estimation}.
\bpublisher{Springer Publishing Company, Incorporated}.
\end{bbook}
\endbibitem

\bibitem[\protect\citeauthoryear{Wille {\it et~al.}}{2004}]{wille:04}
\begin{barticle}[author]
\bauthor{\bsnm{Wille},~\bfnm{Anja}\binits{A.}},
  \bauthor{\bsnm{Zimmermann},~\bfnm{Philip}\binits{P.}},
  \bauthor{\bsnm{Vranov\'a},~\bfnm{Eva}\binits{E.}},
  \bauthor{\bsnm{F\"urholz},~\bfnm{Andreas}\binits{A.}},
  \bauthor{\bsnm{Laule},~\bfnm{Oliver}\binits{O.}},
  \bauthor{\bsnm{Bleuler},~\bfnm{Stefan}\binits{S.}},
  \bauthor{\bsnm{Hennig},~\bfnm{Lars}\binits{L.}},
  \bauthor{\bsnm{Preli\'c},~\bfnm{Amela}\binits{A.}}, \bauthor{\bparticle{von
  }\bsnm{Rohr},~\bfnm{Peter}\binits{P.}},
  \bauthor{\bsnm{Thiele},~\bfnm{Lothar}\binits{L.}},
  \bauthor{\bsnm{Zitzler},~\bfnm{Eckart}\binits{E.}},
  \bauthor{\bsnm{Gruissem},~\bfnm{Wilhelm}\binits{W.}} \AND
  \bauthor{\bsnm{B\"uhlmann},~\bfnm{Peter}\binits{P.}}
(\byear{2004}).
\btitle{Sparse {G}aussian graphical modelling of the isoprenoid gene network in
  \it {A}rabidopsis thaliana\rm}.
\bjournal{Genome Biology}
\bvolume{5}
\bpages{R92}.
\end{barticle}
\endbibitem

\end{thebibliography}


\begin{thebibliography}{16}

\bibitem[\protect\citeauthoryear{Bach and Jordan}{2003}]{Bach03}
\begin{barticle}[author]
\bauthor{\bsnm{Bach},~\bfnm{Francis~R.}\binits{F.~R.}} \AND
  \bauthor{\bsnm{Jordan},~\bfnm{Michael~I.}\binits{M.~I.}}
(\byear{2003}).
\btitle{Beyond Independent Components: Trees and Clusters}.
\bjournal{Journal of Machine Learning Research}
\bvolume{4}
\bpages{1205--1233}.
\end{barticle}
\endbibitem

\bibitem[\protect\citeauthoryear{Banerjee, {El Ghaoui} and
  d'Aspremont}{2008}]{Banerjee:08}
\begin{barticle}[author]
\bauthor{\bsnm{Banerjee},~\bfnm{Onureena}\binits{O.}}, \bauthor{\bsnm{{El
  Ghaoui}},~\bfnm{Laurent}\binits{L.}} \AND
  \bauthor{\bsnm{d'Aspremont},~\bfnm{Alexandre}\binits{A.}}
(\byear{2008}).
\btitle{Model selection through sparse maximum likelihood estimation}.
\bjournal{Journal of Machine Learning Research}
\bvolume{9}
\bpages{485--516}.
\end{barticle}
\endbibitem

\bibitem[\protect\citeauthoryear{Chechetka and
  Guestrin}{2007}]{Chechetka+Guestrin:nips07tjtpac}
\begin{binproceedings}[author]
\bauthor{\bsnm{Chechetka},~\bfnm{Anton}\binits{A.}} \AND
  \bauthor{\bsnm{Guestrin},~\bfnm{Carlos}\binits{C.}}
(\byear{2007}).
\btitle{Efficient Principled Learning of Thin Junction Trees}.
In \bbooktitle{In Advances in Neural Information Processing Systems (NIPS)}.
\end{binproceedings}
\endbibitem

\bibitem[\protect\citeauthoryear{Chow and Liu}{1968}]{chow68}
\begin{barticle}[author]
\bauthor{\bsnm{Chow},~\bfnm{C.}\binits{C.}} \AND
  \bauthor{\bsnm{Liu},~\bfnm{C.}\binits{C.}}
(\byear{1968}).
\btitle{Approximating discrete probability distributions with dependence
  trees}.
\bjournal{IEEE Transactions on Information Theory}
\bvolume{14}
\bpages{462--467}.
\end{barticle}
\endbibitem

\bibitem[\protect\citeauthoryear{Friedman, Hastie and
  Tibshirani}{2007}]{FHT:07}
\begin{barticle}[author]
\bauthor{\bsnm{Friedman},~\bfnm{Jerome~H.}\binits{J.~H.}},
  \bauthor{\bsnm{Hastie},~\bfnm{Trevor}\binits{T.}} \AND
  \bauthor{\bsnm{Tibshirani},~\bfnm{Robert}\binits{R.}}
(\byear{2007}).
\btitle{Sparse inverse covariance estimation with the graphical lasso}.
\bjournal{Biostatistics}
\bvolume{9}
\bpages{432--441}.
\end{barticle}
\endbibitem

\bibitem[\protect\citeauthoryear{Gin\'{e} and Guillou}{2002}]{Gine:2002}
\begin{barticle}[author]
\bauthor{\bsnm{Gin\'{e}},~\bfnm{E.}\binits{E.}} \AND
  \bauthor{\bsnm{Guillou},~\bfnm{A.}\binits{A.}}
(\byear{2002}).
\btitle{Rates of strong uniform consistency for multivariate kernel density
  estimators}.
\bjournal{Annales de l'institut Henri Poincar\'{e} (B), Probabilit\'{e}s et
  Statistiques}
\bvolume{38}
\bpages{907-921}.
\end{barticle}
\endbibitem

\bibitem[\protect\citeauthoryear{Kruskal}{1956}]{Kruskal:1956}
\begin{barticle}[author]
\bauthor{\bsnm{Kruskal},~\bfnm{Joseph~B.}\binits{J.~B.}}
(\byear{1956}).
\btitle{On the Shortest Spanning Subtree of a Graph and the Traveling Salesman
  Problem}.
\bjournal{Proceedings of the American Mathematical Society}
\bvolume{7}
\bpages{48--50}.
\end{barticle}
\endbibitem

\bibitem[\protect\citeauthoryear{Lauritzen}{1996}]{Lauritzen:1996}
\begin{bbook}[author]
\bauthor{\bsnm{Lauritzen},~\bfnm{Steffen~L.}\binits{S.~L.}}
(\byear{1996}).
\btitle{Graphical Models}.
\bpublisher{Clarendon Press}.
\end{bbook}
\endbibitem

\bibitem[\protect\citeauthoryear{Liu, Lafferty and Wasserman}{2009}]{npn:09}
\begin{barticle}[author]
\bauthor{\bsnm{Liu},~\bfnm{Han}\binits{H.}},
  \bauthor{\bsnm{Lafferty},~\bfnm{John}\binits{J.}} \AND
  \bauthor{\bsnm{Wasserman},~\bfnm{Larry}\binits{L.}}
(\byear{2009}).
\btitle{The Nonparanormal: Semiparametric Estimation of High Dimensional
  Undirected Graphs}.
\bjournal{Journal of Machine Learning Research}
\bvolume{10}
\bpages{2295--2328}.
\end{barticle}
\endbibitem

\bibitem[\protect\citeauthoryear{Nolan and Pollard}{1987}]{nolan:1987}
\begin{barticle}[author]
\bauthor{\bsnm{Nolan},~\bfnm{Deborah}\binits{D.}} \AND
  \bauthor{\bsnm{Pollard},~\bfnm{David}\binits{D.}}
(\byear{1987}).
\btitle{U-Processes: Rates of Convergence}.
\bjournal{The Annals of Statistics}
\bvolume{15}
\bpages{780 - 799}.
\end{barticle}
\endbibitem

\bibitem[\protect\citeauthoryear{Rigollet and Vert}{2009}]{Rig09}
\begin{barticle}[author]
\bauthor{\bsnm{Rigollet},~\bfnm{Philippe}\binits{P.}} \AND
  \bauthor{\bsnm{Vert},~\bfnm{R\'{e}gis}\binits{R.}}
(\byear{2009}).
\btitle{Fast rates for plug-in estimators of density level sets}.
\bjournal{Bernoulli (to appear)}.
\end{barticle}
\endbibitem

\bibitem[\protect\citeauthoryear{Rinaldo and Wasserman}{2009}]{rinaldo:2009}
\begin{barticle}[author]
\bauthor{\bsnm{Rinaldo},~\bfnm{Alessandro}\binits{A.}} \AND
  \bauthor{\bsnm{Wasserman},~\bfnm{Larry}\binits{L.}}
(\byear{2009}).
\btitle{Low-Noise Density Clustering}.
\bjournal{Technical report, Carnegie Mellon University}.
\end{barticle}
\endbibitem

\bibitem[\protect\citeauthoryear{Tan, Anandkumar and Willsky}{2009}]{Tan:09b}
\begin{barticle}[author]
\bauthor{\bsnm{Tan},~\bfnm{V.}\binits{V.}},
  \bauthor{\bsnm{Anandkumar},~\bfnm{A.}\binits{A.}} \AND
  \bauthor{\bsnm{Willsky},~\bfnm{A.}\binits{A.}}
(\byear{2009}).
\btitle{Learning Gaussian Tree Models: Analysis of Error Exponents and Extremal
  Structures}.
\bjournal{\rm arXiv:0909.5216}.
\end{barticle}
\endbibitem

\bibitem[\protect\citeauthoryear{Tan {\it et~al.}}{2009}]{Tan:09a}
\begin{barticle}[author]
\bauthor{\bsnm{Tan},~\bfnm{V.}\binits{V.}},
  \bauthor{\bsnm{Anandkumar},~\bfnm{A.}\binits{A.}},
  \bauthor{\bsnm{Tong},~\bfnm{L.}\binits{L.}} \AND
  \bauthor{\bsnm{Willsky},~\bfnm{A.}\binits{A.}}
(\byear{2009}).
\btitle{{A Large-Deviation Analysis for the Maximum Likelihood Learning of Tree
  Structures}}.
\bjournal{\rm arXiv:0905.0940}.
\end{barticle}
\endbibitem

\bibitem[\protect\citeauthoryear{Tsybakov}{2008}]{Tsybakov09}
\begin{bbook}[author]
\bauthor{\bsnm{Tsybakov},~\bfnm{Alexandre~B.}\binits{A.~B.}}
(\byear{2008}).
\btitle{Introduction to Nonparametric Estimation}.
\bpublisher{Springer Publishing Company, Incorporated}.
\end{bbook}
\endbibitem

\bibitem[\protect\citeauthoryear{Wille {\it et~al.}}{2004}]{wille:04}
\begin{barticle}[author]
\bauthor{\bsnm{Wille},~\bfnm{Anja}\binits{A.}},
  \bauthor{\bsnm{Zimmermann},~\bfnm{Philip}\binits{P.}},
  \bauthor{\bsnm{Vranov\'a},~\bfnm{Eva}\binits{E.}},
  \bauthor{\bsnm{F\"urholz},~\bfnm{Andreas}\binits{A.}},
  \bauthor{\bsnm{Laule},~\bfnm{Oliver}\binits{O.}},
  \bauthor{\bsnm{Bleuler},~\bfnm{Stefan}\binits{S.}},
  \bauthor{\bsnm{Hennig},~\bfnm{Lars}\binits{L.}},
  \bauthor{\bsnm{Preli\'c},~\bfnm{Amela}\binits{A.}}, \bauthor{\bparticle{von
  }\bsnm{Rohr},~\bfnm{Peter}\binits{P.}},
  \bauthor{\bsnm{Thiele},~\bfnm{Lothar}\binits{L.}},
  \bauthor{\bsnm{Zitzler},~\bfnm{Eckart}\binits{E.}},
  \bauthor{\bsnm{Gruissem},~\bfnm{Wilhelm}\binits{W.}} \AND
  \bauthor{\bsnm{B\"uhlmann},~\bfnm{Peter}\binits{P.}}
(\byear{2004}).
\btitle{Sparse {G}aussian graphical modelling of the isoprenoid gene network in
  \it {A}rabidopsis thaliana\rm}.
\bjournal{Genome Biology}
\bvolume{5}
\bpages{R92}.
\end{barticle}
\endbibitem

\end{thebibliography}

\end{document}